\documentclass[Afour,sageh,times]{sagej}

\usepackage{moreverb,url}

\newcommand\BibTeX{{\rmfamily B\kern-.05em \textsc{i\kern-.025em b}\kern-.08em
		T\kern-.1667em\lower.7ex\hbox{E}\kern-.125emX}}

\def\volumeyear{2016}

\usepackage[T1]{fontenc}
\usepackage[utf8]{inputenc}

\usepackage{amsfonts}
\usepackage{algpseudocode}
\usepackage{algorithm}
\usepackage{array}

\usepackage{tabularx,booktabs}
\newcolumntype{Y}{>{\raggedright\arraybackslash}X}

\usepackage[caption=false,font=footnotesize,labelfont=sf,textfont=sf]{subfig}
\usepackage{textcomp}
\usepackage{stfloats}
\usepackage{verbatim}
\usepackage{graphicx}
\usepackage{xcolor}

\usepackage{amsmath,amssymb}
\usepackage{multirow}

\usepackage{amsthm}

\usepackage[colorlinks,bookmarksopen,bookmarksnumbered,citecolor=red,urlcolor=red]{hyperref}
\usepackage[capitalize]{cleveref}
\usepackage{bbm}

\usepackage{makecell}

\newtheorem{prob}{Problem}
\newtheorem{lemma}{Lemma}
\newtheorem{example}{Example}
\newtheorem{theorem}{Theorem}
\newtheorem{corollary}{Corollary}
\newtheorem{definition}{Definition}

\newcommand{\Nature}{{\mathbb{N}}} 
\newcommand{\Rp}{{\mathbb{R}_{\geq 0}}}  
\newcommand{\Rpp}{{\mathbb{R}_{> 0}}}  
\newcommand{\overRp}{\overline{\mathbb{R}}_{\geq 0}}

\newcommand{\weight}{{\phi}}  
\newcommand{\weightmax}{{\phi_{\max}}}  
\newcommand{\weightmin}{{\phi_{\min}}}  
\newcommand{\numnode}{{|\nodes|}}
\newcommand{\graph}{{\mathcal G}}     
\newcommand{\nodes}{{\mathcal V}}                
\newcommand{\edges}{{\mathcal E}}                
\newcommand{\edgelength}{{A}}   
\newcommand{\Ggeo}{{|\graph|}}
\newcommand{\dist}{{d_{\Ggeo}}}  
\newcommand{\Ggeok}{{|\graph|^\numrobot}}
\newcommand{\distGk}{{d_{|\graph|^{\numrobot}}}} 
\newcommand{\distG}{{d_{|\graph|}}} 
\newcommand{\graphdiameter}{{D_{ \graph}}}

\newcommand{\numrobot}{{k}}        
\newcommand{\robots}{{\mathcal R}}      
\newcommand{\schedule}[1]{\pi_{#1}}    
\newcommand{\strategy}{\boldsymbol{\pi}}   
\newcommand{\strategyset}{\Pi}     

\newcommand{\latency}[2]{{L_{#1}^{#2}}}    
\newcommand{\maxlatency}[2]{{L}^{#1}(#2)}
\newcommand{\rvisitvatt}[2]{h_{#1}^{#2}}
\newcommand{\allvisitvatt}[1]{{h}^{#1}}
\newcommand{\objectivefm}[2]{J_{#1}(#2)}
\newcommand{\longlatency}[1]{{J_{\infty}({#1})}}
\newcommand{\optobjall}{J^*}
\newcommand{\optobjT}[1]{J^*_{#1}}
\newcommand{\optstrategyT}[1]{\strategyset_{#1}^*}
\newcommand{\inipos}{\boldsymbol{p}_0}
\newcommand{\intstrategyset}[1]{\strategyset_{#1}}

\newcommand{\torlance}{\delta_{\mathrm{tol}}}

\newcommand{\jointstate}{{\Xi}}  
\newcommand{\jointspace}{{\mathcal{X}}} 
\newcommand{\jointdist}{{d_{\jointspace}}}

\newcommand{\cmpctjointspace}[1]{{\mathcal{X}}_{#1}}

\newcommand{\rationalstrategyset}{\bar\strategyset}  

\newcommand{\perstrategy}{\strategy^{\mathrm{per}}}   
\newcommand{\perschedule}[1]{\schedule{#1}^{{\mathrm{per}}}}

\newcommand{\aveobj}{J^{\mathrm{ave}}}

\newcommand{\jointpsotion}{{\boldsymbol{p}}}  
\newcommand{\jointlatency}{{\boldsymbol{L}}}

\newcommand{\trajectory}{\xi}

\DeclareMathAlphabet{\mathbbb}{U}{bbold}{m}{n}

\begin{document}
	
	\runninghead{Wang \MakeLowercase{\textit{et al.}}: Multi-robot Persistent Monitoring: Theory and RL-based Solutions}
	
	\title{Minimizing Worst-Case Weighted Latency for Multi-Robot Persistent Monitoring: Theory and RL-Based Solutions}
	
	\author{Weizhen Wang\affilnum{1,2}, Ziheng Wang\affilnum{1}, Jianping He\affilnum{1,2}, Xinping Guan\affilnum{1,2}, and Xiaoming Duan\affilnum{1,2}}
	
	\affiliation{\affilnum{1} School of Automation and Intelligent Sensing, Shanghai Jiao Tong University, Shanghai, China\\
		\affilnum{2} Key Laboratory of System Control and Information Processing, Ministry of Education of China, Shanghai, China}
	
	\corrauth{Xiaoming Duan, 800 Dongchuan Rd, Shanghai, 200240, China.}
	
	\email{xduan@sjtu.edu.cn}
	
	\begin{abstract}
		We study multi-robot persistent monitoring on weighted graphs, where node weights encode monitoring priorities and edge weights encode travel distances. The goal is to design joint robot trajectories that minimize the worst-case weighted latency across all nodes over an infinite time horizon. The widely adopted worst-case latency objective evaluates team performance over the entire time horizon and therefore may fail to distinguish strategies with poor transient behavior but strong asymptotic performance. To address this limitation, we propose a family of tail-performance objectives that generalize the standard objective and study the resulting functional optimization problems. We establish several key theoretical properties, including the existence of optimal strategies, relationships among the proposed objectives and their corresponding optimization problems, approximation by periodic solutions to arbitrary accuracy, and reductions to event-driven decision models with discretized waiting time. 
		Building on these results, we construct an equivalent event-driven Markov decision process (MDP), called the Tail Worst-case Latency-Optimizing Markov Decision Process (TWLO-MDP), which reformulates the tail-performance objective as a standard average-cost criterion. 
		We then develop reinforcement-learning-based solution methods for the TWLO-MDP and introduce the multi-robot monitoring benchmark (M2Bench), a unified platform  that supports the evaluation and comparison of heuristic and learning-based monitoring algorithms. 
		Experiments on synthetic and realistic monitoring scenarios show that our methods effectively reduce the worst-case weighted latency and outperform representative baselines.
	\end{abstract}
	
	
	\maketitle

	\section{Introduction}
	Persistent monitoring is a key application area for many real-world robotic systems, with representative scenarios including urban crime surveillance~\citep{JP-EG-MB:2025}, post-disaster area monitoring~\citep{VK-MS-DK:2022}, anomaly detection~\citep{SW-JC-JM-JC-FM-PL-MV:2017}, ocean monitoring~\citep{RS-MS-SS-BJ-DR-GS:2011}, and forest fire detection~\citep{MM-HS-BA:2022}.  In these applications, a team of robots is required to repeatedly visit a set of important locations so that abnormal events can be detected  in a timely manner. This naturally leads to a graph-based multi-robot monitoring problem, where nodes represent locations of interest, edges represent feasible travel paths, and the robots must coordinate their long-term movements on the graph over an infinite time horizon. A  central performance measure for this problem is  latency, which quantifies how long each location remains unvisited. Since large latency at any important location may delay event detection, many studies adopt a worst-case perspective and seek coordinated multi-robot trajectories that minimize the maximum latency across the graph over time.  This problem has attracted significant attention over the past two decades~\citep{AM-GR-JZ-AD:2003,YC:2004,DP-RR:2011,LH-MZ-EH:2019,NB:2022}. Earlier studies formalized the worst-latency objective and established the computational hardness of the resulting optimization problems, including NP-hardness of exact optimization and APX-hardness of approximation in weighted variants~\citep{YC:2004,SA-EF-SS:2012}. Subsequent work has developed a range of solution methods, including hand-crafted monitoring heuristics, problem-specific structured strategies, and learning-based policies. Classical methods typically either rely on intuitive decision rules, such as visiting high-latency nodes, or exploit restricted graph topologies and strategy classes, such as graph partitions and recurrent monitoring patterns. The latter approaches, while often analytically tractable, may not readily extend to general monitoring scenarios. Learning-based approaches instead learn decision policies from interaction data, allowing  adaptation to graph structure, node priorities, and multi-robot coordination patterns. However, most existing learning formulations optimize surrogate rewards rather than an objective equivalent to the original worst-latency metric. 

	In this work, we study the persistent monitoring problem and introduce a family of tail-performance objectives that generalize the standard worst-latency objective and characterize long-run monitoring behavior after an initial transient phase. 
	We establish several key theoretical properties of the resulting optimization problems, including the existence and structural properties of optimal solutions, approximation by periodic solutions, and bounded performance loss under discretized waiting time. 
	Building on these results, we reformulate the original infinite-horizon functional optimization problem as an equivalent event-driven Markov decision process, termed the Tail Worst-case Latency-Optimizing Markov Decision Process (TWLO-MDP). This reformulation converts the non-additive worst-latency objective into a standard average-cost criterion, enabling reinforcement-learning-based solution methods. We further develop the multi-robot monitoring benchmark (M2Bench), a modular platform that integrates heuristic and RL-based monitoring methods and supports  unified training, evaluation, and comparison across prototypical monitoring environments.
	
	\textbf{Prior work:}
	The persistent monitoring problem has been widely studied in the literature and often appears under related names such as persistent surveillance, coverage, and patrolling.   In the simplest \textit{single-robot} setting with \textit{uniform node priorities}, \cite{YC:2004} proves that an optimal policy under the worst-latency criterion is a cyclic tour, and computing such a tour is equivalent to solving the classical NP-hard Traveling Salesman Problem (TSP)~\citep{RK:2010}. Thus, the problem is already computationally challenging even in its most basic form. For metric TSP, the classical Christofides algorithm achieves a $3/2$-approximation guarantee~\citep{NC:2022}, which was later slightly improved by~\cite{AK-Nk-OS:2021}. Moreover, \cite{MK-ML-RS:2015} show that there is no polynomial-time $(123/122-\varepsilon)$-approximation algorithm unless $\mathrm{P}=\mathrm{NP}$. A related extension additionally requires the robot to return to a depot for recharging after a fixed number of visits~\citep{SH-SR-DC:2020}, and the authors analyze the structure of optimal solutions under different recharging requirements and develop  efficient algorithms. 
    
	Beyond uniform-priority formulations, realistic monitoring tasks often involve locations with different levels of importance, urgency, or risk, naturally leading to weighted graph models. In this setting, nodes are assigned \textit{different priorities}, and the objective is to balance revisit frequencies according to their relative importance. \cite{SA-EF-SS:2012} show that the min-max weighted latency monitoring problem is APX-hard,  establish the existence of an optimal strategy for the single-robot heterogeneous-priority case, and prove that a periodic optimal solution always exists, although its period may be exponentially long. They also propose two approximation algorithms with  ratios $O(\log \frac{\weightmax}{\weightmin})$ and $O(\log \numnode)$, where $\frac{\weightmax}{\weightmin}$ is the ratio between the maximum and minimum node weights and $\numnode$ is the number of nodes in the monitoring  graph.

    Several recent works have further studied related weighted single-robot formulations. 	\cite{LG-TJ-TR:2024} consider the continuous Bamboo Garden Trimming Problem, which is mathematically equivalent to the single-robot monitoring problem, and propose two approximation algorithms with approximation ratios  $O\left(\log   \frac{\weightmax}{\weightmin} \right)$ and $O(\log \numnode)$, respectively.
	\cite{LC-JL-RK:2024} also propose an $O\left(\log \frac{\weightmax}{\weightmin}\right)$-approximation algorithm with an  improved constant factor compared with earlier logarithmic bounds. In addition, 
	\cite{DM-AS-LV:2022}  consider a related structured  weighted variant with binary node priorities, where locations are divided into priority and non-priority nodes. 
	While the single-robot setting already exhibits rich structure and significant computational difficulty, many practical monitoring applications require multiple robots to operate cooperatively in the same environment. 
	
	For \textit{multi-robot monitoring}, early work  focuses primarily on \textit{heuristic} strategies~\citep{NA-DU-PS:2011,DP-RR:2013,DP-RR:2013:iros,DP-RR:2016,AF-LI-DN:2017,AG-XL-QZ:2024}. These methods coordinate robots through mechanisms such as partition-based monitoring, distributed local decision-making, adaptive task allocation, and Bayesian adaptation, with the goal of achieving scalable and fault-tolerant monitoring behavior in practice. To support empirical  evaluation, several benchmarking and simulation frameworks have also been developed to integrate and compare representative monitoring heuristic algorithms in both simulation and more realistic experimental settings~\citep{DM-GR-PT:2009,DP-LI-AF:2018,AG-XL-QZ:2024}. Although these methods are often computationally efficient and effective in practice, their performance guarantees with respect to the worst-latency objective are usually limited or problem-specific. 
    
    Given the difficulty of the general multi-robot monitoring problem, another line of work studies  structured graph topologies or restricted strategy classes. Examples include open polylines~\citep{YE-AS-GK:2008}, chains, trees and cyclic graphs with uniform node priorities~\citep{FP-AF-FB:2012}, fences with distinct agent speeds~\citep{JC-LG-EK:2011,AD-AG-CT:2014,AK-MS:2020}, and non-intersecting tours with weighted locations~\citep{FP-JD-FB:2012}. These studies provide useful structural insights and, in some cases, exact or approximate algorithms, but their assumptions on the graph topology or strategy  structure limit their direct applicability to general weighted monitoring environments.
    
    For more general graph topologies, several approximation algorithms have been developed. \cite{PA-BD-HY:2020} study  the multi-robot monitoring problem in general metric spaces and propose an $O\!\left(\numrobot^2 \log \frac{\weightmax}{\weightmin}\right)$-approximation algorithm for the weighted setting, where $\numrobot$ is the number of robots. They also obtain stronger  results for patrol scheduling on the line $\mathbb{R}^1$: for uniform weights, an optimal solution always exists in the form of $k$ disjoint zigzag schedules and can be computed in polynomial time, whereas for heterogeneous weights they provide a $12$-approximation algorithm. In the unweighted case, 
	\cite{PA-MB-KB:2022} study  a class of \textit{cyclic solutions} in which the sites are partitioned into $\ell \le k$ groups and each group is patrolled by one or more robots evenly spaced along a TSP tour. They show that an optimal cyclic solution gives a $2(1-1/k)$-approximation to the global optimum, and that for $k=2$ an optimal cyclic solution is globally optimal. 
    Furthermore, \cite{LC-JL-RK:2024}  propose an algorithm based on multiclass minimum spanning forests for the weighted-node setting, improving the dependence on the number of robots and achieving an approximation ratio of $O\!\left(\numrobot \log \frac{\weightmax}{\weightmin}\right)$. More recently, \cite{DM-AS-LV:2025}  construct a class of strategies in which robots are only allowed to depart from nodes at times lying on a discrete time grid, and they prove that in this restricted class, there always exists a recurrent strategy whose objective value is at most \(\epsilon\) away from the optimum.
    
	Several related variants have also been studied.     
    \cite{AA-SS-SS:2019} investigate a dual problem in which the objective is to minimize the number of robots required to satisfy prescribed latency constraints at all vertices. This  line of work is further extended in~\citep{AA-SS-SS:2025}, which incorporates recharging constraints by requiring robots with limited battery capacity to periodically return to a depot while still minimizing the number of robots needed to satisfy the latency requirements. 
	Despite these advances, many  theoretical questions remain open for the general multi-robot weighted monitoring problem. In particular, strong existence, optimality, and structural results known for the single-robot case~\citep{SA-EF-SS:2012} do not directly extend to the multi-robot setting.

	Persistent monitoring can also be naturally viewed as a sequential decision-making problem and is therefore naturally amenable to reinforcement learning. Motivated by advances in deep reinforcement learning~\citep{VM-KK-DS:2015,VM-AB-KK:2016,JS-FW-PD:2017,MH-JM-DS:2018,TH-AZ-SL:2018} and multi-agent  reinforcement learning~\citep{TR-MS-SW:2020,PS-GL-AG:2017,CD-TG-SW:2020,CY-AV-YW:2022,OR-IB-WC:2019}, learning-based approaches have been increasingly explored for single- and multi-robot monitoring problems. Early learning-based studies mainly rely on tabular reinforcement learning. For example, \cite{HS-GR-VC-BR:2004} propose two Q-learning-based schemes, BBLA and GBLA, with state representations built from  local latency information. This work is later extended in \citep{FL-AK:2014} by incorporating additional coordination-aware features.
	More recent approaches use deep reinforcement learning to handle larger and more complex monitoring environments. Representative examples include value-based approaches such as DQN~\citep{MJ-LV-AS:2022}, as well as policy-gradient and multi-agent policy-gradient methods~\citep{LG-HP-XD-JH:2023,AG-YS-QZ:2024,JP-EG-MB:2025}. In particular, graph neural networks  have also been adopted to encode  graph structure and inter-agent interactions.
	A comparison of representative learning-based monitoring methods, including their reward designs, state representations, and learning algorithms, is provided in~\cref{table:RL}. Despite their empirical success, most existing learning formulations optimize surrogate rewards that are only loosely connected to the original monitoring objective. For  worst-case latency minimization, commonly used rewards such as instantaneous latency or its normalized variants are not equivalent to the trajectory-level worst-latency metric, and may therefore induce policies that are suboptimal under the true evaluation criterion. 
		This mismatch arises because  standard discounted or average-reward objectives are additive over time, whereas worst-case latency is a non-additive functional of the entire trajectory. 
		To address this  gap, we construct an event-driven MDP formulation whose objective is equivalent to the original worst-latency monitoring problem, so that solving the resulting MDP directly optimizes the intended criterion rather than a surrogate reward.

	Beyond latency-minimization formulations, persistent monitoring  also includes several related modeling paradigms. One line of work studies dynamic environments where uncertainty accumulates over time and is reduced through robot motion and sensing~\citep{SS-MS-DR:2012,CC-XL-XD:2013,NZ-XY-SA-CC:2018,JP-DT-CS:2019}. These problems are often formulated as continuous-time control or coverage-planning problems, where trajectories, speeds, dwell times, or coverage paths are optimized to regulate the long-term growth of the uncertainty field. Another line of work designs stochastic monitoring strategies, often based on Markov chains~\citep{XD-FB:2021}, with objectives related to surveillance  efficiency~\citep{RP-AC-FB:2016,GD_FB_JM:2023,WW-XD-JP:TCNS2025}, unpredictability~\citep{MG-SJ-FB:2019,XD-MG-FB:2020,XD-WW-RY:TAC2026,WW-XD-JH:TAC2025}, or trade-offs between the two~\citep{JG-JB:2005,WW-JH-XD:LCSS2024}. Game-theoretic variants further model strategic intruders or attackers and design monitoring strategies against such adversarial behaviors~\citep{SA-AM-KP:2011,AA-SS:2018,XD-DP-FB:2021,DK-AK-VV:2021,MD:KG:RS:2023,AN-MG-TP:2024,YJ-GD-XD-JM-FB:2025}.


	

	\textbf{Contributions:}
	Our core contributions are as follows:
	\begin{enumerate}
		\item We introduce a family of tail-performance objectives for multi-robot persistent monitoring on weighted graphs. 
		This formulation generalizes the conventional worst-case weighted latency criterion by excluding an initial transient phase, thereby capturing the long-run monitoring performance of robot teams.
		\item We establish theoretical properties of the proposed objective family and the associated  strategy classes. 
		In particular, we prove the existence of optimal strategies, invariance of the optimal value with respect to the transient horizon, monotonicity of optimal strategy sets, reachability of optimal long-run performance from arbitrary initial configurations, approximation by periodic solutions, and bounded performance loss under discretized waiting time.
		\item We reformulate the  continuous-time trajectory optimization problem as an equivalent event-driven MDP, termed the TWLO-MDP. 
		By augmenting the state with a historical worst-latency tracker and an elapsed-time counter, the TWLO-MDP transforms the non-additive tail objective into a standard average-cost criterion while preserving optimality.
		\item We develop reinforcement-learning-based solution methods for the TWLO-MDP. The proposed learning pipeline incorporates several problem-specific training techniques, including a time-normalized training objective, transparent generalized advantage estimation for handling asynchronous event transitions,  parameter sharing with role encoding, and imitation-based initialization.

    \item We build M2Bench, a modular benchmark platform for multi-robot persistent monitoring. 
    M2Bench integrates heuristic and reinforcement-learning-based monitoring algorithms within a unified workflow for configuration-driven training, evaluation, comparison, visualization, logging, and hyperparameter search.

    \item We conduct extensive experiments on special, synthetic, and real-map  benchmark environments. 
    The experiments compare representative heuristic and learning-based baselines, validate scalability and robustness under varying deployment conditions, and quantify the contribution of the main training components through ablation studies.
		
	\end{enumerate}
	
	\textbf{Organization:}
	The rest of the paper is organized as follows. \cref{sec:definition}  formulates the multi-robot persistent monitoring problem on weighted graphs and introduces the family of tail-performance objectives. \cref{sec:theoreticalresults}  presents the main theoretical results, including optimality properties, periodic approximation, and discretized waiting-time analysis. \cref{sec:MDPandRL}  develops the TWLO-MDP reformulation and the corresponding reinforcement-learning-based solution methods. \cref{section:platform} describes M2Bench, the proposed unified platform for implementing, training, and evaluating monitoring algorithms. \cref{section:experiments} reports experimental results and comparisons. Finally, \cref{section:conclusion} concludes the paper.
	
	\section{Problem formulation}\label{sec:definition}
	
	\subsection{Monitoring environment and performance metrics}\label{subsection:environment_metrics}
	The monitoring environment is modeled by an undirected connected   graph $\graph = (\nodes,\edges,\edgelength,\weight)$,
	where $\nodes$ is the set of nodes representing critical locations, \(\edges\) is the set of edges connecting these locations, \(\edgelength :\edges\to\Rpp\) assigns each edge a positive length corresponding to the physical distance between its endpoints, and \(\weight:\nodes\to\Rpp\) assigns each node a positive weight reflecting its monitoring priority. Let $\weightmax:= \max_{v\in \nodes} \weight(v)$ and $\weightmin:= \min_{v\in \nodes} \weight(v)$ denote the maximum and minimum node priorities, respectively. 
To describe continuous robot motion along edges, we  construct from $\graph$ a metric graph $\Ggeo = \left(\bigsqcup_{e\in\edges}[0,\edgelength(e)]\right)\big/\sim$, where $\bigsqcup$ denotes the disjoint union, each edge $e\in \edges$ is represented as a closed interval of length \(\edgelength({e})\), and $\sim$ identifies interval endpoints corresponding to the same incident node~\citep{DM:2021}. 
    Equipped with the shortest-path distance \(\dist(\cdot,\cdot)\), the pair \((\Ggeo,\dist)\) forms a  metric space. Thus, robots may be located not only at nodes but also at arbitrary points along edges. We let $\graphdiameter$ denote the diameter of  $\graph$.

	A team of \(\numrobot\) robots, indexed by \(\robots=\{r_1,\dots,r_\numrobot\}\), is deployed to persistently monitor the environment \(\graph\). Each robot moves on the metric graph \(\Ggeo\) with maximum speed   one and may wait  at any point for an arbitrary duration. The motion of robot \(r \in \robots\) is described by a schedule \(\schedule{r}:\Rp \to \Ggeo\), which specifies its position over time. The joint behavior of the robot team is  represented by the trajectory
	\[
	\strategy(t) = (\schedule{r_1}(t),\dots,\schedule{r_\numrobot}(t)) \in \Ggeo^\numrobot, \quad t\ge0.
	\]
	Since $\nodes$ and $\edges$ are finite  and all edges have finite length, the metric graph $\Ggeo$ equipped with the metric $\dist$ is compact~\citep[Theorem 3.8]{JK:2024}. Moreover, since finite products of compact spaces are compact, $\Ggeok$ is compact under the product topology~\citep[Theorem 3.2.4]{RE:1989}. 
	For $\boldsymbol{p}=(p_1,\dots,p_\numrobot) \in \Ggeok$ and $ \boldsymbol{p}'=(p'_1,\dots,p'_\numrobot)\in \Ggeo^\numrobot$, define the sup  metric: 
	\begin{gather*}
		\distGk(\boldsymbol{p},\boldsymbol{p}') := \max_{i\in \{1,\dots,\numrobot\}} \dist(p_i,p'_i).
	\end{gather*}
    The topology induced by $\distGk$ coincides with the product topology~\citep[Theorem 4.5.1]{MS:2007}.
	Hence, $(\Ggeo^\numrobot,\distGk)$ is a compact metric space.

    We define the set of feasible strategies  by
	\begin{multline}\label{eq:jointset}
		\strategyset  :=  \{\strategy \in \mathfrak{C}(\Rp, \Ggeo^\numrobot): \\\distGk(\strategy(t),\strategy(t'))\leq |t-t'|,\forall t,t' \},
	\end{multline}
    where $\mathfrak{C}(\Rp, \Ggeo^\numrobot)$ denotes the set of continuous functions from $\Rp$ to $\Ggeo^\numrobot$. The Lipschitz condition encodes the unit-speed constraint of the robot team and  is equivalent to $\dist(\schedule{r}(t),\schedule{r}(t'))\leq |t-t'|$ for all $r\in \robots$.  A joint strategy \(\strategy\) is called periodic if there exist $t^*\geq 0$ and $0<W<\infty$ such that \(\schedule{r}(t^*+t+W)=\schedule{r}(t^*+t)\) for all \(t\in\Rp\) and all \(r\in\robots\).

	A node \(v\in\nodes\) is said to be visited at time $t$ if \(\schedule{r}(t)=v\) for some robot \(r\in\robots\).
	To quantify the monitoring quality at each node, we use \emph{latency} as the key performance metric. For each node \(v\in\nodes\), robot $r\in\robots$, and time $t\geq0$, define the most recent visit time of robot \(r\) to \(v\) up to time \(t\) as
	\begin{equation} \label{eq:tauvr}
		\rvisitvatt{r}{\strategy}(v,t):=
		\begin{cases}
			\sup\{ t'\leq t : \schedule r(t')=v \}, & \text{if such $t'$ exists}, \\
			0, & \text{otherwise}.
		\end{cases}
	\end{equation}
	The most recent visit time of the entire robot team to node \(v\) is then
	\begin{equation}\label{eq:deftauv1}
		\allvisitvatt{\strategy}(v,t):= \max_{r\in\robots} \rvisitvatt{r}{\strategy}(v,t).
	\end{equation}
	The latency of node \(v\) under strategy $\strategy$ at time $t$ is defined as
	\begin{equation}\label{definition:latency}
		\latency{v}{\strategy}(t) := t - \allvisitvatt{\strategy}(v,t),
	\end{equation}
	which measures the time elapsed since the most recent visit to \(v\) by any robot. By convention, we assume the initial latency \(\latency{v}{\strategy}(0)=0\) for all nodes $v$. To evaluate the instantaneous monitoring performance of the whole team, we account for node priorities and adopt a worst-case perspective over the graph. Specifically, the instantaneous worst-case  weighted latency is defined as $\maxlatency{\strategy}{t} := \max_{v\in\nodes}\, \weight(v)\,  \latency{v}{\strategy}(t)$. 
	
	\subsection{Tail objective formulation}
	We now extend the instantaneous latency metric to evaluate the overall quality of a  strategy $\strategy$ over time. A natural approach, adopted in existing studies such as~\citep{SA-EF-SS:2012,PA-MB-KB:2022,PA-BD-HY:2020}, is to take the supremum of $\maxlatency{\strategy}{t}$ over the entire time horizon, i.e., $\sup_{t\geq 0} \maxlatency{\strategy}{t}$. This quantity corresponds to the worst-case weighted latency attained  over all time. However, it is sensitive to transient effects at the beginning of a monitoring mission. For instance, if robots start from unfavorable initial positions, some nodes may experience large latency values during the early stage, even if the robot team later settles into a regular monitoring pattern with good long-run performance. The following example illustrates this issue.
	\begin{example}[(Good long-run strategies with bad transients)]\label{example} Consider the graph in~\cref{fig:example}: nodes $\{1,2,3\}$ form a triangle with unit edge lengths, node~$4$ is connected to node~$1$ by an edge of length $5$, and all node weights are equal to one.
		\begin{figure}[h]
			\centering
			\includegraphics[width=0.8\linewidth]{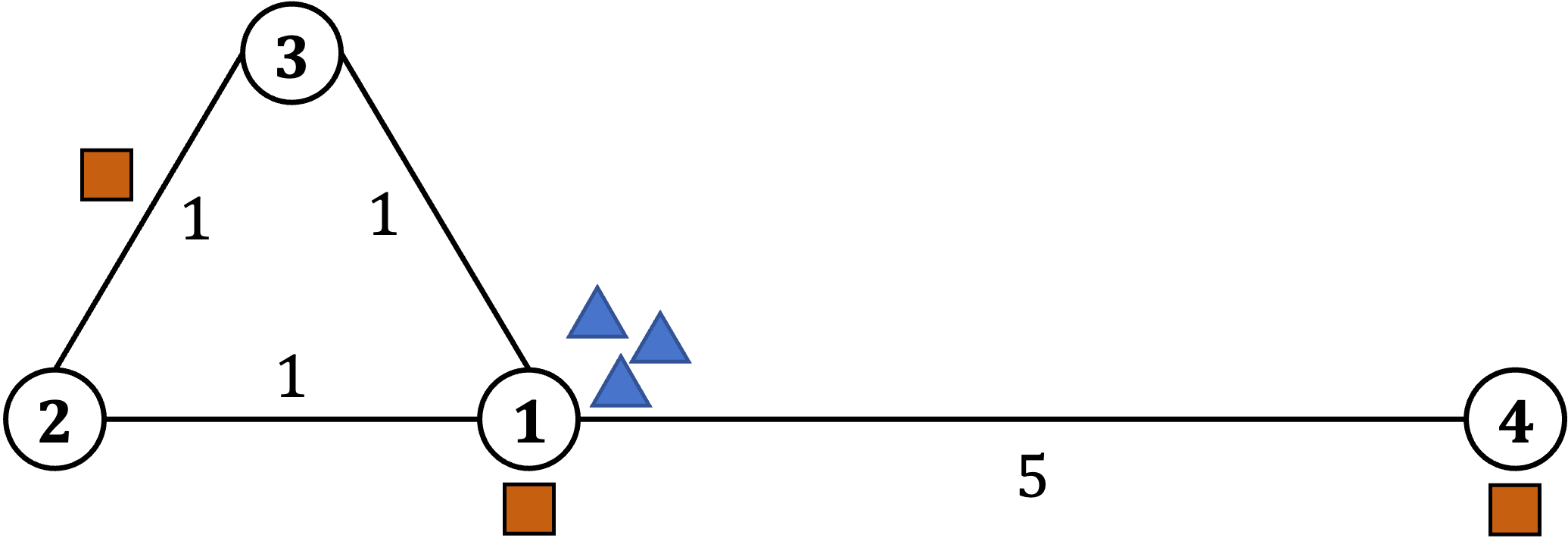}
			\caption{Long-edge graph used in \cref{example}  to illustrate the effect of transient behavior on the worst-latency objective.}
			\label{fig:example}
		\end{figure}
        
		We compare three feasible  strategies for monitoring this graph.
		
		\textbf{Strategy $\strategy_1$.} The robots start from the  positions shown by the  orange squares: one robot is located at node~$1$, one at the midpoint of edge $(2,3)$, and one at node~$4$. The robot at node~$4$ remains stationary, while the other two robots move at full speed along the cycle $1\!\to\!3\!\to\!2\!\to\!1$. Under this strategy, 
		\[
		\sup_{t\ge0}\maxlatency{\strategy_1}{t}=\tfrac{3}{2}.
		\]

		\textbf{Strategy $\strategy_2$.} The robots start from the  positions shown by the blue triangles: all robots are initially located at node~$1$, corresponding to a typical single-base deployment. One robot moves at full speed to node~$4$ and then remains there. A second robot immediately starts moving at full speed along the cycle  $1\!\to\!3\!\to\!2\!\to\!1$. The third robot waits at node~$1$ for $3/2$ units of time and then follows the same cycle at full speed. In this case, the instantaneous worst-case weighted latency $\maxlatency{\strategy_2}{t}$ approaches $5$ during the initial transient period  $t\in[0,5)$, due to the delayed first visit to node $4$, and becomes  $3/2$ thereafter. Hence, 
		\[
		\sup_{t\ge0}\maxlatency{\strategy_2}{t}=5,
		\quad\text{and}\quad
		\sup_{t\ge5}\maxlatency{\strategy_2}{t}=\tfrac{3}{2}.
		\]
		
		\textbf{Strategy $\strategy_3$.} All robots again start at node~$1$. One robot moves to node~$4$ and remains there, one robot moves to node~$3$ and then shuttles between nodes~$2$ and~$3$ at full speed, and the third robot stays at node~$1$. Then,  
		\[
		\sup_{t\ge0}\maxlatency{\strategy_3}{t}=5,
		\quad\text{while}\quad
		\sup_{t\ge5}\maxlatency{\strategy_3}{t}=2.
		\]

        This example shows that the global supremum $\sup_{t\ge0}\maxlatency{\strategy}{t}$ may fail to distinguish strategies with different long-run behaviors.  Strategies $\strategy_2$ and $\strategy_3$ have the same value under this metric, even though, after the initial transient phase, $\strategy_2$ reaches the same long-run monitoring pattern as $\strategy_1$ and achieves the same long-run latency $3/2$, whereas $\strategy_3$ has long-run latency $2$. Thus, excluding an initial transient period, for example by considering $\sup_{t\ge5}\maxlatency{\strategy}{t}$, provides a more faithful assessment here.
        
	\end{example}

	Motivated by \cref{example}, we now introduce a family of tail-performance metrics that evaluate monitoring quality after an initial transient period. These metrics retain the worst-case perspective of the classical objective while focusing on the long-run behavior of the robot team. Let $\overRp:= \mathbb{R}_{\geq 0} \cup \{+ \infty\}$. 
	\begin{definition}[(Tail worst-case weighted latency)]
		For any \(T \in \overRp\),   the tail worst-case weighted latency of a strategy \(\strategy\) is defined by
		\begin{equation}\label{eq:defofobj}
			\objectivefm{T}{\strategy}
            :=
            \begin{cases}
                \sup_{t\geq T} \maxlatency{\strategy}{t},& \textup{for }T<\infty,\\
                \limsup_{t\to+\infty}\maxlatency{\strategy}{t},&\textup{for }T=+\infty.
            \end{cases}
		\end{equation}
	\end{definition} 
	Note that the family \(\{\objectivefm{T}{\strategy}\}_{T \in \overRp}\) generalizes the  performance metric commonly adopted in the persistent monitoring literature. In particular, the classical worst-case weighted latency is recovered by taking $T=0$. Finite values of $T>0$ exclude an initial transient period and evaluate the worst-case weighted latency from time $T$ onward, while the case $T=+\infty$ captures the asymptotic tail performance of the strategy. Since the limit superior  is  well-defined in the extended nonnegative real line $\overRp$~\citep[Theorem 3.1.7]{WR:1976}, $\objectivefm{\infty}{\strategy}$ is  well-defined. 
    
    
	With these definitions in place, the multi-robot monitoring problem can  be formally stated as follows.
	\begin{prob}[(Multi-robot monitoring problem)]\label{prob:monitoring-problem}
		Given an undirected connected  graph \(\graph=(\nodes,\edges,\edgelength,\weight)\), a robot team \(\robots=\{r_1,\dots,r_{\numrobot}\}\), and a tail parameter  $T\in\overRp$, the multi-robot monitoring problem is to find a feasible  joint strategy \(\strategy^* \in \strategyset\) that minimizes the tail worst-case weighted latency defined in~\eqref{eq:defofobj}, namely, 
		$\strategy^* \in \underset{{\strategy \in \strategyset}}{\arg \min}\objectivefm{T}{\strategy}$.
	\end{prob}

	\subsection{Rationale for the tail-performance family}\label{subsec:rationale}
	The family of objective functions $\{ \objectivefm{T}{\strategy} \}_{T\in\overRp}$ introduced in~\eqref{eq:defofobj} 
	serves as the foundation for our subsequent  analysis. 
	Before presenting the formal theoretical results, we summarize several key properties that motivate this formulation and clarify its role in the development of the paper. 
	The formal statements are postponed to~\cref{sec:theoreticalresults}.

	\textbf{1) Existence of optimal solutions.}  
	For every $T\in\overRp$, there exists at least one optimal strategy $\strategy_T^*\in\strategyset$ that attains the minimum of $\objectivefm{T}{\cdot}$. 
	This property ensures that each optimization problem induced by the tail-performance family is well posed.
	
	\textbf{2) Consistency of optimal values.}  
	Let $\optobjT{T} := \inf_{\strategy\in\strategyset}\objectivefm{T}{\strategy}$.  
	We will show that $\optobjT{T}=\optobjall$ for all $T\in\overRp$. Thus, although different values of $T$ evaluate different portions of a trajectory, all members of the family yield the same optimal value. In this sense, excluding an initial transient period does not change the best achievable long-run monitoring performance.
	
	\textbf{3) Monotonicity of optimal strategy sets.}  
	Denote by $\optstrategyT{T}$ the set of optimal strategies for $\objectivefm{T}{\cdot}$.  
	For any $0\leq T_1<T_2\leq+\infty$, we will prove that
	$\optstrategyT{0}\subseteq \optstrategyT{T_1}\subseteq \optstrategyT{T_2}$.
	Hence, any strategy that is optimal under the classical criterion ($T=0$) remains optimal when an initial transient period is excluded. Increasing $T$ can only enlarge the optimal strategy set, thereby allowing strategies with different transient behaviors but the same tail performance.
	
	\textbf{4) Reachability from fixed initial configurations.}  
	Given an initial configuration $\inipos\in\Ggeok$, let $\intstrategyset{\inipos}$ denote the set of feasible strategies starting from $\strategy(0)=\inipos$.  
	We will show that, for every $T> \graphdiameter+\frac{\optobjall}{\weightmin}$, there exists a tail-optimal strategy starting from $\inipos$, namely,  $\optstrategyT{T}\cap\intstrategyset{\inipos}\neq\emptyset$.
	This property is particularly important in practical monitoring missions, where the initial robot deployment is typically fixed. It shows that, after a sufficiently long transient period, the optimal tail performance can still be achieved from any initial configuration. 
	
	Together, these properties show that the tail-performance family is mathematically well posed, consistent with the classical worst-case objective, and better suited to fixed-start monitoring scenarios. They also provide the theoretical basis for the MDP reformulation and reinforcement-learning-based solution methods developed later in the paper.
	
	\section{Theoretical results}\label{sec:theoreticalresults}
	
	\subsection{Existence and properties of optimal solutions}\label{subsection:existence}
	\cref{prob:monitoring-problem} is an infinite-dimensional optimization problem over the space of feasible  strategies. Before proving the existence of an optimal strategy, we first note that the optimal value is finite. Indeed, since the graph is finite and connected, one can construct a traveling-salesman tour of length $\ell_{\rm{TSP}}$ that visits all nodes. Placing the \(\numrobot\) robots uniformly along this tour and letting them move at unit speed yields a feasible periodic strategy whose worst-case weighted latency is at most  $\weightmax \ell_{\rm{TSP}} / \numrobot <\infty$. Therefore, $\inf_{\strategy\in\strategyset} \objectivefm{T}{\strategy}<\infty$.
	However, the existence of a strategy that actually attains this infimum is nontrivial. We establish this property by applying a generalized extreme value theorem to the tail-performance functional.
	
	\begin{theorem}[(Existence of an optimal strategy for \cref{prob:monitoring-problem})]\label{thm:existence}
		Consider \cref{prob:monitoring-problem} with $T\in \overRp$. 
		Then there exists at least one feasible  strategy $\strategy^*\in\strategyset$ such that
		\[
		J_T(\strategy^*) \;=\; \inf_{\strategy\in\strategyset} J_T(\strategy).
		\]     
	\end{theorem}

	\begin{proof}
        
		The detailed proof is postponed   to Appendix~\ref{appendix:existence}.
	\end{proof}
	The existence result above shows that \cref{prob:monitoring-problem}  admits an optimal solution for every $T\in \overRp$. We next establish several structural properties of the optimal values and optimal strategy sets induced by the tail-performance family. 
	
	\begin{theorem}[(Properties of~\cref{prob:monitoring-problem})]\label{thm:properties}
		Consider the multi-robot monitoring problem in~\cref{prob:monitoring-problem} with the tail objective family 
		$\{\objectivefm{T}{\strategy}\}_{T\in\overRp}$ defined in~\eqref{eq:defofobj}. Then 
		the following properties hold:
		\begin{enumerate}
			\item The optimal value is independent of the tail parameter $T$, namely,
			$\optobjT{T}=\optobjall$ for all $T\in \overRp$. \label{item:independence}
			\item  The optimal strategy sets are monotone with respect to $T$. Specifically, for any $T_1,T_2\in\overRp$ with $T_1<T_2$,
			$ \optstrategyT{T_1}\subseteq\optstrategyT{T_2}$.  \label{item:including}
			
			\item 
			For any initial configuration $\inipos\in\Ggeok$, there exists a tail-optimal strategy starting from $\inipos$  whenever the transient horizon is sufficiently large. Specifically, we have  $\optstrategyT{T}\cap\intstrategyset{\inipos}\neq\emptyset$ for all $T > \graphdiameter+\frac{\optobjall}{\weightmin}$. \label{item:accessible}
			
		\end{enumerate}
	\end{theorem}
	\begin{proof}
		The proof is postponed to Appendix~\ref{appendix:properties}.
	\end{proof}
	The results above formalize the rationale in \cref{subsec:rationale}: the proposed tail-performance family is well posed, consistent with the classical worst-case objective, and suitable for characterizing long-run monitoring quality. 
	
	\subsection{Approximation by Periodic Strategies} \label{subsection:periodic}
The results in~\cref{sec:theoreticalresults} establish the existence and structural properties of optimal strategies in the full feasible strategy space. However, such strategies are defined over an infinite time horizon, making them difficult to represent, construct, or analyze directly. Periodic strategies provide a more tractable class: after a finite transient phase, the joint behavior of the robot team repeats with a fixed period. The next theorem shows that restricting attention to periodic strategies does not lose optimality at the level of approximation. Specifically, for any feasible strategy $\strategy$ with finite tail performance, i.e., $\objectivefm{T}{\strategy}<\infty$, and for any prescribed accuracy $\torlance>0$, one can construct a periodic strategy $\perstrategy$ whose tail worst-case weighted latency is within $\torlance>0$ of that of $\strategy$. Consequently, although an exactly optimal periodic strategy need not be guaranteed, periodic strategies are sufficient to approximate the optimal value arbitrarily closely.
    

	\begin{theorem}[(Approximation by periodic strategies)]\label{thm:periodic}
		For any $T\in\overRp$, any feasible strategy $\strategy \in \strategyset$ with $\objectivefm{T}{\strategy} <\infty$, and any tolerance $\torlance>0$, there exists a periodic strategy $\perstrategy \in \strategyset$ such that
		\begin{equation}\label{eq:thmbound}
			\objectivefm{T}{\perstrategy} 
			\;\le\;
			\objectivefm{T}{\strategy} + \torlance.
		\end{equation}
		
	\end{theorem}
	\begin{proof}
		We postpone the proof to Appendix~\ref{appendix:periodic}.
	\end{proof}
	
	\cref{thm:periodic} shows that periodic strategies are sufficient for approximating the optimal value of the tail worst-case weighted latency arbitrarily closely. This structural property is useful because periodic strategies can be represented by finite repeating patterns, making both theoretical analysis and algorithmic design more tractable. In particular, although the original problem is posed over an infinite time horizon, the theorem shows that near-optimal performance can always be achieved within the class of periodic strategies. A related approximation result was recently established in~\citep{DM-AS-LV:2025} under a discretized recurrent-strategy framework. In their formulation, edge traversal times are fixed by the graph weights, and the construction first aligns departure times to a temporal grid determined by a discretization constant. In contrast, our result is stated directly in the continuous-time feasible strategy space and does not require a prior temporal discretization of robot trajectories.
    

	\subsection{Event-driven decisions, discretized waiting time, and fixed initial configuration} \label{subsection:discretized}
	We next introduce a more structured class of strategies that will serve as the basis for our subsequent event-driven reformulation. 
    A feasible  strategy $\strategy$ is called rational if each robot always traverses edges at maximum speed and may  wait only at nodes. Let $\rationalstrategyset \subset \strategyset$ denote the set of such rational strategies. Restricting attention to $\rationalstrategyset$ does not change the optimal value of the monitoring problem. Indeed, whenever a robot moves along an edge slower than unit speed or waits at an interior point of an edge, we can instead let the robot traverse the edge at unit speed and transfer the remaining time to waiting at the destination node. This modification never delays any node visit and therefore does not increase the worst-case weighted latency. 
    
	
	For rational strategies, robot motion evolves deterministically between consecutive \emph{events}. An event occurs whenever at least one robot arrives at a node or completes a waiting period at a node. Let $0=t_0<t_1<t_2<\cdots$ denote the increasing sequence of distinct event times, with simultaneous events grouped at the same time. Under this canonical event representation, Zeno behavior cannot occur. Indeed, consecutive waiting periods of the same robot at the same node can be merged into a single waiting period, so only the ending time of the merged waiting period needs to be recorded as an event. Hence, between two recorded waiting events of the same robot, the robot must traverse at least one edge. Since the graph is finite and all edge lengths are positive, every such traversal takes at least the minimum edge length. Therefore, each robot can generate only finitely many recorded events in any finite physical-time interval. It is then sufficient to record robot positions and node latencies at event times. Between two consecutive event times, each robot either moves along an edge at unit speed or remains stationary at a node, and the latency dynamics are fully determined: latencies of unvisited nodes increase linearly with elapsed time, while nodes occupied or visited by robots remain at, or are reset to, zero. Thus, the continuous-time evolution between events can be reconstructed from the event-time state information. Moreover, over each interval $[t_j,t_{j+1})$, the supremum of the instantaneous worst-case weighted latency  is  attained in the left-limit sense at $t_{j+1}^-$. Hence, for rational strategies,
	\[ 
    	\objectivefm{T}{\strategy}
            =
            \begin{cases}
                \sup_{\{j:\,t_j>T\}}\maxlatency{\strategy}{t_j^-},& \textup{for }T<\infty,\\
                \limsup_{j\to+\infty}\maxlatency{\strategy}{t_j^-},&\textup{for }T=+\infty.
            \end{cases}
	\]
    Therefore, the continuous-time monitoring problem can be represented through event-driven decision times without losing the information needed to evaluate the tail worst-case weighted latency. 
	

    This event-driven representation of rational strategies isolates the only continuous decision variable: the waiting duration at a node. To obtain a finite action space, we discretize this waiting duration decision. Given a resolution $\Delta>0$, we allow a single waiting action of duration $\Delta$. Longer waiting periods can be represented by repeatedly selecting this $\Delta$-waiting action over consecutive decision steps. After this discretization, each available robot chooses from a finite set of movement actions, determined by the neighboring nodes, together with one discrete waiting action. As shown later, this discretization yields a finite-action strategy class while introducing only a controllable approximation error.
    
	To further align the formulation with practical deployment scenarios, we fix an initial configuration $\inipos\in\nodes^\numrobot$ for the robot team, where all robots are initially located at nodes. This setting reflects common applications in which robots begin operation from predetermined deployment sites such as charging docks or task-specific starting posts. Let $\rationalstrategyset^\Delta_{\inipos}$ denote the set of rational strategies that start from $\inipos$ and use discretized waiting resolution $\Delta$. 
	\begin{prob}[(Multi-robot monitoring with discretized waiting time and fixed initial configuration)]\label{prob:monitoring-discrete}
		Given an undirected connected  graph \(\graph=(\nodes,\edges,\edgelength,\weight)\), a robot team \(\robots=\{r_1,\dots,r_{\numrobot}\}\), a tail parameter  $T\in\overRp$, an initial configuration $\inipos\in\nodes^\numrobot$, and a waiting-time resolution $\Delta$, the discretized multi-robot monitoring problem is to find a rational  strategy in $\rationalstrategyset^\Delta_{\inipos}$ that minimizes the tail worst-case weighted latency, namely,  		\begin{equation}\label{eq:monitoring-discrete}
			\strategy^*_{\inipos,\Delta} 
			\in 
			\underset{\strategy \in \rationalstrategyset^\Delta_{\inipos}}{\arg\min}
			\;
			\objectivefm{T}{\strategy}.  
		\end{equation}
	\end{prob}
	We first show that the optimization problem over $\rationalstrategyset^\Delta_{\inipos}$ still admits an optimal solution: 
	\begin{lemma}[(Existence of an optimal solution in the discretized strategy class)]\label{lemma:existenceofdiscretized}
		For any finite tail parameter $T<\infty$, any waiting-time resolution $\Delta>0$, and any initial configuration $\inipos\in\nodes^\numrobot$, 
		the discretized monitoring problem, \cref{prob:monitoring-discrete}, admits an optimal solution. That is, there exists $\strategy_{\inipos,\Delta}^{*}\in \rationalstrategyset^\Delta_{\inipos}$ such that
		\[
		\objectivefm{T}{\strategy_{\inipos,\Delta}^{*}}
		=
		\inf_{\strategy\in \rationalstrategyset^\Delta_{\inipos}}\objectivefm{T}{\strategy}.
		\]
	\end{lemma} 
	\begin{proof}
		The proof is postponed to Appendix~\ref{subsection:existenceofdiscretized}.
	\end{proof}
	Our next result shows that discretizing the waiting time introduces only a bounded approximation error.
	
	\begin{theorem}
		[(Suboptimality bound for strategies with discretized waiting time)] 
		\label{thm:suboptimality}
		For any initial configuration $\inipos\in\nodes^\numrobot$ and any waiting-time resolution $\Delta>0$, let $T\in\overRp$ satisfy $T > \graphdiameter+\frac{\optobjall}{\weightmin}+\Delta$. Then there exists a strategy $\strategy_{\inipos,\Delta} \in \rationalstrategyset^{\Delta}_{\inipos}$ such that 
		\[
		\objectivefm{T}{\strategy_{\inipos,\Delta}}
		\le 
		\optobjall
		+ 2\weightmax\Delta.
		\]
	\end{theorem}
	\begin{proof}
		The proof is postponed to Appendix~\ref{appendix:discretization}.
	\end{proof}
\cref{thm:suboptimality} shows that waiting-time discretization yields a controlled approximation error that scales linearly with the resolution $\Delta$. Thus, by choosing $\Delta$ sufficiently small, the discrete-action formulation remains faithful to the original continuous-time monitoring problem while providing a finite action space for the subsequent MDP formulation.

    \section{MDP reformulation and reinforcement learning}\label{sec:MDPandRL}
	
	\subsection{MDP reformulation} \label{subsection:MDP}

    Building on the event-driven structure developed in the previous section, we reformulate the original multi-robot monitoring problem, \cref{prob:monitoring-problem}, with a fixed finite tail parameter \(T<+\infty\) as an event-driven Markov decision process $\mathcal{M}_T=(\mathcal{S},\mathcal{A},{P},C)$, termed the Tail Worst-case Latency-Optimizing MDP (TWLO-MDP). The limiting case $T=+\infty$ is used only as a theoretical benchmark for asymptotic tail performance. In the continuous TWLO-MDP, decisions are made only at event times,  and waiting durations can be arbitrary positive real numbers. The key idea is to augment the state with sufficient historical information so that the non-additive tail worst-case weighted latency can be represented through a standard average-cost objective.
    

	
	\textbf{State space $\mathcal{S}$}: The MDP is \emph{event-driven}: each decision step corresponds to an event time at which at least one robot becomes available after completing a travel or waiting segment.  The state at the event time $t_n$ is $s_n = (\boldsymbol{p}_n,\boldsymbol{L}_n,z_n,\eta_n) \in \mathcal{S}$, where $\boldsymbol{p}_{n}=(p_{n,1},\dots,p_{n,\numrobot})$ encodes the joint positions of all robots. For each robot $r$, $p_{n,r}=(v^{\mathrm{from}}_{n,r},v^{\mathrm{to}}_{n,r},\delta_{n,r})$ contains the starting node $v^{\mathrm{from}}_{n,r}$, destination node $v^{\mathrm{to}}_{n,r}$, and the remaining duration of the current action $\delta_{n,r}$. The vector $\boldsymbol{L}_n=(L_{n,1},\dots,L_{n,\numnode})$ records the weighted node latencies at time $t_n$. The scalar $z_n = \sup_{T\leq t\leq t_n} L(t)$ tracks the largest weighted latency observed since the tail parameter $T$, and $\eta_n$ is a capped elapsed-time counter used to determine whether the evaluation phase  has begun.
	
	\textbf{Action space $\mathcal{A}$}: The joint action at event step $n$ is denoted by $\boldsymbol{a}_n = (a_{n,r})_{r\in \robots} \in \mathcal{A}_1\times\cdots\times \mathcal{A}_\numrobot$. For each robot $r$, if $\delta_{n,r}=0$, then it has completed its previous action, is available for decision-making, and is located at $v^{\mathrm{cur}}_{n,r}=v^{\mathrm{to}}_{n,r}$. Its action is represented as $a_{n,r}= (v^{\mathrm{tgt}}_{n,r},\tau_{n,r})$, where \(v^{\mathrm{tgt}}_{n,r}\in\{v^{\mathrm{cur}}_{n,r}\}\cup\mathcal{N}(v^{\mathrm{cur}}_{n,r})\) specifies the next target node and $\tau_{n,r}$ specifies the corresponding action duration. If $v^{\mathrm{tgt}}_{n,r}\in\mathcal N(v^{\mathrm{cur}}_{n,r})$, then robot $r$ moves to the neighboring node $v^{\mathrm{tgt}}_{n,r}$, and the action duration is fixed as $A(v^{\mathrm{cur}}_{n,r},v^{\mathrm{tgt}}_{n,r})$. If $v^{\mathrm{tgt}}_{n,r}=v^{\mathrm{cur}}_{n,r}$, then robot $r$ waits at its current node, and the waiting duration $\tau_{n,r}$ is positive. For robots with $\delta_{n,r}>0$, they continue their current travel or waiting segment and take a \textit{no-op}  action.            
	
	\textbf{Transition kernel \(P\)}: The transition is event-driven and deterministic. Given the current state $s_n$ and a joint action $\boldsymbol{a}_n$, we first apply the actions assigned to all available robots. Specifically, for each robot $r\in \robots$ at $v^{\mathrm{cur}}_{n,r}$ with $\delta_{n,r}=0$ and action $a_{n,r} = (v^{\mathrm{tgt}}_{n,r},\tau_{n,r})$, we update its segment state as $p_{n,r}= (v^{\mathrm{cur}}_{n,r},\,v^{\mathrm{tgt}}_{n,r},\,\tau_{n,r})$. Robots with $\delta_{n,r}>0$ are unavailable for decision-making and keep their current travel or waiting segment, corresponding to a no-op action. The system then evolves until the next event. Let $\Delta t_{\mathrm{evt}} = \min_{r\in \robots} \delta_{n,r}$ be the smallest remaining duration among all robots. To ensure that the performance tracker starts exactly at the  tail parameter $T$, we also introduce an artificial event  when the capped elapsed-time counter $\eta_n$ reaches $T$. Thus, the transition duration is
	$$
	\Delta t= \begin{cases}\min \left(\Delta t _\mathrm{evt}, T-\eta_n\right), &\text{if } \eta_n<T, \\ \Delta t_ \mathrm{evt}, & \text{if }\eta_n = T.\end{cases}
	$$
	The remaining duration of each robot $r\in \robots$ is updated as 
    $\delta_{n+1,r} = \delta_{n,r}-\Delta t$. The components $v^{\textrm{from}}_{n,r}$ and $v^{\textrm{to}}_{n,r}$ are inherited from the previous state: $v^{\textrm{from}}_{n+1,r}=v^{\textrm{from}}_{n,r}$ and $v^{\textrm{to}}_{n+1,r}=v^{\textrm{to}}_{n,r}$. Node latencies are then updated according to the deterministic motion during
the transition interval. For each node $v$,
	$$
	L_{n+1,v}=
	\begin{cases}
		0, \quad \text{if } v \text{ is visited or occupied in } (t_n,t_{n}+\Delta t].&\\
		L_{n,v}+ \weight(v)\Delta t, \qquad \qquad \qquad \quad \quad \  \text{otherwise}.&
	\end{cases}
	$$
    The capped elapsed-time counter is updated by
    \begin{equation*}
        \eta_{n+1}  = \min \{\eta_n+\Delta t,T\}.
    \end{equation*}
	In addition, we update the performance tracker as follows. Before the evaluation phase begins (\(\eta_{n+1}<T\)), the tracker $z$ remains inactive, while once the horizon \(T\) is reached, it records the worst-case weighted latency observed so far:
	\begin{multline*}
        z_{n+1} \\ = \begin{cases}\max \{ z_n, \underset{{v \in \mathcal V_n^{\mathrm{unocc}}} }{\max}  [L_{n,v}+\weight(v) \Delta t ]\}, & \text{if }   \eta_{n+1}=T,\\
			0, &\text{if }    \eta_{n+1}<T, 
		\end{cases}
	\end{multline*}
    where $\mathcal V_n^{\mathrm{unocc}}$ denotes the set of nodes that are not occupied by any waiting robot during the open transition interval $(t_n,t_n+\Delta t)$.
	The next state is $s_{n+1}=(\boldsymbol{p}_{n+1},\boldsymbol{L}_{n+1},z_{n+1},\eta_{n+1})$ and $P(s_{n+1}\mid s_n,\boldsymbol{a}_n)=1$.
	
	\textbf{Cost function $C$}: At each event $t_n$, we define the one-step cost as $$C_n=C(s_n,\boldsymbol{a}_n)=z_n.$$ 
	
	\textbf{Objective function $\aveobj_{s_0}$}: For a stationary Markov policy $\mu$, we adopt the standard \emph{average–cost} criterion 
	\begin{equation} \label{eq:avestep}
		\aveobj_{s_0}(\mu)
		= \lim_{N\to\infty} \frac{1}{N} \sum_{n=0}^{N-1} \mathbb{E}_\mu[C_n \, | \, s_0 ],
	\end{equation}
	where $s_0 = (\inipos,\boldsymbol{0}_\numnode,0,0)$ is the initial MDP state.

	The above construction defines a standard event-driven MDP that satisfies the Markov property and admits a conventional average-cost objective. Since this formulation is designed to optimize the tail worst-case weighted latency, we refer to it as the \emph{Tail Worst-case Latency-Optimizing Markov Decision Process} (TWLO-MDP). We next show that, when the tail horizon $T$ is sufficiently large, solving the TWLO-MDP is equivalent to solving the original monitoring problem in terms of optimal tail performance.
    

	\begin{theorem}[(Optimality of the TWLO-MDP  reformulation)] \label{thm:MDPopt}
		Consider the TWLO-MDP  $\mathcal{M}_T$ associated with the 
		monitoring problem in~\cref{prob:monitoring-problem}, where waiting durations can take arbitrary positive real values. Let ${\Omega}_{\mathrm{NZ}}$ denote the class of admissible MDP policies whose induced event-time sequences diverge to infinity, i.e., policies that induce non-Zeno executions. Fix an initial configuration $\inipos\in\nodes^\numrobot$, and let \(s_0\) be the corresponding initial MDP state. For any tail parameter $T$ satisfying $T > \graphdiameter+\frac{\optobjall}{\weightmin}$, the TWLO-MDP admits an optimal non-Zeno 
		stationary Markov policy $\mu^*\in{\Omega}_{\mathrm{NZ}}$ , and its optimal value coincides with the 
		optimal tail  performance of \cref{prob:monitoring-problem}:
		\[
		\aveobj_{s_0}(\mu^*)=\inf_{\mu\in{\Omega}_{\mathrm{NZ}}}\aveobj_{s_0}(\mu)=\optobjall.
		\]
		Moreover, any such optimal policy induces a feasible 
		monitoring strategy starting from $\inipos$ that achieves the same optimal value $\optobjall$.
	\end{theorem}
	\begin{proof}
		The proof is postponed to Appendix~\ref{appendix:MDPopt}.    
	\end{proof}
    \cref{thm:MDPopt} guarantees the existence of an optimal stationary Markov policy for the TWLO-MDP within the non-Zeno policy class and shows that its value  coincides with the optimal tail performance of the original monitoring problem. Moreover, such a policy induces a feasible monitoring strategy starting from the prescribed initial configuration with the same optimal value. The non-Zeno restriction is natural from an algorithmic perspective, since a Zeno policy would require executing infinitely many decisions within a finite time interval and is therefore not implementable.  A corresponding exact equivalence also holds between the discretized-waiting-time MDP and the discretized monitoring problem. 
	\begin{corollary}[(Optimality of the discretized-waiting-time MDP)] \label{cor:optimality}
		Consider the discretized monitoring problem in~\cref{prob:monitoring-discrete} with waiting-time  resolution $\Delta$, fixed initial configuration $\inipos$, and tail parameter $T > \graphdiameter+\frac{\optobjall}{\weightmin}$. Let $\mathcal{M}_T^{\Delta}$ denote the event-driven MDP obtained from TWLO-MDP by restricting waiting actions to the \(\Delta\)-waiting action. Then the optimal value of $\mathcal{M}_T^{\Delta}$ coincides exactly with the optimal value of \cref{prob:monitoring-discrete}. Moreover, any  optimal stationary policy of $\mathcal M_\Delta$ induces  an optimal strategy for \cref{prob:monitoring-discrete}.
	\end{corollary}
	\begin{proof}
The discretized formulation excludes Zeno behavior by construction. Indeed, every waiting action has duration $\Delta>0$, and, because the graph is finite with strictly positive edge lengths, every movement action has duration at least the minimum edge length. Hence, only finitely many events can occur in any finite physical-time interval. The rest of the proof follows the same structure as that of \cref{thm:MDPopt}, with the induced strategy class restricted to $\rationalstrategyset^\Delta_{\inipos}$ and the waiting-time action set replaced by its discretized counterpart. We therefore omit the details.
	\end{proof}
In the sequel, we mainly focus on the discretized-waiting-time TWLO-MDP $\mathcal M_\Delta$. Compared with its continuous-waiting-time  counterpart, $\mathcal M_\Delta$ has a finite  action space and is therefore more tractable for reinforcement learning. Moreover, as shown in \cref{thm:suboptimality}, the approximation loss induced by waiting-time discretization is controllable. In addition, Zeno behavior is naturally excluded in \(\mathcal M_\Delta\), since all waiting and movement actions have positive durations bounded away from zero.

	We also note that reinforcement learning with non-standard or non-cumulative objectives has recently attracted increasing  attention~\citep{SG-YP-RN:2020,RW-PZ-LY:2020,WC-WY:2023,GV-WB-MD:2024}. Many existing studies in this direction consider objectives such as maximum reward along a trajectory, general nonlinear trajectory objectives, or other non-cumulative criteria. These formulations are typically developed for discounted, finite-horizon, or episodic settings. In contrast, the monitoring problem studied here is an infinite-horizon problem with a worst-case trajectory-level objective and no discount factor in its theoretical formulation. The TWLO-MDP reformulation shows that, under the proposed state augmentation and the non-Zeno execution condition, this non-additive functional optimization problem can be represented as an average-cost MDP. Beyond the persistent monitoring setting considered here, this perspective may also provide a useful approach for handling other infinite-horizon decision-making problems with non-cumulative objectives. 
	
	\subsection{Reinforcement-learning-based solution methods}\label{sub:algorithm}
	Having reformulated the  monitoring problem  as the TWLO-MDP, whose optimal policy can be  mapped to an optimal monitoring strategy, we next develop reinforcement-learning-based solution methods  for the resulting event-driven MDP. The theoretical objective of the TWLO-MDP is an infinite-horizon average-cost criterion. In practice, however, we optimize a discounted surrogate objective with a discount factor \(\gamma\) close to one, following the standard discounted-return implementation of deep reinforcement learning algorithms. Learning is performed using finite rollout segments, and the large discount factor is intended to emphasize the long-run behavior encoded by the TWLO-MDP. The proposed framework can accommodate multiple reinforcement learning algorithms, as detailed in~\cref{section:platform}. In this subsection, we instantiate the proposed learning approach with MAPPO~\citep{CY-AV-YW:2022}, which is used as the main policy-optimization algorithm in our experiments because it achieves the strongest and most stable performance among the tested RL methods. Based on this instantiation, we describe the main design choices and training techniques used to solve the TWLO-MDP in practice. Additional implementation details, including auxiliary training signals, action masks, and hyperparameter selection, are provided in Appendix~\ref{appendix:training}.

	\textbf{Time-normalized training objective.} 
	A practical issue arises when training the event-driven MDP using rollouts truncated by the number of decision steps. Since different event steps may correspond to very different amounts of physical time, a fixed-step rollout does not in general correspond to a fixed elapsed monitoring time.  In particular, a policy may repeatedly choose short-duration actions, such as minimal waiting, thereby slowing down  the progression of physical time within each rollout. Over a truncated fixed-step horizon, this can artificially suppress the growth of latency at unvisited nodes, even though such behavior is undesirable for the original monitoring problem. This reveals a mismatch between fixed-step truncation and the physical-time semantics of persistent monitoring.
    
	A natural remedy is to truncate each rollout according to elapsed physical time rather than the number of event steps. However, modifying only the stopping rule is not sufficient. If the original step-wise cost $z_n$ is used without accounting for the duration of each event interval, then a policy can still affect the training objective value by changing the number of events that occur within the same physical-time horizon. Therefore, to align training with the physical-time semantics of the monitoring task, we use both physical-time truncation and duration-weighted event costs. The following lemma states the corresponding asymptotic equivalence. It shows that, at the infinite-horizon level, replacing the step-wise average in~\eqref{eq:avestep} by a time-normalized average does not change the limiting objective value, while the latter provides a more faithful finite-rollout surrogate for training. 
	\begin{lemma}[(Equivalent time-normalized  objective)]\label{lemma:timenormalized}
		For any non-Zeno stationary policy $\mu\in\Omega_{\mathrm{NZ}}$ and initial state $s_0$,  let $\Delta t_n$ denote the physical duration of the $n$-th event interval, and let $z_n$ be the tail worst-case weighted latency tracker defined in \cref{subsection:MDP}. Then the step-average objective in~\eqref{eq:avestep} is equivalent to the following time-normalized objective:
        \begin{equation}\label{eq:normalizedobj}
			\aveobj_{s_0}(\mu)
			=\lim _{N \rightarrow \infty} \mathbb{E}_\mu\left[\left.\frac{\sum_{n=0}^{N-1} z_n \Delta t_n}{\sum_{n=0}^{N-1} \Delta t_n} \right\rvert\, s_0\right].
		\end{equation}
	\end{lemma}
	\begin{proof}
		The proof is postponed to Appendix~\ref{appendix:timenormalized}.
	\end{proof}
	Based on \cref{lemma:timenormalized}, we modify the cost used in practical RL training. Instead of assigning the event-wise cost $z_n$ to the $n$-th event step, we assign the  duration-weighted cost $z_n \Delta t_n$. We also terminate each rollout according to a fixed physical-time horizon rather than a fixed number of event steps. In this way, the contribution of each event is weighted by the amount of elapsed monitoring time, aligning the learning signal with the time-normalized objective in~\eqref{eq:normalizedobj}. 
	
	To illustrate the practical difference between~\eqref{eq:avestep} and~\eqref{eq:normalizedobj} for finite rollouts, consider a two-node graph connected by an edge of length $A$, with $T=0$ and minimum waiting duration $\Delta$. We compare two strategies:  repeatedly waiting for duration $\Delta$ at one node, and moving back and forth between the two nodes. The moving strategy achieves worst-case latency $2A$. Under a fixed-step rollout of length $N$, the repeated-waiting strategy yields the truncated step-average objective $\frac{1}{N}\sum_{n=1}^{N} n\Delta
	=
	\frac{N+1}{2}\Delta$.  Therefore, the waiting strategy ceases to appear preferable only when $\Delta \frac{N+1}{2} > 2A$, i.e., when $N>\frac{4A}{\Delta}-1$. For example, when $A=20$ and $\Delta=0.1$, roughly $400$ steps are needed before fixed-step truncation can eliminate the spurious preference for repeated waiting. By contrast, under physical-time truncation with horizon $H$, assuming that $H$ is an integer multiple of $\Delta$, the same repeated-waiting strategy yields the time-normalized objective $\frac{\Delta(\Delta+2\Delta+...+H)}{H} = \frac{H+\Delta}{2}$. Hence, under physical-time truncation, the repeated-waiting strategy ceases to appear preferable once $H>4A-\Delta$. This threshold is expressed directly in physical time and therefore has a clear interpretation in terms of the underlying monitoring task.
	In principle, fixed-step truncation can also suppress the pathological waiting strategy, but only if the rollout length is chosen sufficiently large to cover the relevant physical-time scale. In our setting, however, the required number of decision steps can become very large when short-duration actions are available, leading to reduced  training stability and efficiency. Physical-time truncation instead specifies the rollout horizon directly on the appropriate elapsed-time scale. In practice, this modification results in more stable training and helps avoid degenerate local optima, such as repeatedly choosing short waiting actions.

    \textbf{Transparent generalized advantage estimation.}
    Graph-based monitoring environments may involve heterogeneous edge lengths and waiting actions, so the time interval between two consecutive decision points of an individual robot can vary. This makes the monitoring problem closely related to an event-driven Semi-Markov Decision Process (SMDP) formulation~\citep{RS-DP-SS:1999}. Existing asynchronous multi-agent reinforcement learning methods~\citep{YX-WT-JH:2025} often handle variable execution times through hierarchical macro-actions. While such mechanisms are useful in general asynchronous decision-making settings, they introduce additional computational overhead and are unnecessary for the atomic movement and waiting actions considered in our graph-based monitoring problem.
    
    Instead, we retain a globally synchronized event-step rollout while distinguishing, for each robot, between active decision steps and inactive intermediate steps. At a global event step, a robot is active only if it has arrived at a node or completed a waiting action and is therefore available for decision-making. Otherwise, the step is an inactive intermediate step for that robot. The implementation details for constructing such synchronized rollouts are provided in Appendix~\ref{subsubsection:action_mask}. Given this rollout structure, naively applying standard generalized advantage estimation (GAE)~\citep{JS-PM-PA:2016} can distort advantage estimation, because standard GAE discounts and accumulates temporal-difference (TD) errors across every recorded global transition, including inactive intermediate steps that do not correspond to genuine decision points of a particular robot. Treating these inactive steps as ordinary policy-update steps can introduce misaligned or noisy advantage estimates.
    
    To address this issue, we propose \emph{Transparent Generalized Advantage Estimation (Trans-GAE)}, which folds inactive intermediate steps and computes advantages only over the active decision sequence of each robot. Formally, let $\chi_{n,r} \in \{0, 1\}$ be an indicator that equals $1$ if robot $r$ is active at global event step $n$, meaning that it has arrived at a node or completed a waiting action and is  available for decision-making. Define the set of active decision steps for robot $r$ as $\mathcal{K}_r = \{n\in\Nature \mid \chi_{n,r} = 1\}$. For any $n \in \mathcal{K}_r$, let $n^+ = \min\{n' \in \mathcal{K}_r \mid n' > n\}$ denote the next active decision step of robot $r$, and define $\ell_{n,r} = n^+ - n - 1$ as the number of inactive global steps between these two consecutive active decision points. Instead of using the standard one-step TD error, Trans-GAE constructs a folded multi-step TD error $\tilde{\delta}_{n,r}$ over the interval from $n$ to $n^+$:    
    $$
    \tilde{\delta}_{n,r}
    =
    \sum_{j=0}^{\ell_{n,r}} \gamma^j C_{n+j}
    +
    \gamma^{\ell_{n,r}+1} V_{\omega}(s_{n^+})
    -
    V_{\omega}(s_n),
    $$    
    where $C_n$ and $s_n$ are the shared cost and global state at event step $n$, respectively, $\gamma$ is the discount factor, and $V_\omega$ is the centralized critic parameterized by $\omega$. The summation accumulates the costs collected before the next active decision point of robot \(r\), while the value at \(s_{n^+}\) provides the bootstrap term. The generalized advantage $\tilde{A}_{n,r}$ is then computed recursively over active decision steps of each robot:    
    $$
    \tilde{A}_{n,r}
    =
    \begin{cases}
    \tilde{\delta}_{n,r}
    +
    (\gamma\lambda)^{\ell_{n,r}+1}
    \tilde{A}_{n^+,r},
    & \text{if } n \in \mathcal{K}_r, \\
    0,
    & \text{if } n \notin \mathcal{K}_r,
    \end{cases}
    $$    
    where $\lambda$ is the trace-decay parameter. The factor $(\gamma\lambda)^{\ell_{n,r}+1}$ accounts for the number of global event steps between two consecutive active decision points of robot $r$. In this way, inactive intermediate steps are transparent to the agent-level advantage recursion, while their accumulated costs are still assigned to the preceding active decision. For critic optimization, the corresponding GAE target is constructed as    
    $$
    G^{\text{GAE}}_{n,r}
    =
    \tilde{A}_{n,r}
    +
    V_{\omega}(s_n).
    $$    
    Both the policy-gradient loss and the critic loss are evaluated only at active decision steps. Inactive intermediate steps are excluded from the optimization objective, but the costs incurred during these intervals are incorporated through the folded TD error. As a result, costs collected during a movement or waiting segment are attributed to the decision that initiated the segment, rather than being treated as independent updates at inactive steps. This improves the alignment of advantage estimation with the asynchronous event-driven structure of the monitoring problem and stabilizes policy optimization.
    
    \textbf{Parameter sharing and role encoding.} To improve scalability and sample efficiency, we adopt parameter sharing across robots, a common practice in MARL~\citep{JF-IA-SW:2016,NG-SS-AA:2020,JY-JV-MS:2022}. This design is also natural for our monitoring problem, since the robots are homogeneous and are expected to follow a common decision rule.  However, parameter sharing creates a symmetry issue.  In our formulation, each robot can access the global state to obtain sufficient information for coordinated decision-making. If all robots receive the same global input and use the same policy network, then the policy cannot distinguish which robot is being controlled. A straightforward remedy is to include each robot's own position in its observation. Nevertheless, this information is not always sufficient in our setting. If multiple robots become available at the same node and observe the same global state, then a shared deterministic policy may assign them identical actions, causing redundant motion and inefficient use of monitoring resources. A common solution in parameter-sharing MARL is to augment each agent's observation with a fixed agent identity or index~\citep{JF-GF-SW:2018,TR-MS-SW:2020,CY-AV-YW:2022}. While this approach can break symmetry and improve coordination, it also introduces dependence on arbitrary agent labels, which is undesirable for a homogeneous monitoring team. To mitigate this issue, we use a lightweight role encoding based on local decision order rather than fixed agent identity. Specifically, for each robot $r$ that is available for decision-making at event step $n$, we define a local decision-order label $\ell^{\text{dec}}_{n,r}$, which records the order of robot $r$ among the robots simultaneously ready to act at the same node. The observation of robot $r$ is then augmented as $o_{n,r}=\{s_n, \ell^{\text{dec}}_{n,r}\}$. This role encoding allows the shared policy to distinguish robots that are otherwise indistinguishable under the same global state and local position, while avoiding the use of fixed agent identities. 
    
	\textbf{Customized observation preprocessing.} Standard observation preprocessing methods in deep reinforcement learning often normalize all observation dimensions using running means and standard deviations. However, the observation space of TWLO-MDP contains discrete vertex-index dimensions that explicitly represent nodes in the monitoring graph. Applying  continuous normalization to these categorical variables can distort their semantics and degrade learning performance. To address this issue, we use a heterogeneous observation preprocessing module that applies different transformations to different observation dimensions before feeding them into the neural networks.
	
	For vertex-index dimensions, including each robot's departure node, destination node, and current node when available, we replace standard normalization with a graph-structural embedding. Specifically, we compute a Graph Positional Encoding (GPE)~\citep{MB-PN:2003} for each vertex by eigendecomposing the symmetric normalized Laplacian $\tilde{L} = I - D^{-1/2}WD^{-1/2}$, where $W \in \mathbb{R}^{|\mathcal{V}| \times |\mathcal{V}|}$ denotes the binary adjacency matrix of the monitoring graph,
	and $D$ is the corresponding diagonal degree matrix with $D_{ii} = \sum_{j} W_{ij}$. We retain the $d_{\text{gpe}}$ eigenvectors associated with the smallest nonzero eigenvalues. This produces a fixed $(|\nodes| \times d_{\text{gpe}})$ embedding table, which encodes the topological position of each vertex in the graph. The embedding table is kept frozen throughout training. At runtime, each integer vertex index is mapped to its corresponding $d_{\text{gpe}}$-dimensional positional representation, which is then passed through a trainable linear projection to obtain a $d_{\text{proj}}$-dimensional feature vector. This representation preserves graph-structural information  and avoids dependence on arbitrary integer labels. In all experiments, we set $d_{\text{gpe}} = 8$ and $d_{\text{proj}} = 16$ uniformly across different monitoring graphs. Graphs with $|\nodes| < d_{\text{gpe}}+1$ cannot fully populate all GPE dimensions; in this case, the embedding table retains the available eigenvector coordinates and pads any remaining columns with zero vectors so that every vertex is still represented by a fixed $d_{\text{gpe}}$-dimensional vector before linear projection.

    For scalar continuous features, we apply transformations matched to their dynamics. The worst weighted idleness since $T$ is a running maximum: once a high-priority vertex is left unattended, it can jump sharply and then remain elevated for the rest of the episode. During training, such episodic spikes are infrequent but can dominate running statistics if left untransformed. We therefore apply a $\log(1+\cdot)$ map to attenuate large values while retaining resolution near zero, and then normalize the result using running statistics estimated online from rollouts. Instantaneous  latency values  and per-robot remaining action times are  normalized directly using online running statistics. In contrast, the elapsed-time counter $\eta_n$  has a known upper bound: it increases from zero to the  tail parameter $T$ and remains constant thereafter. We therefore normalize it by dividing by $T$, mapping it to the interval $[0, 1]$ without requiring online statistics. The one-hot encodings of the local decision-order labels are already in $\{0, 1\}$ and do not represent numerical magnitudes, so they are passed through without modification.

    \textbf{Imitation learning warm start.} \label{subsection:imi} Training solely with reinforcement learning can be inefficient for the monitoring problem, especially in large-scale environments with multiple robots. Effective monitoring strategies typically require coordinated motion patterns and strong structural regularity, which can be difficult to discover through exploration alone. This challenge is further amplified by the shared global training signal, which is delayed, weakly informative at the individual-robot level, and coupled across all robots. Motivated by the use of imitation learning as a warm-start mechanism for challenging RL problems~\citep{AR-VK-SL:2018,TS-RC:2018,TH-MV-OP:2018,OR-IB-WC:2019}, we combine heuristic demonstrations with subsequent RL fine-tuning. 
    
    As outlined in \cref{alg:alg1}, the imitation-learning stage provides an initial policy and value estimate while still allowing RL to improve beyond the demonstrated heuristic behavior. We first execute  heuristic monitoring strategies to collect demonstration trajectories and compute normalized Monte Carlo returns, as shown in Lines~\ref{line:mc_return}--\ref{line:norm_return}. The resulting transition tuples are stored in an offline dataset $\mathcal{D}$. We then pretrain both the actor and the critic using this dataset before starting online RL optimization. For actor pretraining, we apply behavior cloning to the demonstrated actions, as shown in Line~\ref{line:actor_loss}. The cross-entropy loss is evaluated under the corresponding action mask $m_{n,r}$, so that invalid actions are excluded from the supervised learning objective; see Appendix~\ref{subsubsection:action_mask} for details on action masking. For critic pretraining, we fit the centralized value function to the normalized Monte Carlo returns $\widehat{G}_j$ using mean squared error, as shown in Line~\ref{line:critic_loss}. The mean $\psi$ and standard deviation $\sigma$ computed from the demonstration returns are used to initialize the running return statistics for the subsequent RL stage, as shown in Line~\ref{line:rl_init_stat}.
    This imitation-learning warm start provides a more informative initial policy, improves the initial value estimates, and reduces the exploration burden during early RL training. The same pretraining mechanism can be applied to other MARL algorithms in our framework. In our experiments, MAPPO combined with this imitation-learning warm start achieves the best empirical performance among the tested RL methods.

\begin{algorithm*}[htb]
	\caption{Sampling, Imitation Learning, and RL training with MAPPO. (Part 1)}\label{alg:alg1}
	\begin{algorithmic}[1]

    \Statex \hspace{-\algorithmicindent}\makebox[\linewidth][c]{\textbf{\textsc{SAMPLING}}}
        \State \textbf{Input:} Discount factor $\gamma$, total number of sampled episodes $M$
        \State \textbf{Output:} Dataset $\mathcal{D}$ and Monte-Carlo returns  $(\psi , \sigma )$
        \Comment{$\psi$ denotes the mean and  $\sigma$ denotes the standard deviation.}
        \State Initialize rollout buffers $\mathcal{B} \gets \emptyset$ and  $\mathcal{D} \gets \emptyset$
        \For{$i = 1$ to $M$}
            \State Reset the environment and a heuristic policy, and set  $\mathcal{B}_{\mathrm{ep}}^{i} \gets \emptyset$
            \For{each step $n$ in episode $i$}
                \For{each robot $r$}
                    \State $b_{n,r}^{i} \gets (o_{n,r}^{i}, s_{n,r}^{i}, a_{n,r}^{i}, m_{n,r}^{i}, \chi_{n,r}^{i}, C_n^{i})$
                    \Comment{Record raw tuple.}
                \EndFor
                \State $\mathcal{B}_{\mathrm{ep}}^{i} \gets \mathcal{B}_{\mathrm{ep}}^{i} \cup \{ b_{n,1}^{i}, \dots, b_{n,\numrobot}^{i} \}$
            \EndFor
            \State $\mathcal{B} \gets \mathcal{B} \cup \mathcal{B}_{\mathrm{ep}}^{i}$
        \EndFor
        \For{each episode-step pair $(i,n)$ in $\mathcal{B}$}
            \State $G_n^{i} \gets \sum_{j=n}^{|\mathcal{B}_{\mathrm{ep}}^{i}|} \gamma^{j-n} C_j^{i}$ \label{line:mc_return}
            \Comment{Compute the Monte-Carlo returns.}
        \EndFor
        \State $\psi  \gets \big( \sum_{b_{n,r}^{i} \in \mathcal{B}} \chi_{n,r}^{i} G_n^{i} \big) \big/ \big( \sum_{b_{n,r}^{i} \in \mathcal{B}} \chi_{n,r}^{i} + \epsilon \big)$ 
        \State $\sigma  \gets \sqrt{\big( \sum_{b_{n,r}^{i} \in \mathcal{B}} \chi_{n,r}^{i} \big(G_n^{i} - \psi \big)^2 \big) \big/ \big( \sum_{b_{n,r}^{i} \in \mathcal{B}} \chi_{n,r}^{i} + \epsilon \big)}$ \label{line:norm_return}
        \For{each $b_{n,r}^{i} \in \mathcal{B}$}
            \State $\widehat{G}_n^{i} \gets (G_n^{i}-\psi ) / (\sigma +\epsilon)$
            \Comment{Normalize the Monte-Carlo returns.}
            \State $\mathcal{D} \gets \mathcal{D} \cup \{ (o_{n,r}^{i}, s_{n,r}^{i}, a_{n,r}^{i}, m_{n,r}^{i}, \chi_{n,r}^{i}, C_n^{i}, \widehat{G}_n^{i}) \}$
        \EndFor

		\Statex 
        \Statex \hspace{-\algorithmicindent}\makebox[\linewidth][c]{\textbf{\textsc{IMITATION LEARNING}}}
        \State \textbf{Input:} Dataset $\mathcal{D}$, batch size $K$
        \State \textbf{Output:} Pre-trained network parameters $\theta$ and $\omega$ for the policy $\mu_\theta$ and the centralized critic $V_\omega$
        \State Initialize $\theta$ and $\omega$  
        \For{each training iteration}
            \State Sample a batch $\mathcal{B}_\mathrm{imi} = \{ b_j=(o_{j}, s_{j}, a_{j}, m_{j}, \chi_{j}, C_j, \widehat{G}_j) \}_{j=1}^K \subseteq \mathcal{D}$
            \For{each $b_j$ in $\mathcal{B}_\mathrm{imi}$}\Comment{Renormalize using action mask $m_j$}
                \State $\tilde{\mu}_\theta(a \mid o_j, m_j) \gets \big(m_j(a)\cdot\mu_\theta(a \mid o_j)\big) \big/ \big( \sum_{a' \in \mathcal{A}} m_j(a')\cdot\mu_\theta(a' \mid o_j) + \epsilon \big)$
            \EndFor
            \State $\mathcal{L}_{\mathrm{actor}}^{\mathrm{imi}} \gets - \big( \sum_{j=1}^K \chi_j \log \tilde{\mu}_\theta( a_j \mid o_j, m_j ) \big) \big/ \big( \sum_{j=1}^K \chi_j + \epsilon \big)$ \label{line:actor_loss}
            \State $\mathcal{L}_{\mathrm{critic}}^{\mathrm{imi}} \gets \big( \sum_{j=1}^K \chi_j \left(V_\omega(s_j) - \widehat{G}_j\right)^2 \big) \big/ \big( \sum_{j=1}^K \chi_j + \epsilon \big)$ \label{line:critic_loss}
            \State Update $\theta, \omega$ with the losses $\mathcal{L}_{\mathrm{actor}}^{\mathrm{imi}}, \mathcal{L}_{\mathrm{critic}}^{\mathrm{imi}}$
        \EndFor
        \algstore{bk_point}
        
	\end{algorithmic}
    \end{algorithm*}

    \begin{algorithm*}[htb]
	\caption{Sampling, Imitation Learning, and RL training with MAPPO. (Part 2)}\label{alg:alg2}
	\begin{algorithmic}[1]
    \algrestore{bk_point}
        \Statex \hspace{-\algorithmicindent}\makebox[\linewidth][c]{\textbf{\textsc{RL TRAINING (MAPPO)}}}
    \State \textbf{Input:} Pre-trained parameters $\theta, \omega$, Monte Carlo returns  $(\psi , \sigma )$, PPO clip range $\epsilon_c$, RMS update rate $\tau_{rms}$, batch size $K'$
    \State \textbf{Output:} Fine-tuned network parameters $\theta_{\mathrm{RL}}, \omega_{\mathrm{RL}}$
    \State Initialize policy $\mu_{\theta_{\mathrm{RL}}}$ with $\theta$, centralized critic $V_{\omega_{\mathrm{RL}}}$ with $\omega$, $\psi_{\mathrm{ret}} \gets \psi $, $\sigma_{\mathrm{ret}} \gets \sigma$ \label{line:rl_init_stat}
    \For{each RL iteration}
        \State Collect data $\mathcal{B}_\mathrm{RL} = \{ b_j=(o_j, s_j, a_j, m_j, \chi_j, C_j)\}_{j=1}^{K'}$ using policy $\mu_{\theta_{\mathrm{RL}}}$
        \State $\theta_{\mathrm{old}} \gets \theta_{\mathrm{RL}}$ \Comment{Save policy parameters for the importance ratio}
        \For{each $b_j$ in $\mathcal{B}_\mathrm{RL}$}
            \State $V_j\gets V_{\omega_{\mathrm{RL}}}(s_j) \cdot \sigma_{\mathrm{ret}} + \psi_{\mathrm{ret}}$ 
            \Comment{Critic outputs normalized values; denormalize for Trans-GAE}
            \State $\tilde{A}_j \gets \textsc{TransGAE}\left(\{C_j\}, \{V_j\}\right)$
            \State $G_j \gets \tilde{A}_j + V_j$
        \EndFor
        \State $\psi_{\mathrm{batch}} \gets \big( \sum_{j=1}^{K'} \chi_j G_j \big) \big/ \big( \sum_{j=1}^{K'} \chi_j + \epsilon \big)$ \Comment{Calculate batch statistics}
        \State $\sigma_{\mathrm{batch}} \gets \sqrt{\big( \sum_{j=1}^{K'} \chi_j \big(G_j- \psi_{\mathrm{batch}}\big)^2 \big) \big/ \big( \sum_{j=1}^{K'}\chi_j + \epsilon \big)}$
        \State $\psi_{\mathrm{ret}} \gets (1-\tau_{rms}) \cdot \psi_{\mathrm{ret}} + \tau_{rms} \cdot \psi_{\mathrm{batch}}$
        \State $\sigma_{\mathrm{ret}} \gets \sqrt{(1-\tau_{rms}) \cdot \sigma_{\mathrm{ret}}^2 + \tau_{rms} \cdot \sigma_{\mathrm{batch}}^2}$ \Comment{Update RMS statistics}
        \For{each $b_j$ in $\mathcal{B}_{RL}$}
            \State $\widehat{G}_j \gets (G_j - \psi_{\mathrm{ret}}) \big/ (\sigma_{\mathrm{ret}} + \epsilon)$
        \EndFor
        \For{each PPO epoch}
            \For{each $b_j$ in $\mathcal{B}_{RL}$}
            \State $\rho_j \gets \mu_{\theta_{\mathrm{RL}}}(a_j \mid o_j, m_j) \big/ \mu_{\theta_{\mathrm{old}}}(a_j \mid o_j, m_j)$
            \EndFor            
            \State $\mathcal{L}_{\mathrm{actor}} \gets - \big( \sum_{j=1}^{K'} \chi_j \min\left( \rho_j\tilde{A}_j,\ \mathrm{clip}(\rho_j, 1-\epsilon_c, 1+\epsilon_c) \tilde{A}_j \right) \big) \big/ \big( \sum_{j=1}^{K'} \chi_j + \epsilon \big)$
            \State $\mathcal{L}_{\mathrm{critic}} \gets \big( \sum_{j=1}^{K'} \chi_j\big(V_{\omega_{\mathrm{RL}}}(s_j) - \widehat{G}_j\big)^2 \big) \big/ \big( \sum_{j=1}^{K'} \chi_j + \epsilon \big)$
            \State Update $\theta_{\mathrm{RL}}, \omega_{\mathrm{RL}}$ with the losses $\mathcal{L}_{\mathrm{actor}}, \mathcal{L}_{\mathrm{critic}}$
        \EndFor
    \EndFor
    
    \end{algorithmic}
    \end{algorithm*}
	

	\section{Platform}\label{section:platform}
	Having established the theoretical foundations and developed RL-based solution methods for the multi-robot persistent monitoring problem,  we now present a modular simulation and evaluation platform for benchmarking  heuristic and RL-based monitoring algorithms within a unified framework. The platform is designed to support controlled comparisons across different methods under shared environments, interfaces, and evaluation criteria, while also allowing newly developed algorithms to be incorporated with minimal engineering effort.  Such a platform is useful  because, although a wide range of methods has been proposed for multi-robot monitoring, systematic comparison across algorithmic families remains limited. As discussed in the related  work section, existing platforms and simulators, such as \texttt{patrolling\_sim}~\citep{DP-LI-AF:2018}, provide useful support for evaluating representative heuristic patrolling algorithms. However, there remains a lack of platforms specifically designed to support both heuristic and reinforcement-learning-based monitoring methods under a unified experimental workflow. This makes it difficult to directly compare classical monitoring heuristics with modern learning-based approaches under consistent problem instances, interfaces, and performance metrics. 
	
    To address this gap, we develop \emph{Multi-Robot Monitoring Benchmark (M2Bench)}, a modular and extensible simulation platform for persistent monitoring. 
    As illustrated in~\cref{fig:platform}, M2Bench organizes the experimental workflow into several modular layers. These include a YAML-driven configuration layer, an environment layer for shared monitoring dynamics and task-specific MDP wrappers, a data layer for rollout collection and replay storage, an algorithm layer that supports heuristic and learning-based methods, and a utility layer for evaluation, visualization, logging, checkpointing, and hyperparameter sweeps. The remainder of this section describes the main modules and features of the platform.
	

	\begin{figure*}[ht]
		\centering
        \includegraphics[width=0.95\linewidth]{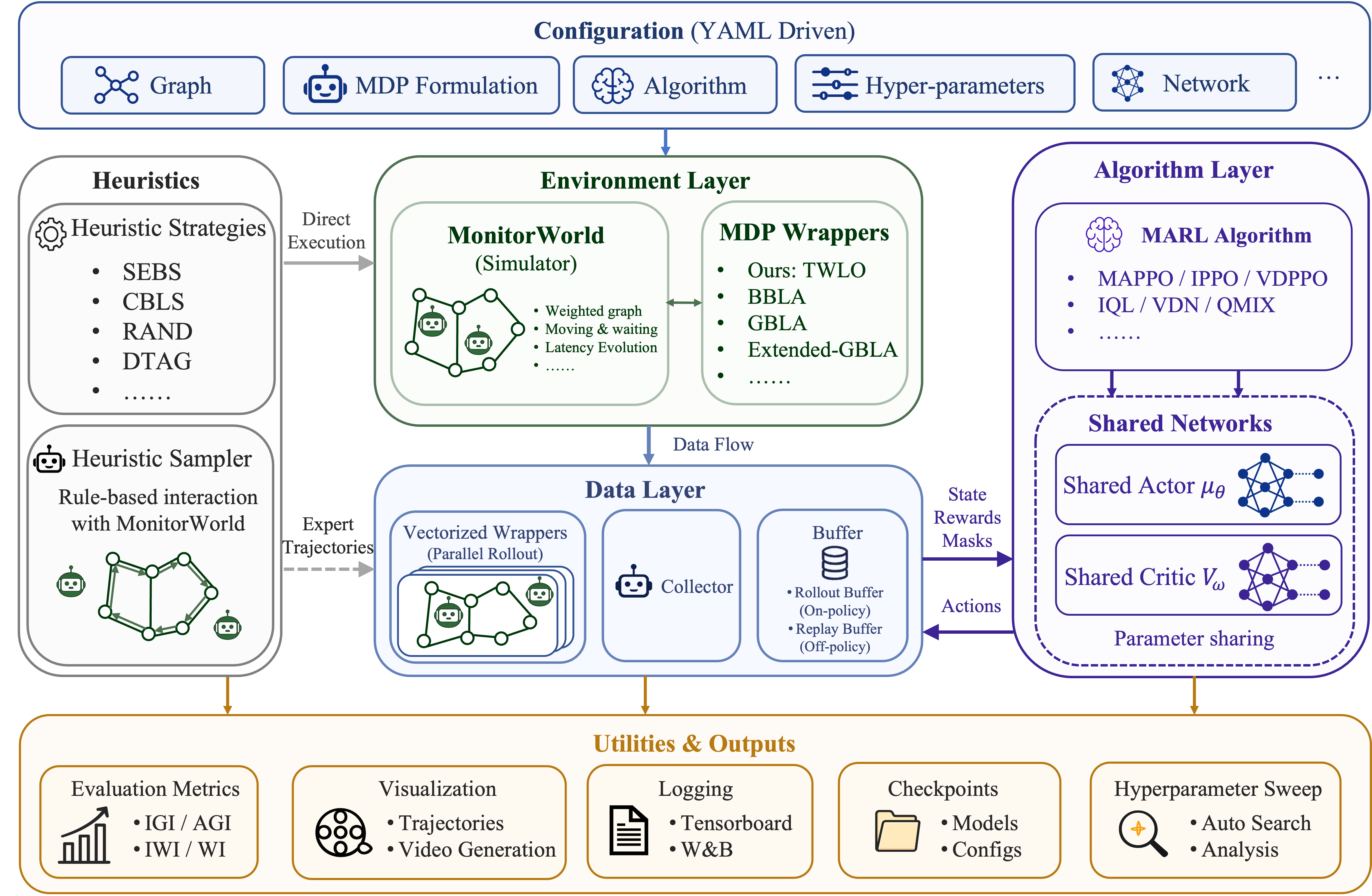}
		\caption{Architecture of M2Bench, a modular benchmark platform for multi-robot persistent monitoring. The platform integrates heuristic and reinforcement-learning-based methods within a unified workflow, including configuration-driven experiment management, monitoring-oriented simulation, task-specific MDP wrappers, rollout collection, shared MARL training, evaluation, visualization, and logging utilities.}
		\label{fig:platform}
	\end{figure*}

\textbf{Modular design and configuration workflow.}
Building on the layered architecture, M2Bench represents each experimental component as an interchangeable module with a clearly defined interface. Users can modify or replace components such as the graph instance, monitoring formulation, policy class, algorithm, and neural network architecture without rewriting the remaining experimental pipeline. The platform builds on widely used libraries, including PyTorch~\citep{Pytorch:2019}, Gymnasium~\citep{Gymnasium:2024}, and PettingZoo~\citep{Pettingzoo:2021}, while avoiding dependence on a monolithic training framework.

From the users' perspective, these modules are assembled through a unified configuration-driven workflow. A single YAML file specifies the graph, monitoring or MDP formulation, robot settings, algorithm type, network options, and training hyperparameters. Given this configuration, M2Bench automatically instantiates the corresponding environment, policy, collector, buffer, trainer, and evaluation routines, and executes the associated experimental pipeline, including training, checkpointing, testing, quantitative evaluation, logging with tools such as W\&B and TensorBoard, visualization, animation rendering, and hyperparameter sweeps.

Extensibility is supported through a centralized registry-and-configuration mechanism. Core components such as MDP wrappers, algorithms, network modules, and trainers are instantiated from the registry rather than hard-coded into the training loop. Thus, adding a new module typically requires only implementing a standardized interface and registering it with the framework, without modifying the overall execution pipeline. This declarative design facilitates switching between heuristic and RL-based methods, comparing different monitoring formulations under shared settings, and reproducing experiments without repeated modifications to low-level code.

	
	\textbf{Monitoring-oriented simulation.} The core environment component of M2Bench is \emph{MonitorWorld}, which simulates persistent monitoring dynamics on finite weighted graphs. Following \cref{subsection:environment_metrics}, it models robot motion, waiting, traveling, visiting events, physical-time evolution, and latency updates. MonitorWorld is independent of specific MDP formulations. Observations, action encodings, costs, and termination conditions are instead defined by external wrappers, allowing multiple learning formulations and non-RL heuristic policies to share the same simulator. It supports both fixed-step synchronous simulation and event-driven asynchronous simulation, enabling discrete-time and event-driven MDPs to be built on a common environment. This shared design also facilitates imitation pretraining in \cref{subsection:imi}, since heuristic trajectories can be directly reused by learning-based methods.    
	
	\textbf{Dataflow across the pipelines.} To support the integrated algorithms, M2Bench organizes the experimental workflow into four main dataflow pipelines: heuristic execution, on-policy MARL, off-policy MARL, and imitation-based pretraining. The first three pipelines largely follow standard control and reinforcement learning paradigms. In the \emph{heuristic} pipeline, the standalone evaluator directly interacts with \emph{MonitorWorld} by reading the global state and executing joint actions. The \emph{RL-based} pipelines instead use tensorized data from vectorized parallel environments, including per-robot observations and auxiliary tensors such as \texttt{active\_mask}, which indicates whether each robot is at a valid decision step. The \emph{imitation-learning} pipeline is detailed in \cref{subsection:imi}, and additional implementation details can be found in our GitHub repository: \url{https://github.com/SpikeW726/M2Bench}.


    \textbf{Policy mapping.}
    Following the modular design of Tianshou~\citep{Tianshou:2021}, M2Bench separates policy modules from learning algorithms. A policy defines how an agent maps observations to actions, whereas a learning algorithm specifies how the policy parameters are updated from collected experience. This separation makes the single-robot policy implementations reusable in multi-robot settings, and avoids  duplicating action-selection logic across different algorithms. In multi-robot environments, M2Bench uses a lightweight policy-mapping mechanism to assign policy instances to robots. With parameter sharing, all robots are mapped to a common policy instance with shared parameters. With independent policies, each robot is mapped to a separate policy instance with the same architecture but independent parameters. Thus, parameter sharing and policy independence are exposed as explicit configuration choices  rather than being hard-coded into the algorithm implementation. The design improves modularity, supports controlled comparisons, and allows learning algorithms to be developed independently of robot-level decision interfaces.



\textbf{Integrated algorithms and MDPs.}
To support systematic benchmarking for multi-robot monitoring, M2Bench integrates $10$ MDP formulations, $14$ classical heuristic policies, and $10$ modern reinforcement learning algorithms. Representative MDP-based monitoring formulations are summarized in \cref{table:RL}. These formulations were originally proposed under different monitoring assumptions or task objectives. 
To enable consistent comparison under the weighted monitoring setting considered in this benchmark, we adapt the reward implementations of these baselines in two ways. First, where applicable, unweighted latency terms are replaced by their weighted counterparts. Second, for NEP and OUCS,  task-specific reward components that are not directly related to the monitoring objective are removed. For comparative evaluation, M2Bench supports a range of MARL algorithms, including policy-gradient methods such as MAPPO~\citep{CY-AV-YW:2022}, IPPO~\citep{CD-TG-SW:2020}, and VDPPO~\citep{YM-JL:2022}, and value-based methods such as IQL~\citep{MT:1993}, VDN~\citep{PS-GL-AG:2017}, and QMIX~\citep{TR-MS-SW:2020}. 
In addition, M2Bench incorporates a broad collection of classical heuristic policies, including CC, CR, and RANDOM~\citep{AM-GR-JZ-AD:2003}, MSP~\citep{DP-RR:2010}, HCR and HPCC~\citep{portugal2013multi},  GBS~\citep{DP-RR:2013}, SEBS~\citep{DP-RR:2013:iros}, ER~\citep{CY-TZ:2016}, CBLS~\citep{DP-RR:2016}, DTAG and DTAP~\citep{AF-LI-DN:2017}, BAPS~\citep{HC-TC-SW:2017}, and AHPA~\citep{AG-XL-QZ:2024}.


    \begin{table*}
		\small\centering
		\caption{Representative MDP-based monitoring formulations integrated in M2Bench, including their reward designs, state representations, and learning algorithms. For methods whose original papers do not provide an explicit algorithm name, we refer to them using names derived from the corresponding paper titles for convenience.}\label{table:RL}
		\begin{tabularx}{\textwidth}{p{3.5cm}  p{4.5cm} Y p{2cm}}
			\toprule
			Source  & Reward & State/observation& Algorithm\\
			\midrule
			BBLA \citep{HS-GR-VC-BR:2004}& Instantaneous latency of the visited node  & Current node, incoming edge, and neighboring nodes with highest and lowest latency &  Q-learning\\
			\hline
			GBLA \citep{HS-GR-VC-BR:2004}& Same as BBLA  &  BBLA state +   adjacent nodes intended to be visited by other robots &  Q-learning\\
			\hline
			Extended-GBLA \citep{FL-AK:2014}& $0$ if the reached node is intended to be visited by others; otherwise its instantaneous latency    & Current node, incoming edge, adjacent nodes ordered by latency, and neighbors' intended target nodes & Q-learning\\
			\hline
			NEP~\citep{ZH-DZ:2010}&Cover-rate-based reward  & Relative positions of robots on nodes and edges  & Q-learning\\
			\hline
			S4R1 \citep{MJ-LV-AS:2022}& Local latency normalized by global average latency  &  Multiple state representations with different communication levels, combining source-target nodes, neighboring latency, and other robots' local information & DQN\\
			\hline
			BEAU \citep{LG-HP-XD-JH:2023}& Instantaneous global latency  &Graph-based observations with node features such as coordinates and latency, together with robot position  & GAT + HAPPO\\
			\hline
			MAGEC \citep{AG-YS-QZ:2024}& Combination of local latency normalized by global average latency and global terminal reward for minimizing episode-average latency & Graph-based observations within a radius, with node features such as type, normalized latency, and degree, and edge features such as normalized distance and edge identifier & GraphSAGE + MAPPO\\
			\hline
			SUNS \citep{JW-RM-EH:2025}& Instantaneous node latency  & Graph-based observation with node features such as latency and distance to the nearest robot, together with edge weights   & GCN + A2C\\
			\hline
			OUCS \citep{JP-EG-MB:2025}& Cooperative reward combining team reward, individual reward, and an exploration bonus  & Robot locations, visible-node visit counts,  and weights of neighboring nodes  & MAPPO, VDPPO, IPPO, MATRPO, VDA2C\\ 
			\hline
			Ours: TWLO & Tail  worst-case weighted latency &  Robot-team state, instantaneous  node latencies, historical worst-case latency, and elapsed-time counter & MAPPO, IPPO, VDPPO, IQL, VDN, QMIX\\ 
			\bottomrule
		\end{tabularx}
	\end{table*}
	
	\textbf{Additional training techniques.} Beyond its core algorithm implementations, M2Bench incorporates several training techniques that improve empirical performance. These include standard stabilization components in deep reinforcement learning, such as GAE and observation/value normalization~\citep{HH-AG-MH-DS:2016}. For event-driven environments, M2Bench supports Trans-GAE and loss masking in the on-policy pipeline, as well as an SMDP-style decision-epoch replay mechanism in the off-policy pipeline, similar to SMDP Q-learning~\citep{RS-DP-SS:1999}. For value-based methods with recurrent networks, the off-policy pipeline samples consecutive trajectory segments from the sequence replay buffer, following training patterns commonly used in DRQN~\citep{MH-PS:2015} and R2D2~\citep{SK-GO-WD:2019}. Moreover, for VDN and QMIX, Peng's \(Q(\lambda)\) return~\citep{JP-RW:1994} is supported for improved multi-step target estimation. Together, these components make M2Bench a unified experimental system for learning-based persistent monitoring.

	\textbf{Unified evaluation metrics.} To enable fair comparison across methods, M2Bench evaluates all approaches under a unified set of monitoring-quality metrics. Following the classical patrolling literature~\citep{HS-GR-VC-BR:2004}, we consider four standard criteria: instantaneous graph idleness (IGI), average graph idleness (AGI), instantaneous worst idleness (IWI), and worst idleness (WI). The metrics are defined consistently with their classical counterparts, but are extended to incorporate the node-priority weights introduced in \cref{subsection:environment_metrics} and adapted to the event-driven evaluation setting. Specifically, at each event step, latency metrics are measured after advancing the physical time of MonitorWorld but before resetting the latency of nodes reached at that time. For a policy $\strategy$ and time $t \ge 0$, the instantaneous graph idleness is defined as
	\begin{equation*}
		\mathrm{IGI}(\strategy,t)
		:= \frac{1}{|\nodes|}\sum_{v\in\nodes}\weight(v)\,\latency{v}{\strategy}(t),
	\end{equation*}
	which measures the mean weighted latency over all nodes.
    Let \(0=t_0<t_1<\cdots<t_K=H\) be the event times recorded during an evaluation episode, where \(H\) is the evaluation horizon. Then the average graph idleness  over horizon $H$ is defined as
    \begin{equation}\label{eq:AGI}
        \mathrm{AGI}(\strategy,H)
        := \frac{1}{H}\sum_{i=1}^{K}
        \mathrm{IGI}(\strategy,t_i)(t_i-t_{i-1}).
    \end{equation}
    The instantaneous worst idleness is defined as
	\begin{equation*}
		\mathrm{IWI}(\strategy,t)
		:= \max_{v\in\nodes}\weight(v)\,\latency{v}{\strategy}(t),
	\end{equation*}
	which is exactly the worst-case weighted latency $\maxlatency{\strategy}{t}$ introduced in \cref{subsection:environment_metrics}. Finally, the worst idleness after the transient cutoff \(T\) and over the evaluation horizon $H$ is defined as
    \begin{equation} \label{eq:WI}
		\mathrm{WI}_T(\strategy,H)
		:= \sup_{T\le t \le H}\mathrm{IWI}(\strategy,t).
	\end{equation}
	When all node weights are identical, these definitions recover the standard unweighted idleness metrics up to a constant scaling factor. This unified metric suite allows heuristic, imitation-based, and reinforcement-learning-based methods to be evaluated under a common protocol.

	\section{Experiments}\label{section:experiments}
	In this section, we empirically evaluate the proposed framework from four complementary perspectives. First, we examine our main methodological contribution, namely the equivalent monitoring MDP formulation and its associated learning pipeline, and validate its effectiveness across several monitoring scenarios. Second, we situate the proposed method within a broader benchmarking context by comparing heuristic and learning-based approaches under the unified M2Bench platform, demonstrating both the practical advantages of the proposed method and the ability of M2Bench to support integrated evaluation. Third, we evaluate the scalability and deployment robustness of the proposed learning method. For scalability, we train and test the method on instances with different numbers of robots to examine whether it can be applied effectively across team sizes. For robustness, we deploy trained policies on stochastic variants of the training graphs, where the actual edge traversal times are perturbed by uniform noise and therefore differ from the nominal traversal times used during training, and evaluate whether the resulting performance remains stable.  Finally, we conduct ablation studies to assess the contribution of the main algorithmic components.
    
To ensure a comprehensive evaluation, we conduct experiments on a broad collection of classical patrolling benchmark maps with diverse topological and application characteristics, covering both uniform and weighted-node settings. Detailed descriptions of the monitoring environments are provided in~\cref{appendix:graphs}. For learning-based methods, the training curves are recorded from the training modules of M2Bench, while the final monitoring performance is obtained by executing the trained policies deterministically in MonitorWorld, using the greedy action at each decision step, and evaluating the resulting trajectories using  the unified monitoring metrics defined in Section~\ref{section:platform}.

	%
    
\subsection{Comparison of RL solvers for the TWLO-MDP}
	We first evaluate the effectiveness of the learning pipeline introduced in \cref{sub:algorithm}. Using M2Bench, we compare several representative RL and MARL algorithms under identical monitoring environments and evaluation protocols. This experiment has two goals. First, it examines whether the proposed TWLO-MDP formulation can be effectively optimized by learning-based methods to produce high-quality monitoring policies. Second, it characterizes the empirical behavior of different policy-gradient-based and value-based solvers under the proposed monitoring formulation. The training curves recorded during policy optimization are shown in~\cref{fig:learning_performance}. After training, each learned policy is executed in MonitorWorld, and the resulting trajectories are evaluated using the unified monitoring metrics defined in~\cref{section:platform}. The final monitoring performance is summarized in~\cref{tab:twlo_marl_comparison}.

	
\begin{figure*}[htbp]
		\centering
		\subfloat[Long-edge graph in \cref{fig:example} with $3$ robots.]{
			\includegraphics[width=0.32\textwidth]{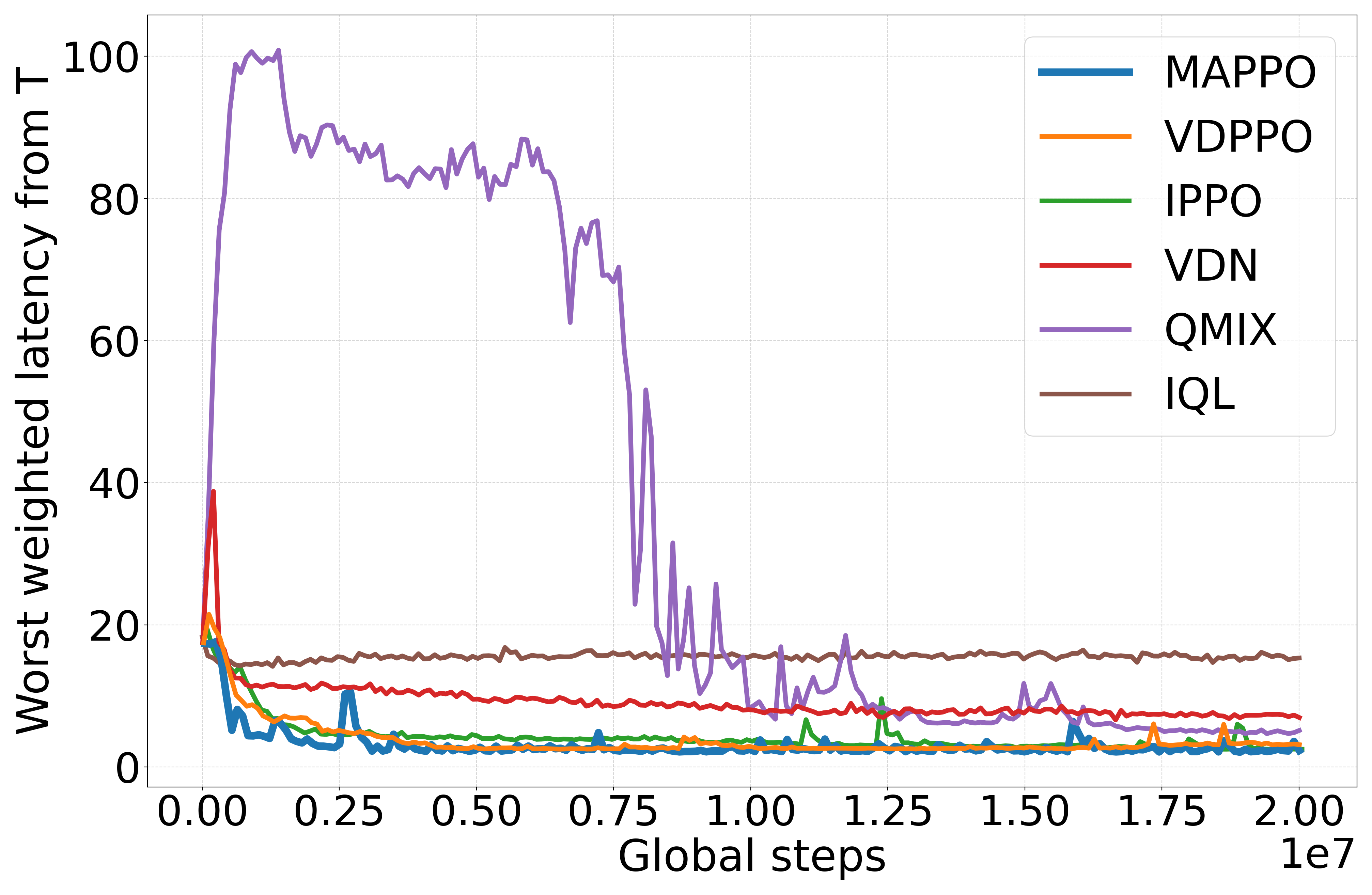}
			\label{fig:alg_longedge}
		}
		\subfloat[San Francisco  map in \cref{fig:SF} with $3$ robots.]{
			\includegraphics[width=0.32\textwidth]{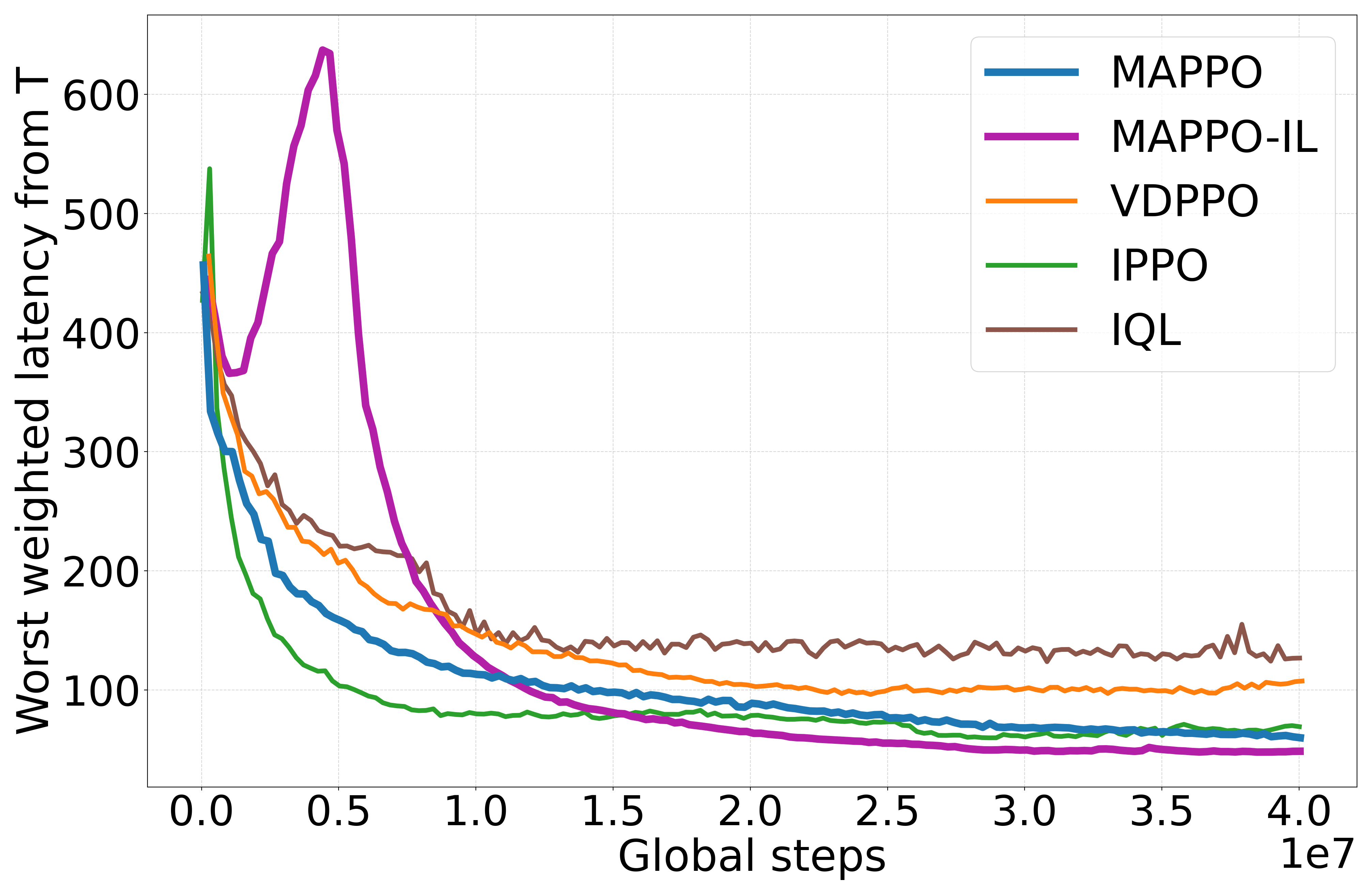}
			\label{fig:alg_SF}
		}
        \subfloat[Grid in \cref{fig:grid} with $6$ robots.]{
			\includegraphics[width=0.32\textwidth]{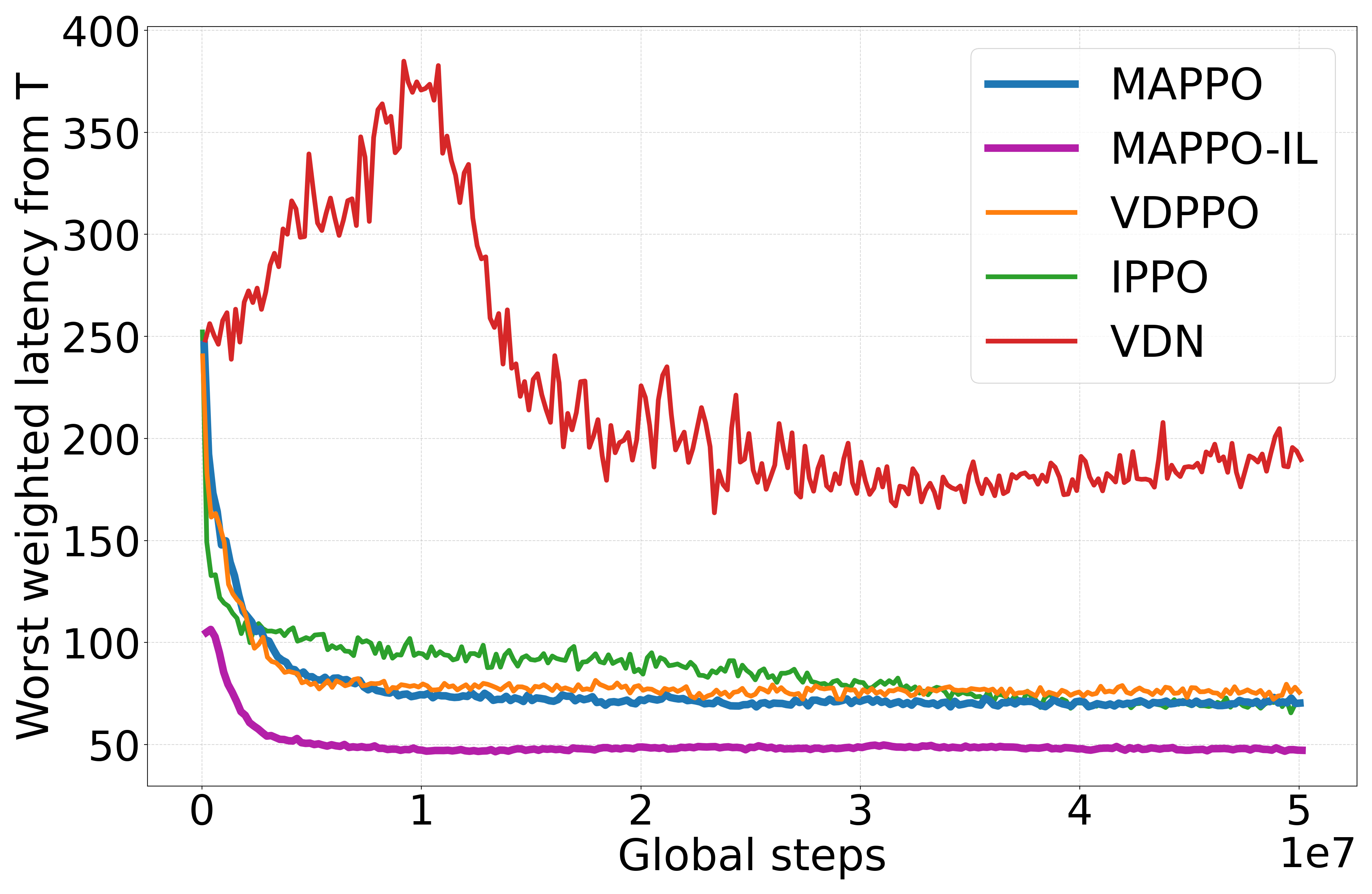}
			\label{fig:alg_grid}
		}
		\caption{Learning curves of representative RL and MARL algorithms under the proposed TWLO-MDP  formulation. The curves are recorded during policy training in M2Bench.  The y-axis reports the worst-case weighted latency measured after the transient period $T$, as defined in~\eqref{eq:defofobj}; lower values indicate better monitoring performance. }
		\label{fig:learning_performance}
	\end{figure*}

   \begin{table}[htbp]
    \centering
    \small
    \caption{Final monitoring performance of TWLO-MDP-based strategies obtained by different RL and MARL algorithms. Each entry reports WI / AGI, as introduced in~\eqref{eq:WI} and~\eqref{eq:AGI}. The best value for each environment is highlighted in bold. N/A indicates that the corresponding training process does not converge to a usable policy.}
    \label{tab:twlo_marl_comparison}
    \begin{tabular}{lccc}
        \toprule
        Algorithm & Long-edge & SF map& Grid \\
        \midrule
        MAPPO      & \textbf{1.5 / 0.42} & 50.8 / 20.59 & 51.5 / 12.12 \\
        MAPPO-IL & --                  & \textbf{47.1 / 18.26} & \textbf{40 / 10.89} \\
        VDPPO      & 3 / 0.77             & 83.5 / 22.89 & 62 / 11.89 \\
        IPPO       & 2 / 0.72             & 68.1 / 20.33 & 82 / 15.88 \\
        VDN        & 4.5 / 2.11           & 973.9 / 128.84 & 179.5 / 23.70 \\
        QMIX       & 3 / 1.34             & N/A & N/A \\
        IQL        & 15 / 2.47            & 88.7 / 21.63 & N/A \\
        \bottomrule
    \end{tabular}
\end{table}

The results in~\cref{fig:learning_performance} and~\cref{tab:twlo_marl_comparison} show that the proposed TWLO-MDP formulation can be effectively optimized by policy-gradient-based methods. MAPPO-based methods provide the most effective and stable performance among the tested algorithms. In particular, MAPPO, IPPO, and VDPPO generally reduce the worst-case weighted latency during training and produce competitive final monitoring policies. On the long-edge graph, MAPPO attains the same worst-case weighted latency as the analytical optimum derived in~\cref{example}, providing empirical evidence consistent with the equivalence result established in~\cref{thm:MDPopt,cor:optimality}. By contrast, value-based methods are less stable under the proposed formulation. IQL performs poorly even on the relatively simple long-edge graph, while both IQL and QMIX fail to converge on the more complex San Francisco and Grid environments. VDN can produce improving learning curves in some cases, but its final policies remain far from competitive in the more complex scenarios.

The results also indicate that imitation-based initialization is beneficial in challenging monitoring scenarios. On the San Francisco crime map and the Grid environment, MAPPO-IL achieves the best final performance and maintains a clear advantage over MAPPO trained from scratch. This suggests that heuristic demonstrations provide useful behavioral guidance for policy optimization, especially when the graph topology is complex or multi-robot coordination is difficult. Based on these observations, we use MAPPO-based learning methods in the following benchmark comparisons.

	\subsection{Benchmark comparison on M2Bench}
	We next conduct a unified benchmark comparison using the algorithms integrated in M2Bench. In contrast to the previous experiment, which compares different RL solvers under the same TWLO-MDP formulation, this experiment compares different monitoring paradigms, including the proposed TWLO-MDP-based learning methods, existing learning-based monitoring algorithms summarized in~\cref{table:RL}, and classical heuristic patrolling policies discussed in~\cref{section:platform}. All methods are evaluated under the same monitoring environments, execution protocols, and monitoring-quality metrics. 
    
    For learning-based methods, the curves in~\cref{fig:comparison_special,fig:comparison_synthetic,fig:comparison_realistic} report the worst-case weighted latency measured after the transient period $T$. Although different learning-based algorithms may use different reward designs during training, all curves are plotted  using the same  worst-case weighted latency metric to provide a consistent view of their progress with respect to the target monitoring objective. Thus, these curves should be interpreted as evaluation curves recorded during training rather than algorithm-specific training losses or rewards. For readability, learning-based methods that fail to converge or exhibit severe instability within the training budget are omitted from the plotted curves, while their final outcomes are reported as N/A in the tables.
    
    The final monitoring performance of all compared methods is summarized in~\cref{tab:worst_latency_comparison,tab:average_latency_comparison}, using the worst-case latency and average-latency metrics, respectively. 
    For heuristic methods, we report the average performance over $10$ evaluation episodes when  random decisions are involved; N/A indicates that the executed policy does not cover all nodes within the fixed evaluation horizon, so the worst-case latency continues to increase.

\begin{table*}[ht]
\small
\centering
\caption{Final monitoring performance of different algorithms on benchmark patrolling environments under the WI metric defined in~\eqref{eq:WI}. The best result in each environment is highlighted in bold. N/A indicates either failure to obtain a usable learned policy for learning-based methods or incomplete node coverage within the evaluation horizon for heuristic methods.}
\label{tab:worst_latency_comparison}
\begin{tabular}{lcccccccccc}
\toprule
Algorithm & Long-edge & Chain & MapA & MapB & Grid & Island & Cumberland & Milwaukee & Marostica & SF map \\
\midrule
TWLO & \textbf{1.5} & \textbf{4} & 88.5 & 125.3 & 51.5 & 111.7 & 106.0 & 148 & {51} & 50.8\\
TWLO-IL & -- & -- & \textbf{64.5} & \textbf{85.4} & \textbf{40} & \textbf{71.1} & \textbf{95.9} & \textbf{124.1} & \textbf{29} & \textbf{47.1}  \\
BBLA & 9 & 10 & 153.0 & 178.0 & 139 & 239.1 & 268 & 270 & 69 & 166.5 \\
GBLA & 16.6 & 8.8 & 110.4 & 154.9 & 100 & 140.4 & 165 & 250 & 52 & 272.9 \\

EGBLA & 11 & 8 & 177 & 100.5 & 84 & 268 & 108 & 200 &54 & N/A \\

NEP & 14.4 & 49.3 & N/A & N/A & N/A & N/A & N/A & N/A & N/A & 339.4 \\

S4R1 & 9 & 10 & 118.0 & 130.8 & 44 & 108 & 241 & 180 & 57 & 87.1 \\

BEAU & 12 & 6 & 153.0 & N/A & N/A & N/A &N/A  & 359.0  & N/A & 240 \\

MAGEC & 6 & 7 & 352.6 & 303.7 & N/A & N/A & N/A & N/A & 111& N/A \\

SUNS & 16 & 9 & 184.7 & 239.3 & N/A & 241.9 & 220.8 & 371.1 & 95.5 & 104.3 \\

OUCS & 9 & 14 & 132.4 & 150.0 & 75 & 232 & 256 & 355 & 66 & 126 \\
\midrule
CC & 12.5 & 20 & 259 & 218.5 & 169.7 & 249.1 & 209.3 & 459.3 & 124.3 & 295.8 \\
CR & 13 & 20 & 236.4 & 164.1 & 203 & 326 & 231 & 383.1 & 214 & 316.3 \\
RANDOM & 15.9 & 37.5 & 426.4 & 437 & 235.3 & 458.9 & 475 & 775.1 & 286.1 & 405.7 \\
MSP & 16 & N/A & N/A & N/A & N/A & N/A & 502.1 & N/A & N/A & 140.4 \\
HCR & 17 & 16 & 238.96 & 175.18 & 228 & 339.02 & 298.8 & 323.82 & 149 & 187 \\
HPCC & 17 & 19.5 & 257.92 & 295.62 & 112 & 374.28 & 371.63 & 551.4 & 109 & 187.4 \\
GBS & 12.5 & 18.8 & 234.2 & 190.5 & 160.5 & 263.7 & 207 & 481.4 & 104.7 & 289.4 \\
SEBS & 12.9 & 15.8 & 246.9 & 219.1 & 167 & 295.2 & 199.4 & 402.3 & 122.2 & 297.3 \\
ER & 11 & 20 & 209 & 125.6 & 209 & 243.4 & 335.6 & 459 & 104 & 269.3 \\
CBLS & 11.9 & 11.5 & 221.7 & 214 & 105.8 & 260.7 & 233.5 & 460.7 & 101.9 & 307.6 \\
DTAG & 13 & 16 & 212.1 & 199.8 & 184 & 323.9 & 266.6 & 469 & 174 & 300 \\
DTAP & 13 & 18.4 & 203.3 & 221.8 & 98.5 & 228.6 & 192.8 & 400.6 & 96 & 276.6 \\
BAPS & 13 & 20 & 107.1 & 96.4 & 167 & 122.7 & 135.7 & 201.4 & 149 & 122.4 \\
AHPA & 11 & 13 & 132 & 291.6 & 173 & 115.7 & 131.4 & 360.7 & 175 & 138.5 \\
\bottomrule
\end{tabular}
\end{table*}
    

\begin{table*}[ht]
\small
\centering
\caption{Final monitoring performance of different algorithms on benchmark patrolling environments under the AGI metric defined in~\eqref{eq:AGI}. The best result in each environment is highlighted in bold. N/A indicates either failure to obtain a usable learned policy for learning-based methods.}
\label{tab:average_latency_comparison}
\begin{tabular}{lcccccccccc}
\toprule
Algorithm & Long-edge & Chain & MapA & MapB & Grid & Island & Cumberland & Milwaukee & Marostica & SF map \\
\midrule
TWLO & \textbf{0.42} & \textbf{1.58} & 18.72 & 26.91 &  12.12 & 31.78 & 26.38 & 47.42 & 12.03 & 20.59 \\
TWLO-IL & --  & -- & \textbf{17.86} & \textbf{20.80} & {10.89}   & \textbf{26.46} & \textbf{25.96} & 44.70 &  \textbf{10.20} & \textbf{18.26 }\\
BBLA & 1.35 & 2.83 & 31.76 & 27.80 & 25.76 & 56.44 & 40.65 & 53.02 & 15.01 & 42.92 \\
GBLA & 1.74 & 2.99 & 25.08 & 23.32 & 18.69 & 32.37 & 34.36 & 52.06 & 12.67 & 46.92 \\
EGBLA & 2.43 & 2.54 & 27.84 & 22.49 & 13.00 & 36.06 & 28.05 & 47.71 & 11.78 & N/A \\
NEP & 2.12 & 6.43 & N/A & N/A & N/A & N/A & N/A & N/A & N/A & 52.32 \\
S4R1 & 1.56 & 2.90 & 20.27 & 21.62 & \textbf{9.91} & 28.73 & 27.96 & \textbf{44.59} & 10.59 & 20.52 \\
BEAU & 2.47 & 1.88 &  31.00 &N/A  &N/A  & N/A & 84.65 & 65.62 & N/A & 52.08 \\
MAGEC & 1.20 & 2.01 & 44.89 & 50.90 & 37.98 & N/A & 49.26 & 102.00 & 18.17 & N/A \\
SUNS & 4.13 & 2.90 & 31.89 & 49.26 & N/A & 51.41 & 43.83 & 68.61 & 14.26 & 32.80 \\
OUCS & 1.91 & 3.19 & 20.41 & 26.02 & 14.58 & 60.77 & 43.72 & 57.48 & 14.07 & 24.92 \\
\midrule
CC & 6.54 & 7.21 & 48.14 & 33.16 & 42.23 & 59.51 & 43.42 & 105.52 & 28.00 & 80.35 \\
CR & 7.78 & 7.27 & 55.32 & 28.12 & 51.73 & 77.00 & 53.87 & 101.63 & 34.22 & 88.07 \\
RANDOM & 3.47 & 4.61 & 57.86 & 60.96 & 27.03 & 77.10 & 74.68 & 125.22 & 33.99 & 55.57 \\
MSP & 4.76 & 33.68 & 144.09 & 141.70 & 225.03 & 197.57  & 203.78 & 322.37 & 226.43 & 41.99 \\
HCR & 7.09 & 4.31 & 37.75 & 31.38 & 47.25 & 59.91 & 49.65 & 78.66 & 28.5 & 50.83 \\
HPCC & 7.09 & 5.82 & 54.62 & 48.37 & 37.07 & 100.44 & 99.46 & 135.48 & 38.2 & 50.89 \\
GBS & 6.26 & 6.56 & 42.82 & 34.21 & 39.54 & 58.26 & 46.25 & 96.91 & 23.32 & 76.59 \\
SEBS & 7.06 & 5.23 & 46.77 & 37.34 & 40.21 & 73.21 & 43.72 & 91.56 & 27.42 & 78.76 \\
ER & 5.06 & 7.14 & 41.33 & 26.88 & 50.85 & 48.83 & 68.13 & 103.08 & 27.77 & 69.09 \\
CBLS & 4.71 & 2.74 & 42.95 & 35.30 & 25.69 & 54.19 & 49.77 & 98.59 & 20.65 & 82.14 \\
DTAG & 7.45 & 4.82 & 43.24 & 28.43 & 51.83 & 76.88 & 64.55 & 110.41 & 31.76 & 80.23 \\
DTAP & 7.26 & 5.56 & 37.68 & 34.92 & 21.67 & 49.95 & 40.68 & 86.10 & 19.63 & 65.28 \\
BAPS & 7.78 & 7.14 & 21.75 & 22.09 & 38.49 & 27.21 & 27.11 & 52.73 & 33.69 & 28.81 \\
AHPA & 5.30 & 3.78 & 24.38 & 35.80 & 38.16 & 28.77 & 27.33 & 67.20 & 39.01 & 30.98 \\
\bottomrule
\end{tabular}
\end{table*}


  The results in~\cref{fig:comparison_special,fig:comparison_synthetic,fig:comparison_realistic} and~\cref{tab:worst_latency_comparison,tab:average_latency_comparison} show that the TWLO-MDP-based learning methods achieve the strongest  performance across the benchmark suite. In particular, TWLO and TWLO-IL consistently obtain the lowest worst-case latency values in most environments. This is expected because the proposed MDP formulation directly targets the tail worst-case latency objective after the transient period. More importantly, the advantage is not limited to the optimized metric: under the average-latency metric, the TWLO-MDP-based methods also achieve the best performance in most environments or remain very close to the best-performing baseline. By contrast, the competing algorithms exhibit more environment-dependent behavior. Among existing learning-based baselines, GBLA performs reasonably well on MapA and Island but becomes much less competitive on the Long-edge and San Francisco maps. Similarly, MAGEC achieves good performance on the Long-edge and chain graphs, but performs poorly on MapA and MapB and fails to produce usable policies on several larger environments. Classical heuristic methods can perform reasonably well on certain topologies, especially when their hand-crafted decision rules match the graph structure, but their performance is substantially less consistent than that of the TWLO-MDP-based methods. These observations suggest that the TWLO-MDP-based learning framework provides a more adaptive and reliable way to optimize monitoring policies across heterogeneous environments.  
    \begin{figure*}[htbp]
    \centering
    \subfloat[Long-edge graph in \cref{fig:example} with $3$ robots.]{
        \includegraphics[width=0.36\textwidth]{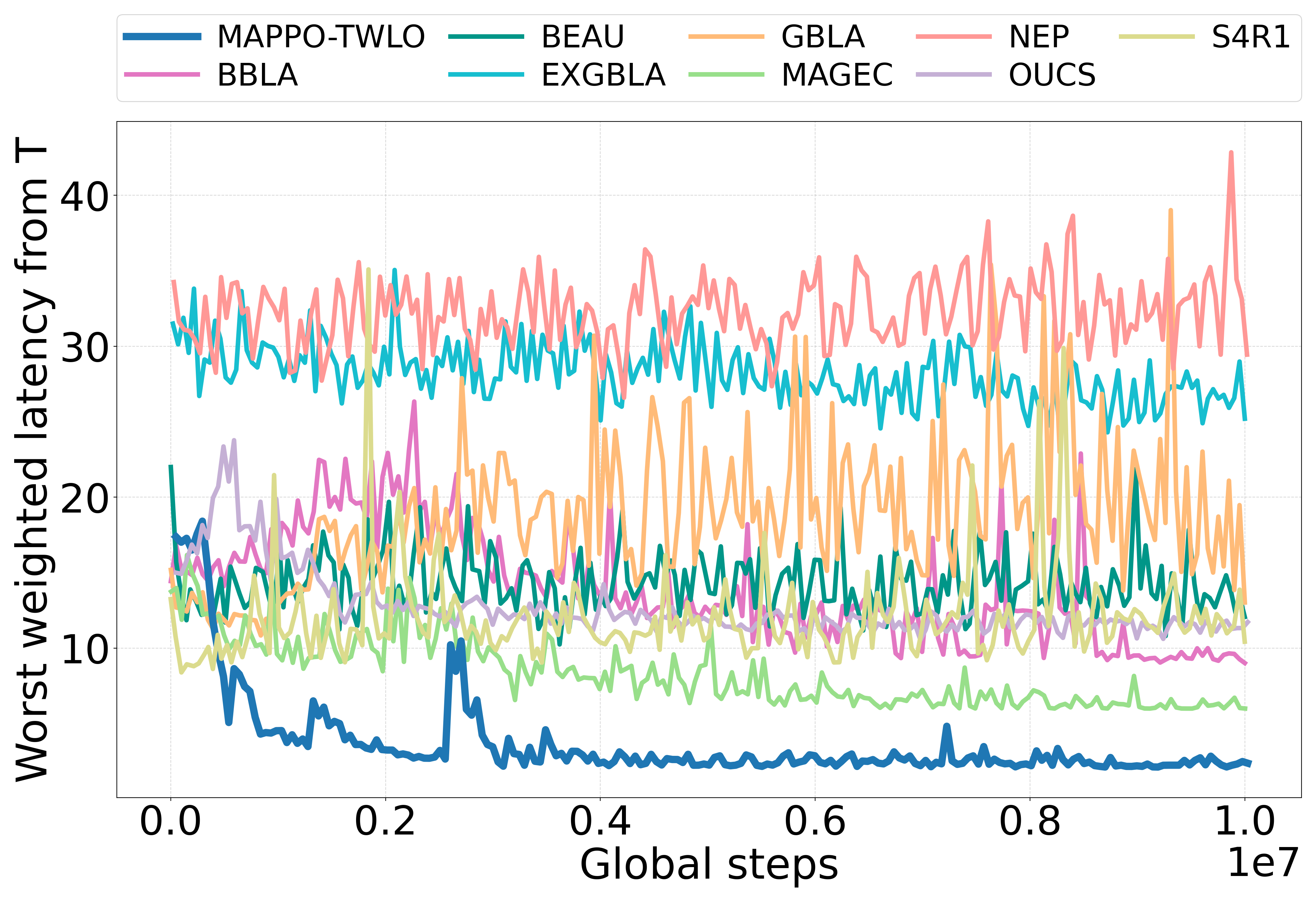}
        \label{fig:all_longedge}
    }
    \subfloat[Chain in \cref{fig:chain} with $4$ robots.]{
        \includegraphics[width=0.36\textwidth]{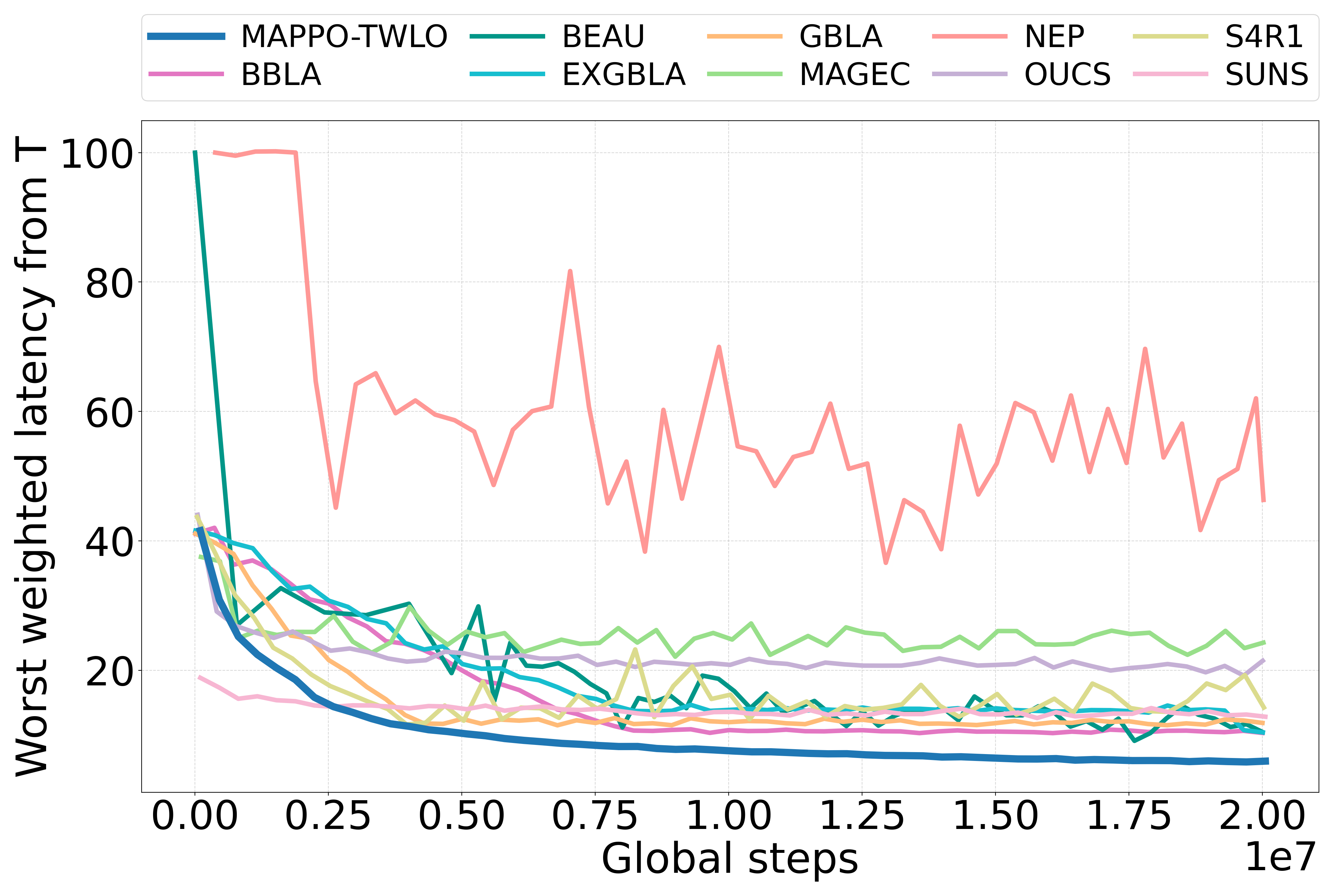}
        \label{fig:all_chain}
    }
    \caption{Learning curves of learning-based benchmark algorithms on special graph topologies.  The curves report the worst-case latency metric measured after the transient period \(T\). 
    The compared algorithms are summarized in~\cref{table:RL}; for readability, algorithms that fail to converge or exhibit severe instability are omitted from the curves.}
    \label{fig:comparison_special}
\end{figure*}

    For special graph topologies with known analytical optima, the benchmark results provide direct empirical evidence for the proposed formulation. On the long-edge graph in \cref{fig:example}, the evaluation results in \cref{tab:worst_latency_comparison} show that TWLO attains the same worst-case latency value as the analytical optimum derived in \cref{example}. On the chain graph, TWLO also finds the optimal strategy characterized in \cref{fig:chain}, which follows from the analytical result in~\citep{FP-AF-FB:2012}. These observations are consistent with the theoretical equivalence established in \cref{thm:MDPopt,cor:optimality}: optimizing the TWLO-MDP can recover an optimal solution to the original monitoring problem. 

\begin{figure*}[htbp]
    \centering
    \subfloat[MapA in \cref{fig:mapA} with $6$ robots.]{
        \includegraphics[width=0.31\textwidth]{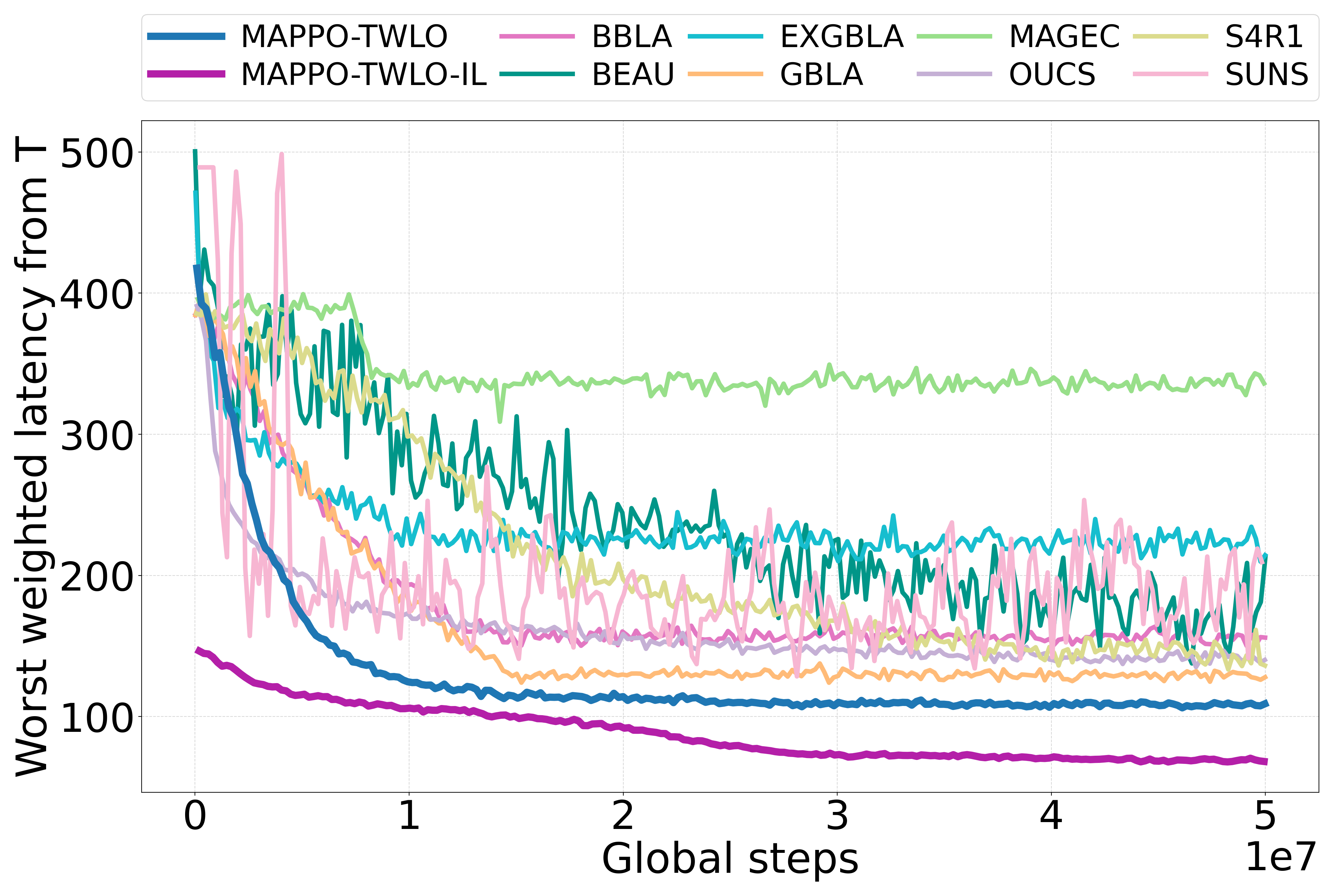}
        \label{fig:all_mapA}
    }
    \subfloat[MapB in \cref{fig:mapB}  with $6$ robots.]{
        \includegraphics[width=0.31\textwidth]{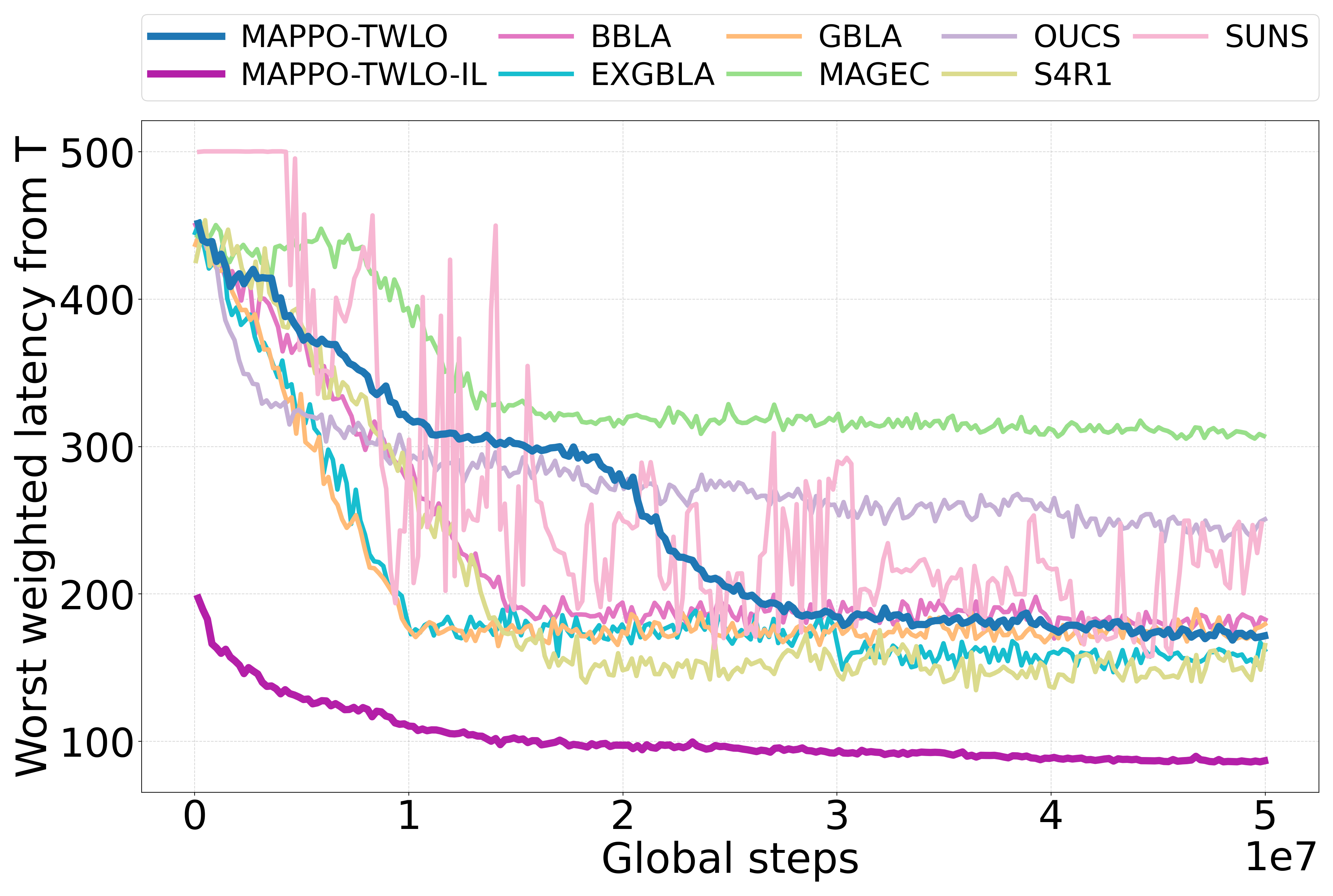}
        \label{fig:all_mapB}
    }
    \subfloat[Grid in \cref{fig:grid}  with $6$ robots.]{
        \includegraphics[width=0.31\textwidth]{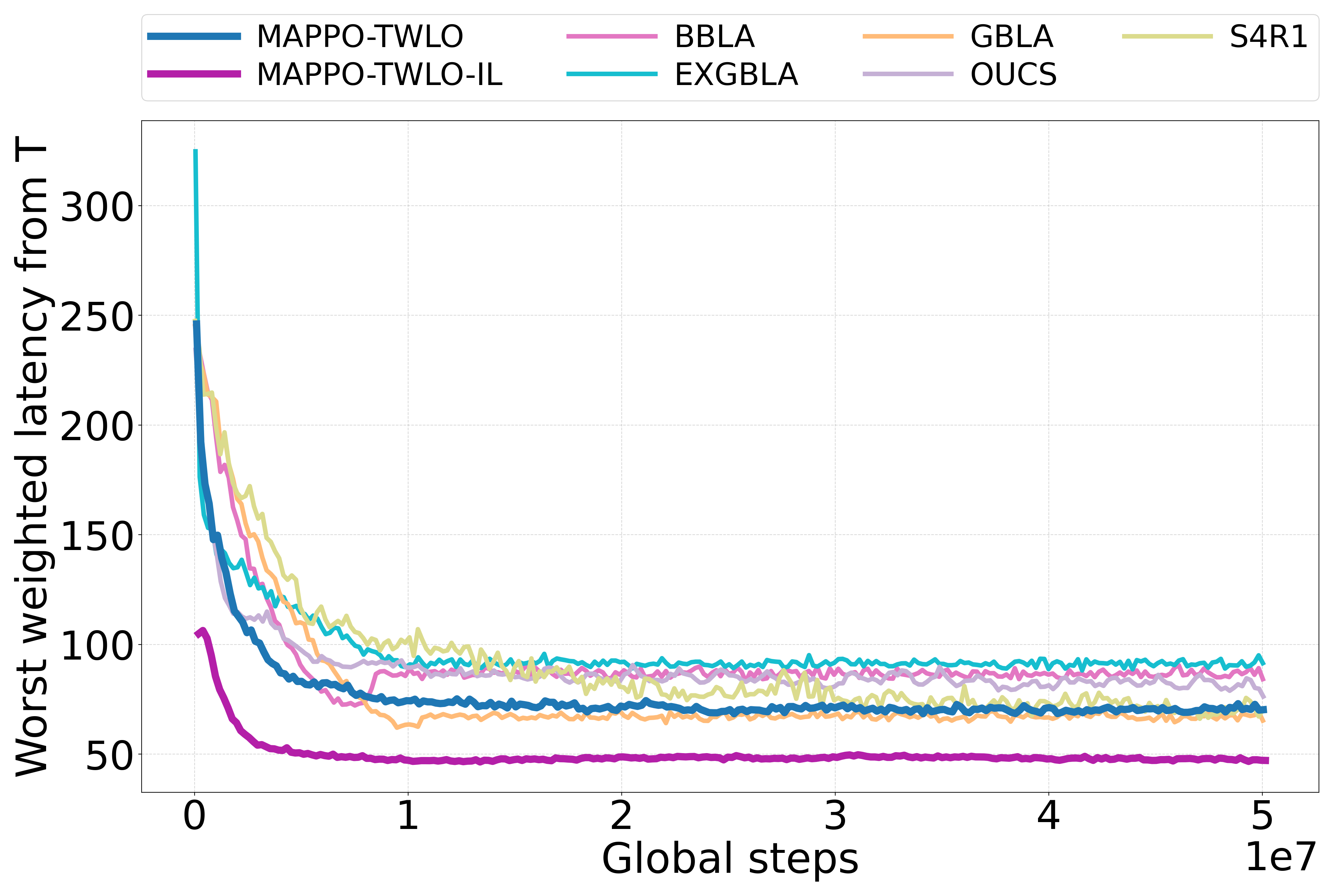}
        \label{fig:all_grid}
    }
    
    \subfloat[Island in \cref{fig:island} with $6$ robots.]{
        \includegraphics[width=0.31\textwidth]{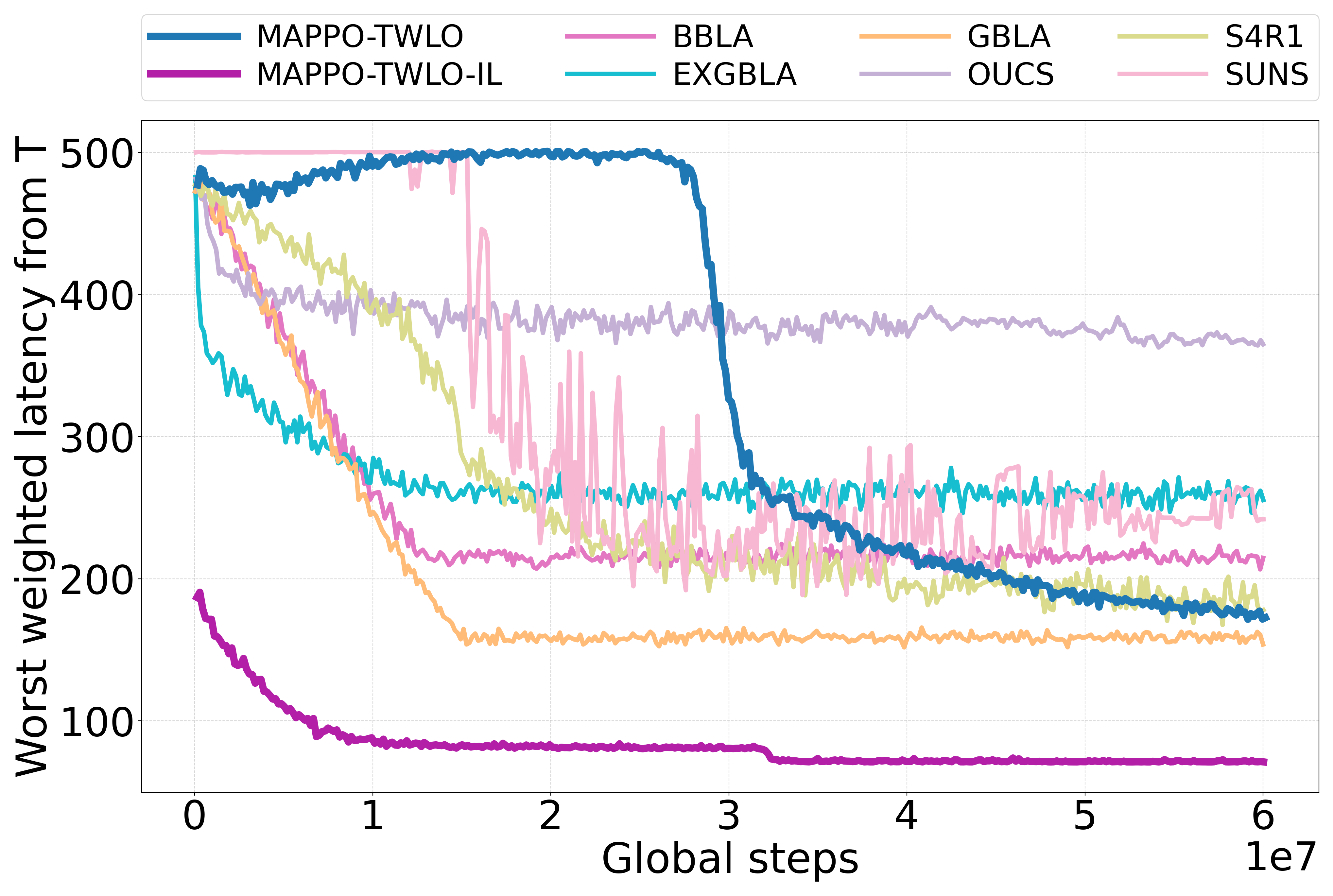}
        \label{fig:all_island}
    }
    \subfloat[Cumberland in \cref{fig:cumberland} with $6$ robots.]{
        \includegraphics[width=0.31\textwidth]{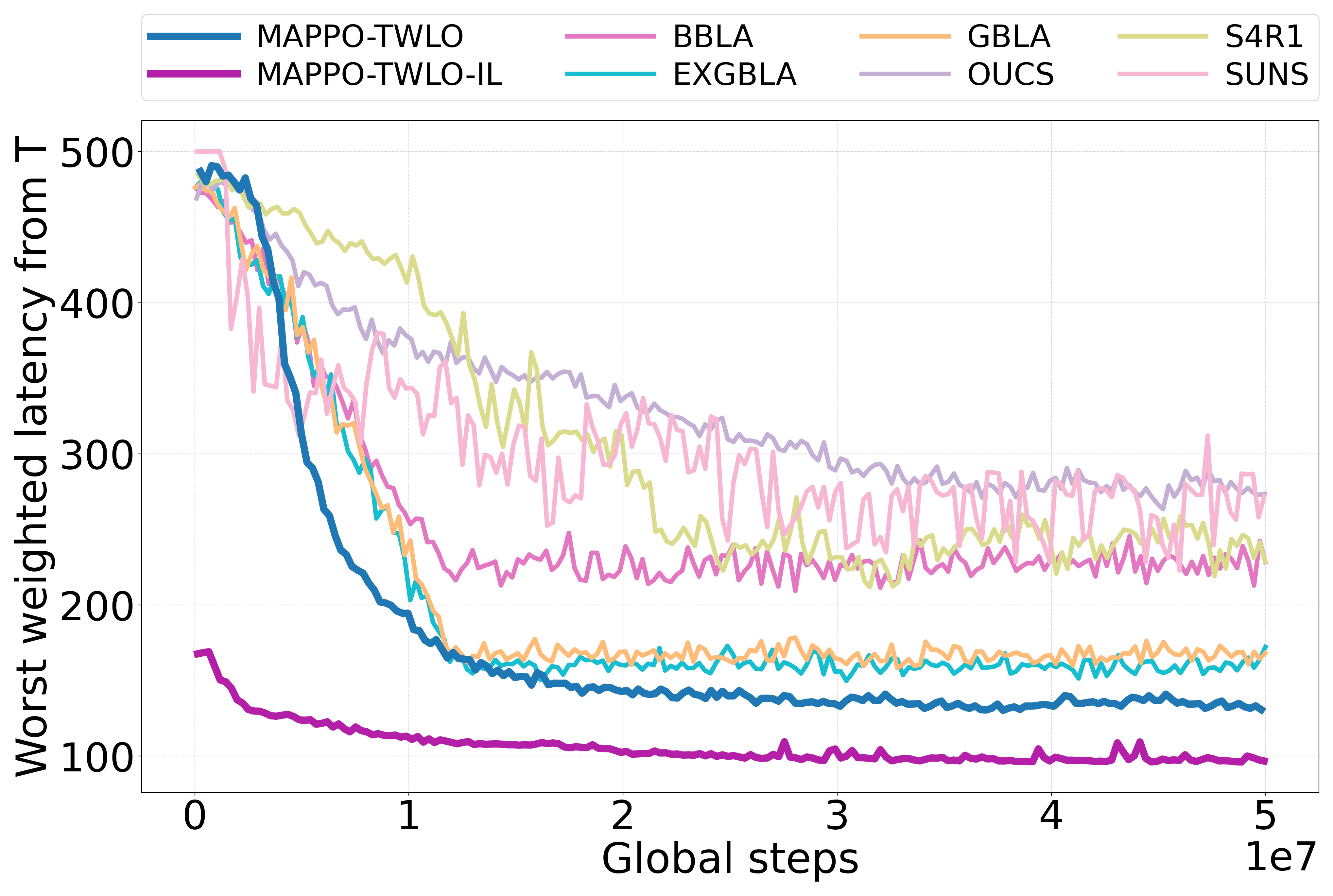}
        \label{fig:all_cumberland}
    } 
    \subfloat[Milwaukee in \cref{fig:milwaukee} with $6$ robots.]{
        \includegraphics[width=0.31\textwidth]{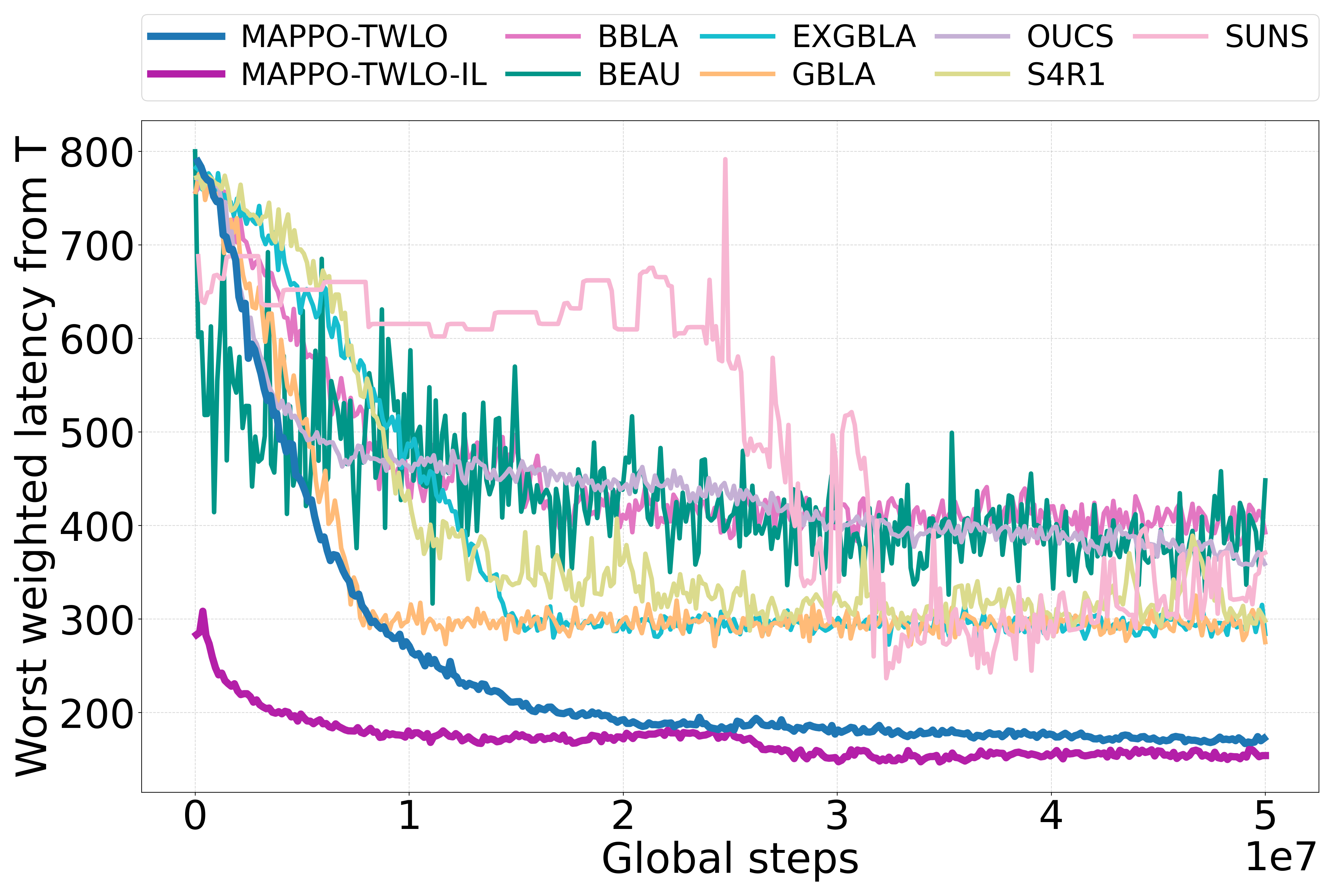}
        \label{fig:all_milwaukee}
    } 
    \caption{Learning curves of learning-based benchmark algorithms on large-scale synthetic surveillance environments. The curves report the worst-case latency metric measured after the transient period \(T\). 
    The compared algorithms are summarized in~\cref{table:RL}; for readability, methods that fail to converge or exhibit severe instability are omitted from the curves.}
    \label{fig:comparison_synthetic}
\end{figure*}

For large-scale synthetic environments in \cref{fig:graph}, the results  in~\cref{fig:comparison_synthetic,tab:worst_latency_comparison,tab:average_latency_comparison} further highlight the benefit of imitation-based initialization. As the graph size and structural complexity increase, training from scratch becomes more difficult, and TWLO-IL provides a more reliable starting point for policy optimization. This effect is particularly clear on MapA, MapB, Island, Cumberland, and Milwaukee, where TWLO-IL improves over TWLO trained from scratch and achieves the best WI values. Under the AGI metric, TWLO-IL also achieves the best performance in most of these environments, while S4R1 obtains slightly lower AGI on the Grid and Milwaukee maps. These results indicate that imitation-based initialization improves final monitoring quality in challenging graph environments.

\begin{figure*}[htbp]
    \centering
    \subfloat[San Francisco  map in \cref{fig:SF} with $3$ robots.]{
        \includegraphics[width=0.353\textwidth]{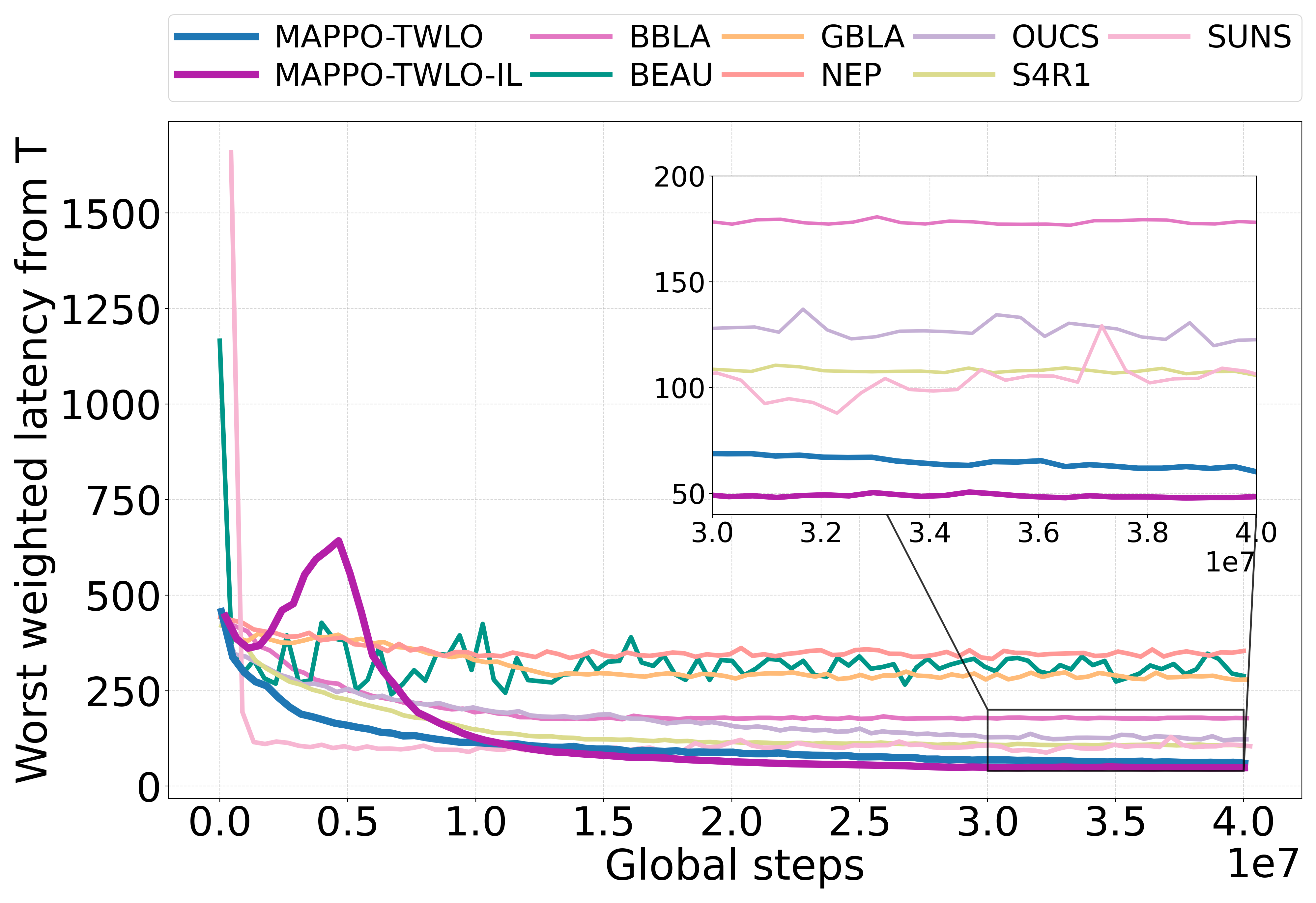}
        \label{fig:all_SF}
    }
    \subfloat[Marostica roadmap in \cref{fig:ity} with $6$ robots.]{
        \includegraphics[width=0.36\textwidth]{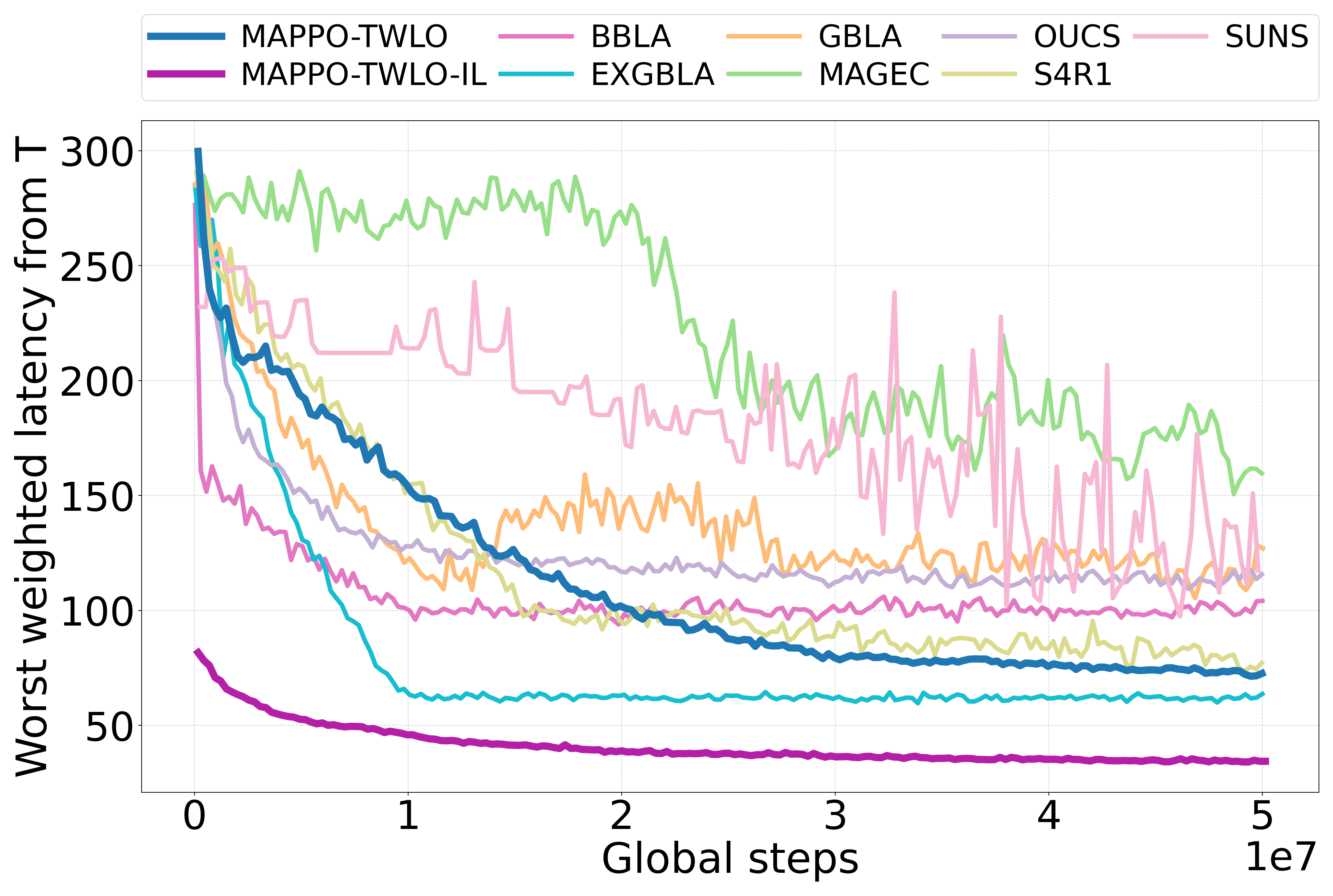}
        \label{fig:all_marostica}
    }
    \caption{Learning curves of learning-based benchmark algorithms on realistic map-based surveillance scenarios. The curves report the worst-case latency metric measured after the transient period \(T\).  
    The compared algorithms are summarized in~\cref{table:RL}; for readability, algorithms that fail to converge or exhibit severe instability are omitted from the curves.}
    \label{fig:comparison_realistic}
\end{figure*}

In the realistic map-based scenarios as shown in \cref{fig:real}, the results in~\cref{fig:comparison_realistic,tab:worst_latency_comparison,tab:average_latency_comparison} show that TWLO-MDP-based methods maintain strong performance on  graphs constructed from real-world maps. On the San Francisco map, TWLO-IL achieves the best performance under both worst-case latency and average-latency metrics, while TWLO trained from scratch also outperforms most competing methods. On the Marostica roadmap, TWLO-IL again achieves the best worst-case latency and average-latency values, and TWLO remains competitive among the remaining methods. These results suggest that the proposed learning pipeline can adapt to realistic weighted monitoring environments where fixed heuristic rules or alternative learning objectives may be less reliable.

	Overall, the benchmark comparison demonstrates two points. First, M2Bench provides a unified experimental protocol for comparing various monitoring algorithms, including both learning-based approaches and classical heuristic methods. Second, the proposed TWLO-MDP-based learning framework achieves consistently strong performance across special,  synthetic, and realistic benchmark environments, supporting the practical effectiveness of the proposed MDP-based formulation and learning pipeline.

\subsection{Scalability  and robustness analysis}
  	The previous experiments evaluate the proposed method under fixed training and testing settings. In practical applications, a monitoring framework should scale to different team sizes and remain robust when deployment conditions deviate from the training model. We examine these two aspects in this subsection. First, we study whether the proposed TWLO-MDP-based learning framework can be effectively trained and applied with different numbers of robots.  Second, we evaluate zero-shot robustness when a trained policy is deployed on stochastic variants of the training graph, where the actual edge traversal times differ from the nominal edge lengths used during training.
    

We first evaluate the proposed method with different numbers of robots on representative environments. For each team size, the policy is trained and evaluated with the corresponding number of robots. As shown in~\cref{fig:number} and \cref{tab:num_agents}, the proposed method can be successfully trained across different team sizes and produces effective monitoring policies in all tested settings.  Increasing the number of robots generally reduces both the worst-case weighted latency and the average latency, indicating that the learned policies can make use of additional monitoring resources rather than producing redundant motions. These results suggest that the TWLO-MDP-based learning framework is not restricted to a single prescribed team size and can be instantiated for different multi-robot monitoring configurations. As an additional comparison, in the one-robot San Francisco setting, which matches the monitoring scenario considered in~\citep{SA-EF-SS:2012}, our method achieves a WI value of $134$, substantially improving upon the weighted maximum latency value $202.3$ reported therein under the same travel-time scaling.
\begin{figure*}[htbp]
		\centering
		\subfloat[MapA in \cref{fig:mapA}]{
			\includegraphics[width=0.31\textwidth]{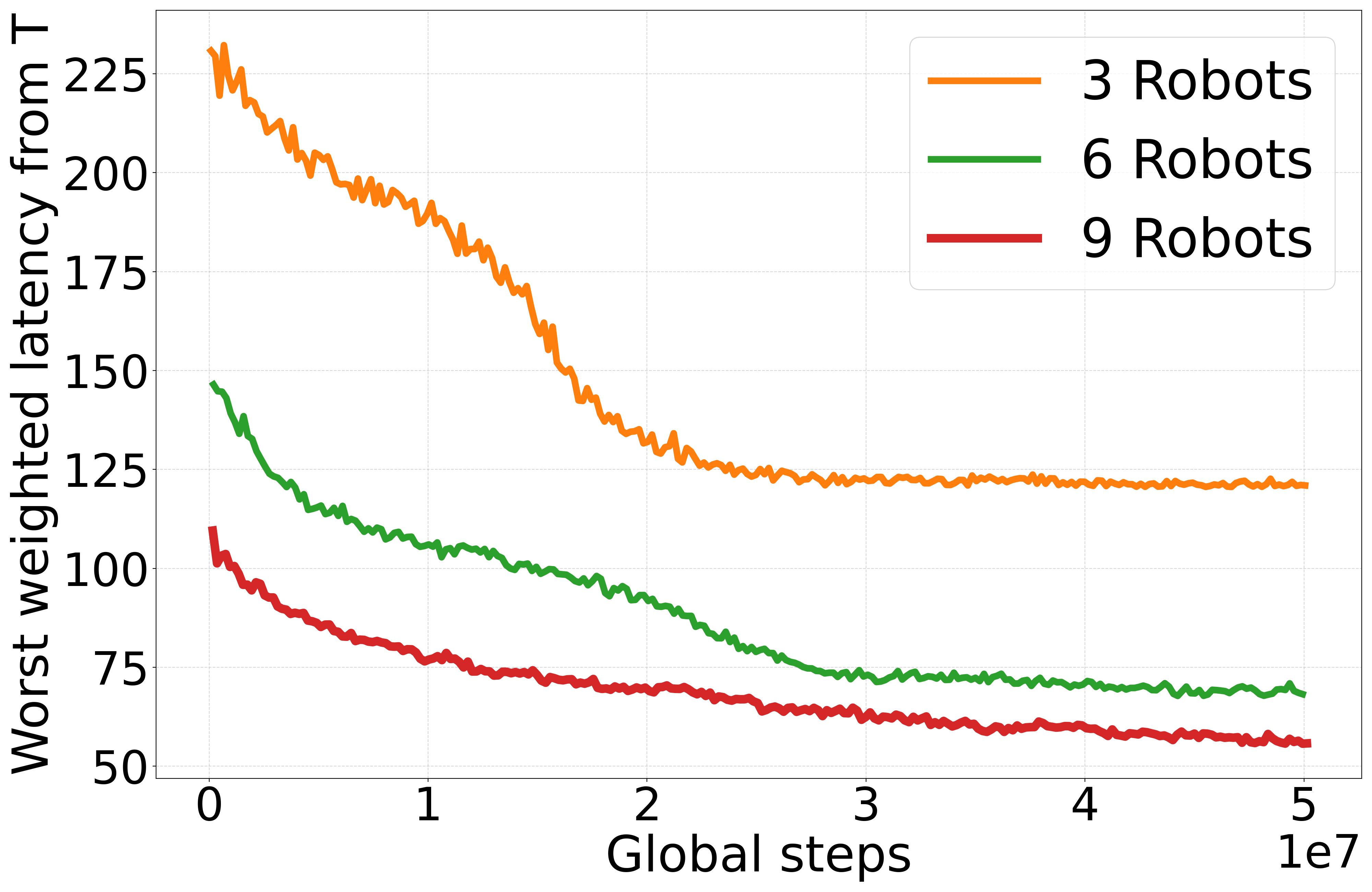}
			\label{fig:agents_comparison_mapA}
		}
        \subfloat[Grid in \cref{fig:grid}]{
			\includegraphics[width=0.31\textwidth]{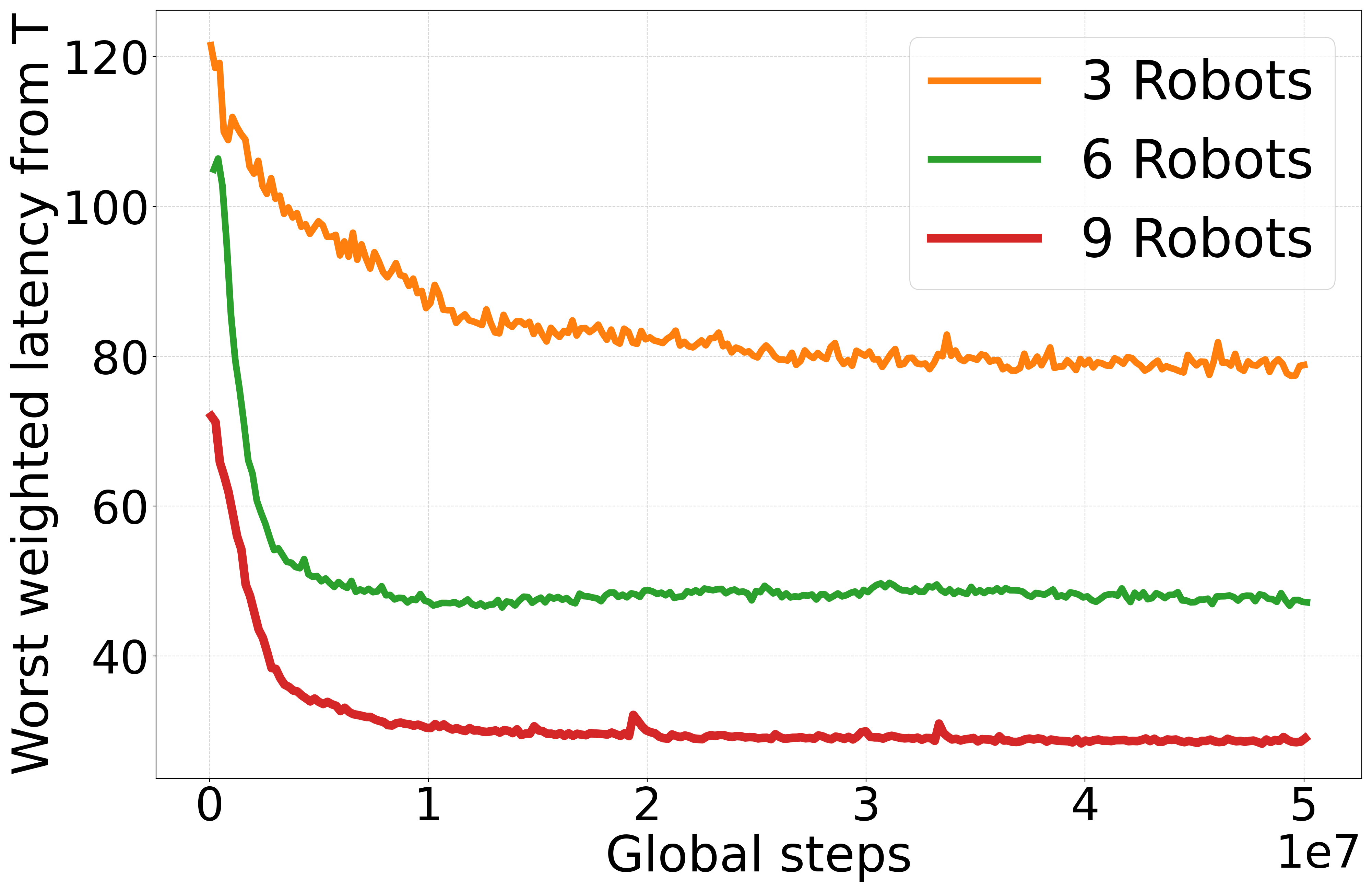}
			\label{fig:agents_comparison_grid}
		}
        \subfloat[San Francisco map in \cref{fig:SF}]{
			\includegraphics[width=0.31\textwidth]{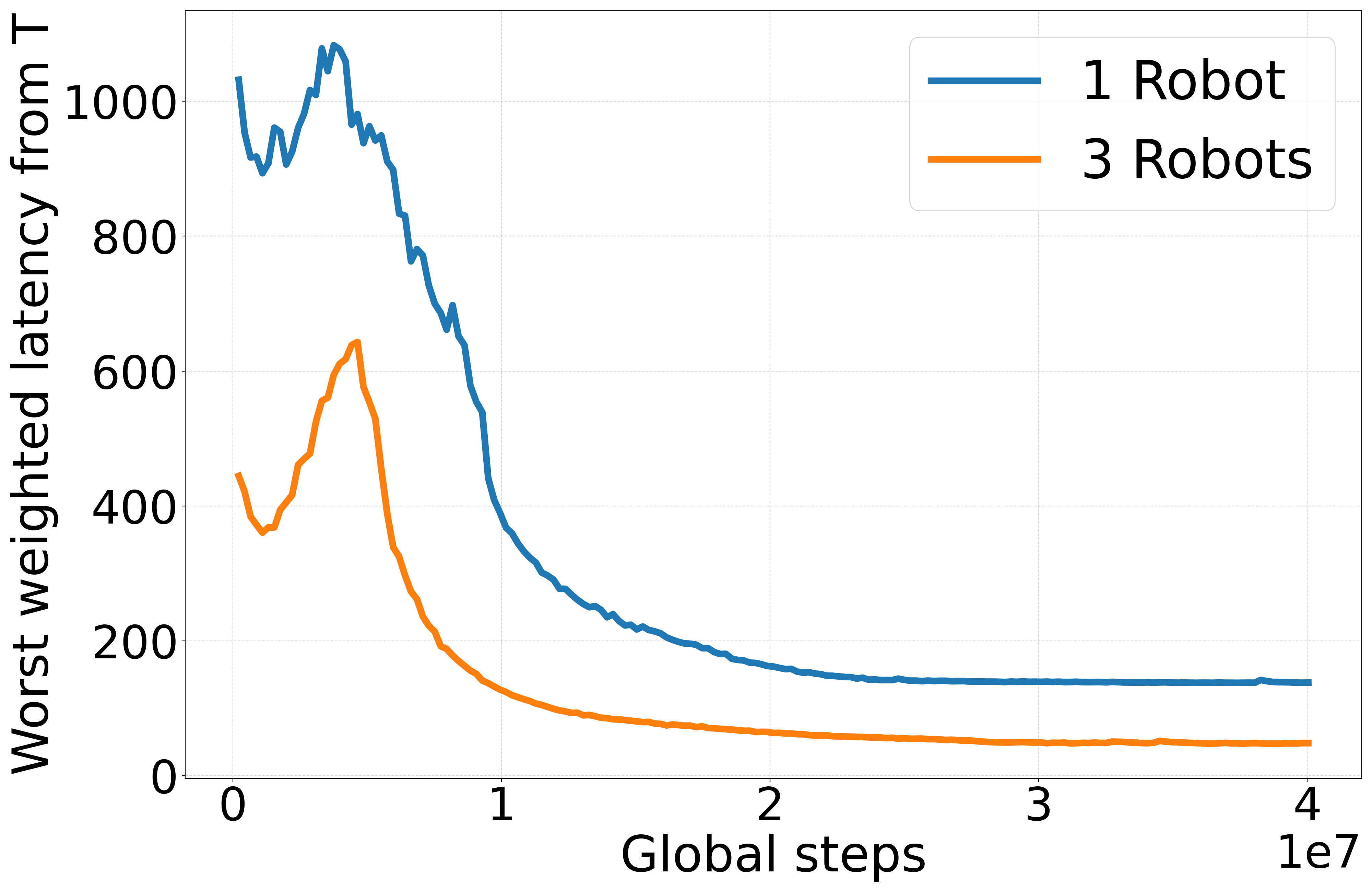}
			\label{fig:agents_comparison_SF}
		}
		\caption{Learning curves of the proposed TWLO-MDP-based method with different numbers of robots on representative environments. For each team size, the policy is trained and evaluated using the corresponding number of robots.}
		\label{fig:number}
	\end{figure*}
  \begin{table}[htbp]
    \centering
    \small
    \caption{Final monitoring performance of the proposed TWLO-MDP-based method with different numbers of robots.  The WI and AGI metrics are introduced in~\eqref{eq:WI} and~\eqref{eq:AGI}, respectively.}
    \label{tab:num_agents}
    \begin{tabular}{lccc}
        \toprule
        Graph & Number of robots  & WI & AGI \\
        \midrule
        SF        & 1 & 134 & 54.63 \\
        SF        & 3 & 47.1  & 18.26 \\
        Grid      & 3 & 64  & 19.51 \\
        Grid      & 6 & 40 & 10.89 \\
        Grid      & 9 & 24  & 7.57 \\
        MapA      & 3 & 120.6 & 33.30 \\
        MapA      & 6 & 64.5  & 17.86 \\
        MapA      & 9 & 50.8  & 12.23 \\
        \bottomrule
    \end{tabular}
\end{table}

We next evaluate zero-shot robustness to stochastic edge traversal times. In this experiment, the policy is trained on the nominal deterministic graph, where traversal times are determined by the original edge lengths. During evaluation, the graph topology and node priorities are kept unchanged, but the actual traversal time of each edge is randomly perturbed around its nominal value. No additional training or fine-tuning is performed on the perturbed environments. Specifically, for each edge $e \in \edges$, the perturbed traversal time is generated as
$$
\widetilde{A}(e)=A(e)\left(1+\epsilon_e\right), \quad \epsilon_e \sim \text { Uniform }[-\rho, \rho],
$$
where $\rho$ controls the perturbation level. For each value of \(\rho\), we evaluate the trained policy over 
$100$ episodes with independently sampled perturbations. The results in \cref{fig:robustness} show that the monitoring performance degrades gradually as the perturbation level increases, but the proposed method maintains competitive monitoring performance under moderate perturbations. This suggests that the learned policy does not simply overfit to the exact deterministic traversal times in the training graph. Instead, it captures monitoring patterns that remain effective under mild model mismatch and stochastic traversal-time variations. Such robustness is desirable in practical deployments, where robot travel times may vary due to speed fluctuations, localization errors, or uncertain traversal conditions.

\begin{figure*}[htbp]
		\centering
        \subfloat[Chain in \cref{fig:chain} with $4$ robots]{
			\includegraphics[width=0.31\textwidth]{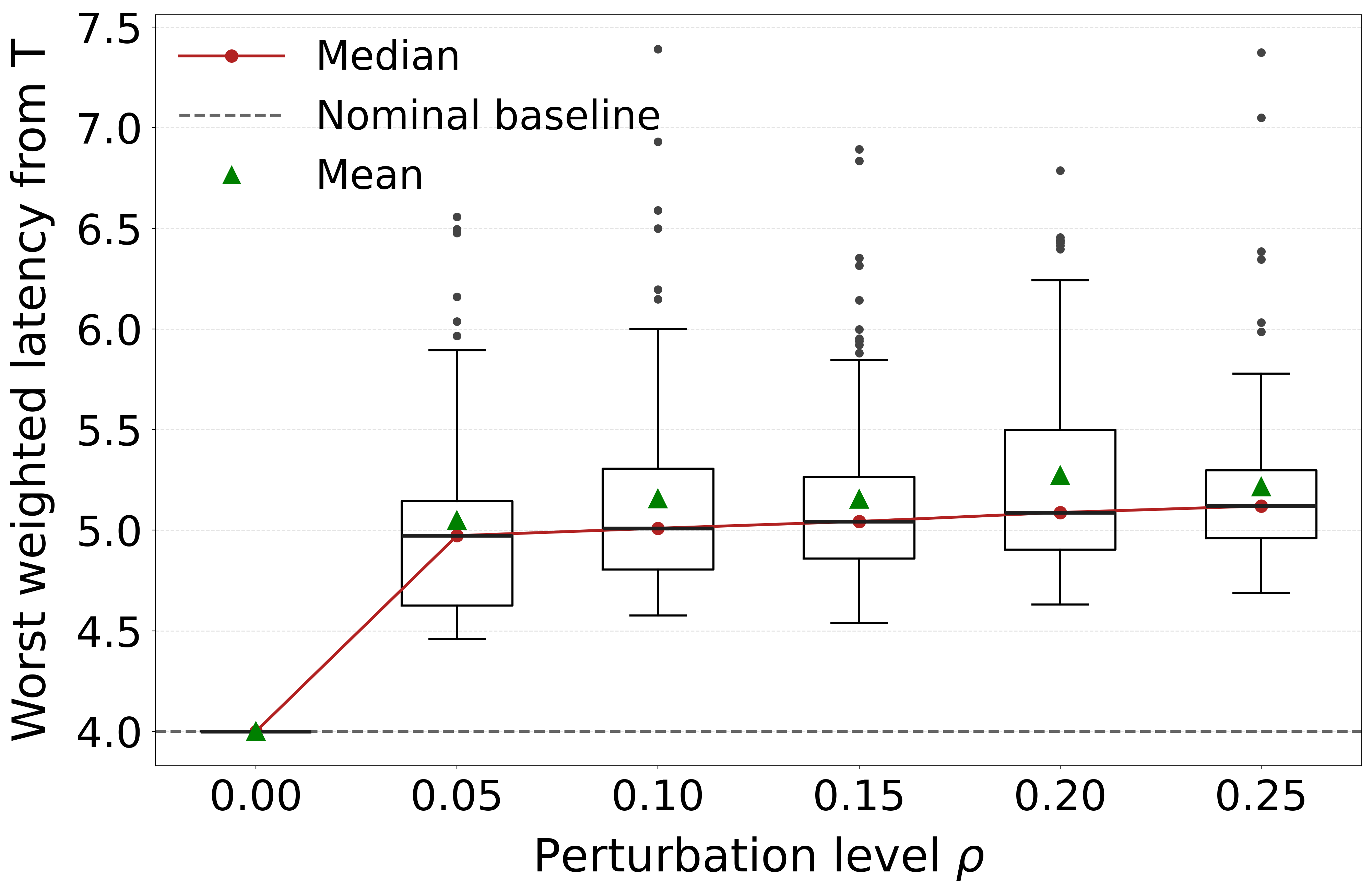}
			\label{fig:chain10_robustness}
		}
		\subfloat[MapA in \cref{fig:mapA}  with $6$ robots]{
			\includegraphics[width=0.31\textwidth]{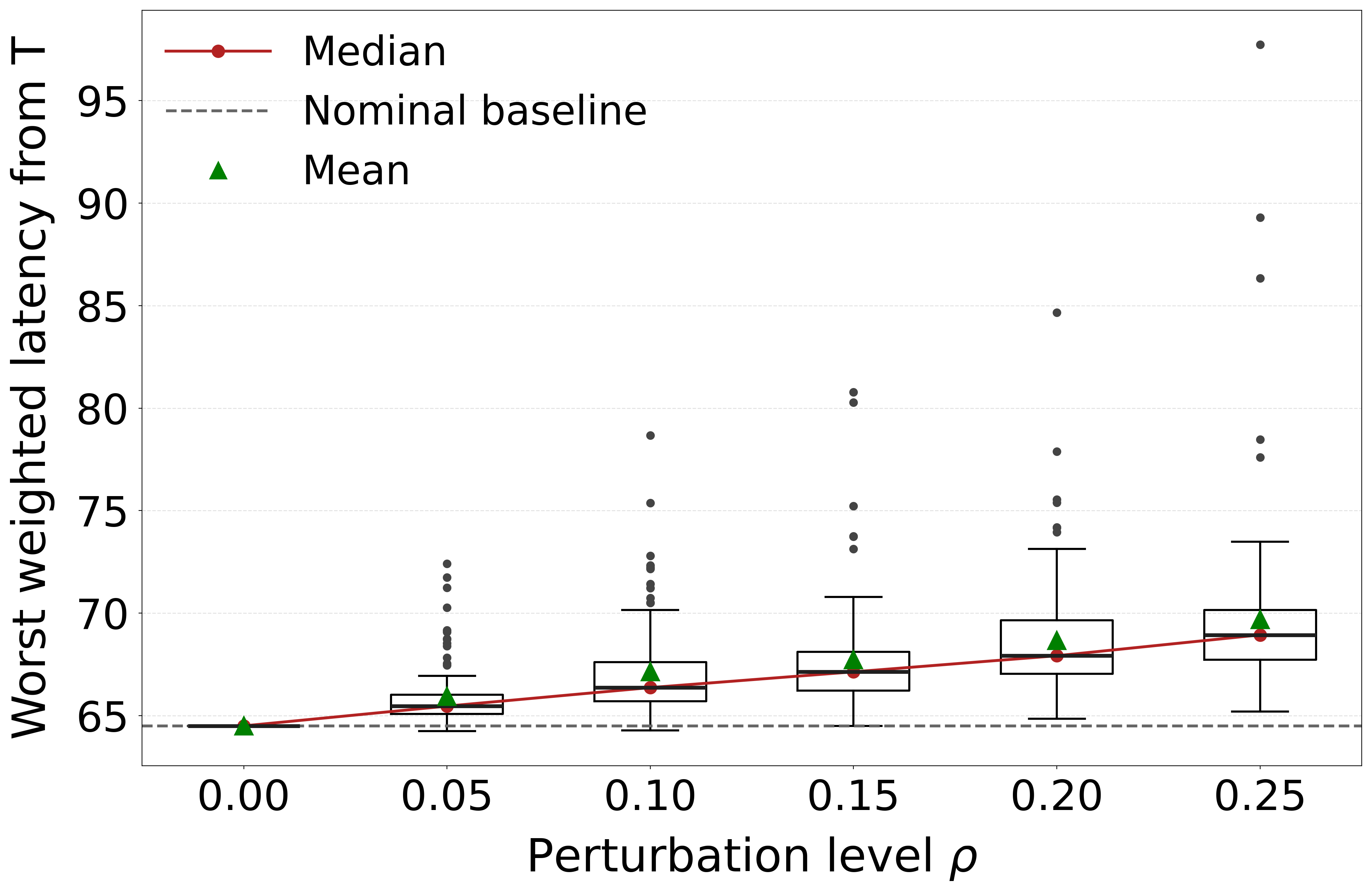}
			\label{fig:mapA_robustness}
		}
        \subfloat[San Francisco map in \cref{fig:SF}  with $3$ robots]{
			\includegraphics[width=0.31\textwidth]{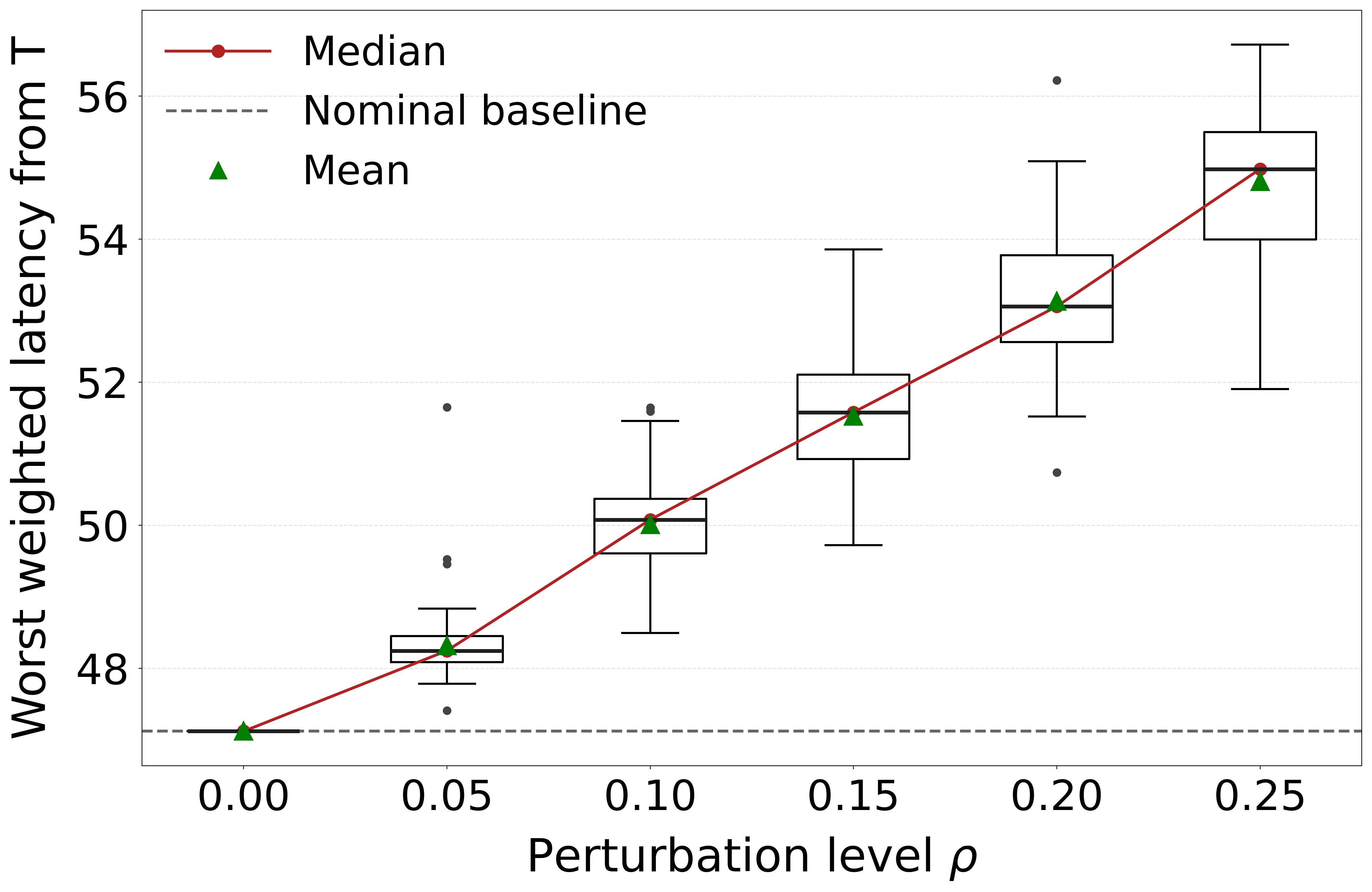}
			\label{fig:SF_robustness}
		}
		\caption{Zero-shot robustness to stochastic edge traversal times. For each perturbation level $\rho$, we run $100$ evaluation episodes with independently sampled edge-length perturbations. Box plots show the distribution of WI over these episodes: each box spans the interquartile range from the first to the third quartile, the center line denotes the median, whiskers indicate the non-outlier range within $1.5$ times the interquartile range, and individual dots denote outliers. }
		\label{fig:robustness}
	\end{figure*}

	\subsection{Ablation studies}
  In this subsection, we conduct ablation studies to evaluate the contribution of the main algorithmic components introduced in \cref{sub:algorithm}. Starting from the full TWLO-MDP-based learning method, we selectively remove one component at a time while keeping all other training and evaluation settings unchanged. The considered variants remove imitation-based initialization, customized observation preprocessing, and Trans-GAE, respectively. The learning curves are shown in \cref{fig:ablation}, and the final monitoring performance is summarized in \cref{tab:ablation}.
\begin{figure}[htbp]
		\centering
		\subfloat[Marostica roadmap in \cref{fig:ity} with $6$ robots.]{
			\includegraphics[width=0.33\textwidth]{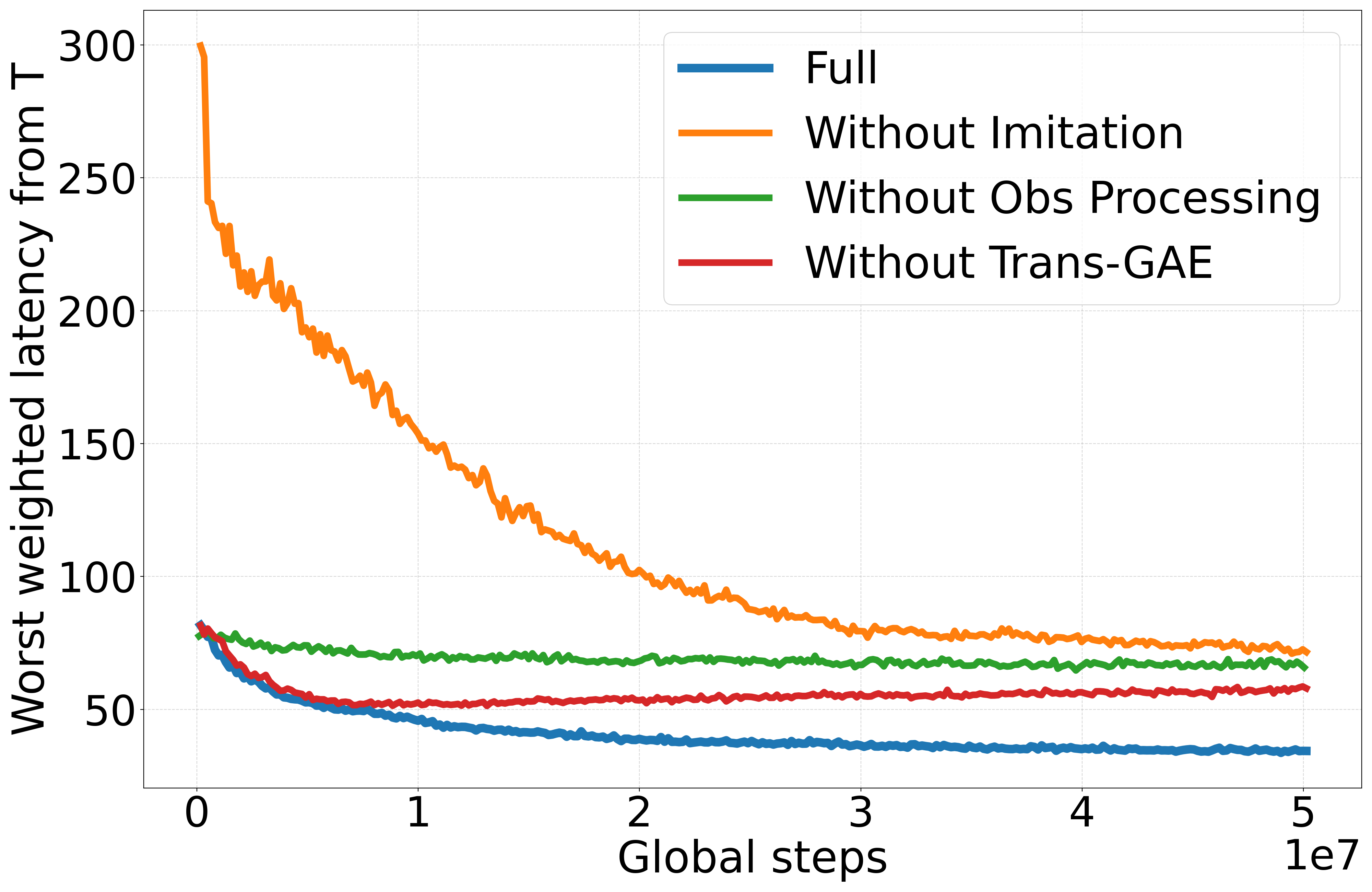}
			\label{fig:ablation_marostica}
		}
        
        \subfloat[Grid in \cref{fig:grid} with $6$ robots]{
			\includegraphics[width=0.33\textwidth]{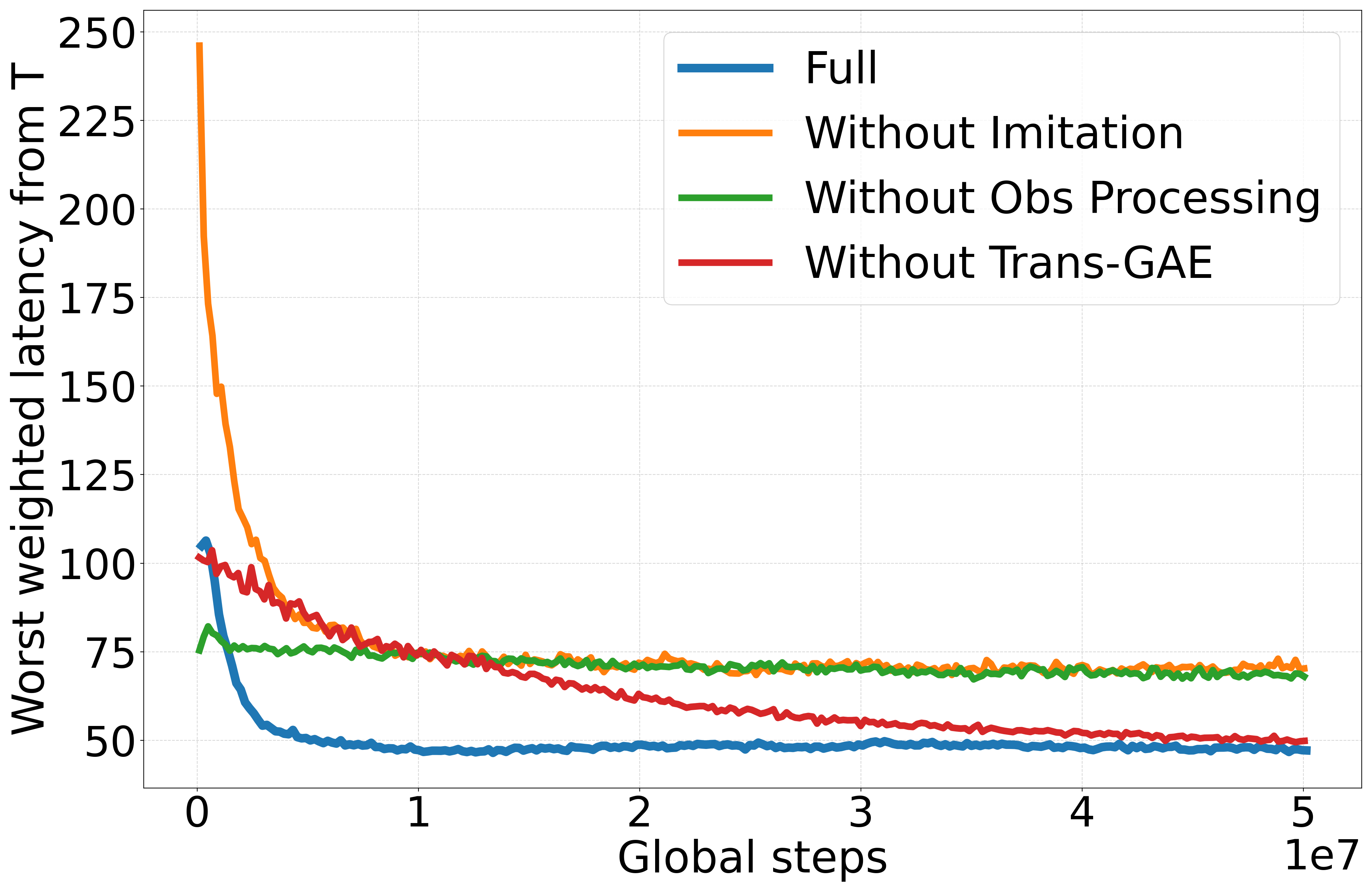}
			\label{fig:ablation_grid}
		}
		\caption{Ablation learning curves for the proposed TWLO-MDP-based method. Each variant removes one algorithmic component while keeping the remaining settings unchanged.}
		\label{fig:ablation}
	\end{figure}
\begin{table}[htbp]
    \centering
    \small
    \caption{Final monitoring performance of ablation variants. Each entry reports WI / AGI.}
    \label{tab:ablation}
    \begin{tabular}{lcc}
        \toprule
          & Marostica & Grid \\
        \midrule
        Full & \textbf{29 / 10.20} & \textbf{40 / 10.89} \\
        Without imitation learning &51 / 12.03 &51.5 / 12.12\\
        Without observation preprocessing & 66 / 16.37 & {63 / 14.06} \\
        Without Trans-GAE & 43 / 11.00  & 51 / 12.39  \\
        \bottomrule
    \end{tabular}
\end{table}
        The results show that the full method achieves the best overall performance in both environments, confirming that the proposed components jointly contribute to more effective and stable learning. Removing any of the three components leads to a degradation in final monitoring performance, although the degree of degradation varies across environments.

        Observation preprocessing has a particularly strong effect on final performance. Without the customized preprocessing module, the learned policies obtain substantially higher WI and AGI values in both Marostica and Grid. This indicates that properly encoding and normalizing heterogeneous state features is important for stable policy learning in graph-based monitoring tasks. Imitation-based initialization also provides a clear benefit. Removing imitation learning leads to slower convergence and worse final performance, especially in the Marostica roadmap. This suggests that heuristic demonstrations provide useful initial behavioral guidance and help policy optimization avoid poor regions of the search space in large monitoring graphs.


\section{Conclusion}\label{section:conclusion}

This paper studied multi-robot persistent monitoring on weighted graphs under a tail worst-case weighted-latency objective. We introduced a family of tail-performance objectives to capture long-run monitoring quality beyond the initial transient phase, and established several theoretical results, including optimal-solution existence, invariance of the optimal value, periodic approximation, and bounded  error  under discretized waiting
times. We then reformulated the original continuous-time problem as an equivalent event-driven Markov decision process, termed the TWLO-MDP. This reformulation converts the non-additive tail objective into a standard average-cost criterion while preserving optimality, thereby enabling reinforcement-learning-based solution methods. We further developed M2Bench, a unified platform for training, evaluating, and comparing heuristic and learning-based monitoring algorithms. Experiments on synthetic and realistic scenarios demonstrated the effectiveness of the proposed methods.

Future work will consider more complicated monitoring settings, including dynamic graphs, changing node priorities, and additional robot constraints, as well as improved scalability to larger teams and environments.
	
	\bibliographystyle{SageH}
	\bibliography{wwzbib.bib}

@article{WW-XD-JP:TCNS2025,
  author={Wang, Weizhen and Duan, Xiaoming and He, Jianping},
  journal={IEEE Transactions on Control of Network Systems}, 
  title={Multirobot Stochastic Patrolling via Graph Partitioning}, 
  year={2025},
  volume={12},
  number={1},
  pages={300--312}}

@article{WW-XD-JH:TAC2025,
  author={Wang, Weizhen and Duan, Xiaoming and He, Jianping},
  journal={IEEE Transactions on Automatic Control}, 
  title={Growing Well-Connected Networks for Entropy Rate Maximization}, 
  year={2025},
  volume={70},
  number={11},
  pages={7636--7643}}

@article{XD-WW-RY:TAC2026,
  title={Maximizing {Markov} Trajectory Entropy under {Kemeny} Constraints for Robotic Surveillance},
  author={Duan, Xiaoming and Wang, Weizhen and Yan, Rui},
  journal={IEEE Transactions on Automatic Control},
  year={2026},
  publisher={IEEE}
}

@article{WW-JH-XD:LCSS2024,
  author={Wang, Weizhen and He, Jianping and Duan, Xiaoming},
  journal={IEEE Control Systems Letters}, 
  title={On the Trade-Off Between Efficiency and Unpredictability in Stochastic Robotic Surveillance}, 
  year={2024},
  volume={8},
  pages={2829--2834}}

@ARTICLE{GD_FB_JM:2023,
	author={Díaz-García, Gilberto and Bullo, Francesco and Marden, Jason R.},
	title={Distributed {M}arkov Chain-Based Strategies for Multi-Agent Robotic Surveillance}, 
	journal={IEEE Control Systems Letters}, 
	year={2023},
	volume={7},
	number={},
	pages={2527--2532}}

@incollection{RK:2010,
	author={Karp, Richard M.},
	title={Reducibility Among Combinatorial Problems},
	editor="J{\"u}nger, Michael
	and Liebling, Thomas M.
	and Naddef, Denis
	and Nemhauser, George L.
	and Pulleyblank, William R.
	and Reinelt, Gerhard
	and Rinaldi, Giovanni
	and Wolsey, Laurence A.",
	year={2010},
	pages={219-241},
	chapter ={8},
	publisher={Springer Berlin Heidelberg},
	booktitle={50 Years of Integer Programming 1958-2008: From the Early Years to the State-of-the-Art}
}

@ARTICLE{SW-JC-JM-JC-FM-PL-MV:2017,
	author={Witwicki, Stefan and Castillo, Jose Carlos and Messias, Joao and Capitan, Jesus and Melo, Francisco S. and Lima, Pedro U. and Veloso, Manuela},
	journal={IEEE Robotics \& Automation Magazine}, 
	title={Autonomous Surveillance Robots: A Decision-Making Framework for Networked Multiagent Systems}, 
	year={2017},
	volume={24},
	number={3},
	pages={52-64}}

@InProceedings{AM-GR-JZ-AD:2003,
	author="Machado, Aydano
	and Ramalho, Geber
	and Zucker, Jean--Daniel
	and Drogoul, Alexis",
	title="Multi-agent Patrolling: {A}n Empirical Analysis of Alternative Architectures",
	booktitle="Multi-Agent-Based Simulation II",
	year="2003",
	publisher="Springer Berlin Heidelberg",
	pages="155--170"
	
}

@article{RP-AC-FB:2016,
	author = {Patel, Rushabh and Carron, Andrea and Bullo, Francesco},
	title = {The Hitting Time of Multiple Random Walks},
	journal = {SIAM Journal on Matrix Analysis and Applications},
	volume = {37},
	number = {3},
	pages = {933-954},
	year = {2016},
}

@INPROCEEDINGS{YC:2004,
	author={Chevaleyre, Y.},
	booktitle={ IEEE/WIC/ACM International Conference on Intelligent Agent Technology}, 
	title={Theoretical analysis of the multi-agent patrolling problem}, 
	year={2004},
	pages={302-308}}

@article{RS-MS-SS-BJ-DR-GS:2011,
	author = {Smith, Ryan N. and Schwager, Mac and Smith, Stephen L. and Jones, Burton H. and Rus, Daniela and Sukhatme, Gaurav S.},
	title = {Persistent ocean monitoring with underwater gliders: Adapting sampling resolution},
	journal = {Journal of Field Robotics},
	volume = {28},
	number = {5},
	pages = {714-741},
	year = {2011}
}

@article{SA-EF-SS:2012,
	title={Persistent monitoring in discrete environments: Minimizing the maximum weighted latency between observations},
	author={Alamdari, Soroush and Fata, Elaheh and Smith, Stephen L},
	journal={The International Journal of Robotics Research},
	year={2014},
	volume={33},
	pages={138-154}
	
}

@inproceedings{PA-MB-KB:2022,
	author = {Afshani, Peyman and de Berg, Mark and Buchin, Kevin and Gao, Jie and Maarten L\"{o}ffler and Nayyeri, Amir and Raichel, Benjamin and Sarkar, Rik and Wang, Haotian and Yang, Hao-Tsung}, 
	title = {On Cyclic Solutions to the Min-Max Latency Multi-Robot Patrolling Problem}, 
	volume = {224},
	pages =	{2:1--2:14},
	booktitle = {International Symposium on Computational Geometry},
	year = {2022} }

@article{FP-JD-FB:2012,
	author = {Pasqualetti, Fabio and Durham, J.W. and Bullo, Francesco},
	year = {2012},
	pages = {1181-1188},
	title = {Cooperative Patrolling via Weighted Tours: Performance Analysis and Distributed Algorithms},
	volume = {28},
	journal = {IEEE Transactions on Robotics}
}

@inproceedings{DP-RR:2010,
	author = {Portugal, David and Rocha, Rui},
	title = {{MSP} Algorithm: Multi-Robot Patrolling Based on Territory Allocation Using Balanced Graph Partitioning},
	year = {2010},
	booktitle = {ACM Symposium on Applied Computing},
	pages = {1271–1276},
	
}

@InProceedings{HS-GR-VC-BR:2004,
	author =	 {Hugo Santana and Geber Ramalho and Vincent Corruble and Bohdana Ratitch},
	title =	 {Multi-Agent Patrolling with Reinforcement Learning},
	booktitle =	 {International Joint
	Conference on Autonomous Agents and MultiAgent
	Systems},
	pages =	 {1122--1129},
	year =	 {2004}
}

@article{MJ-LV-AS:2022,
	author = {Jana, Meghdeep and Vachhani, Leena and Sinha, Arpita},
	year = {2022},
	pages = {724--745},
	title = {A deep reinforcement learning approach for multi-agent mobile robot patrolling},
	volume = {6},
	journal = {International Journal of Intelligent Robotics and Applications}
}

@article{XD-FB:2021,
	author = {Duan, Xiaoming and Bullo, Francesco},
	title = {{M}arkov Chain–Based Stochastic Strategies for Robotic Surveillance},
	journal = {Annual Review of Control, Robotics, and Autonomous Systems},
	volume = {4},
	number = {1},
	pages = {243-264},
	year = {2021},
}

@INPROCEEDINGS{AA-SS:2018,
	author={Asghar, Ahmad Bilal and Smith, Stephen L.},
	booktitle={IEEE Conference on Decision and Control}, 
	title={A Patrolling Game for Adversaries with Limited Observation Time}, 
	year={2018},
	volume={},
	number={},
	pages={3305-3310}}

@ARTICLE{XD-DP-FB:2021,
	author={Duan, Xiaoming and Paccagnan, Dario and Bullo, Francesco},
	journal={IEEE Transactions on Control of Network Systems}, 
	title={Stochastic Strategies for Robotic Surveillance as {S}tackelberg Games}, 
	year={2021},
	volume={8},
	number={2},
	pages={769-780}}

@ARTICLE{MG-SJ-FB:2019,
	author={George, Mishel and Jafarpour, Saber and Bullo, Francesco},
	journal={IEEE Transactions on Automatic Control}, 
	title={{M}arkov Chains With Maximum Entropy for Robotic Surveillance}, 
	year={2019},
	volume={64},
	number={4},
	pages={1566-1580}}

@ARTICLE{XD-MG-FB:2020,
	author={Duan, Xiaoming and George, Mishel and Bullo, Francesco},
	journal={IEEE Transactions on Automatic Control}, 
	title={Markov Chains With Maximum Return Time Entropy for Robotic Surveillance}, 
	year={2020},
	volume={65},
	number={1},
	pages={72-86}}

@INPROCEEDINGS{LG-HP-XD-JH:2023,
	author={Guo, Lingxiao and Pan, Haoxuan and Duan, Xiaoming and He, Jianping},
	booktitle={IEEE International Conference on Robotics and Automation}, 
	title={Balancing Efficiency and Unpredictability in Multi-robot Patrolling: A {MARL}-Based Approach}, 
	year={2023},
	volume={},
	number={},
	pages={3504-3509}}

@ARTICLE{YJ-GD-XD-JM-FB:2025,
  author={John, Yohan and Díaz-García, Gilberto and Duan, Xiaoming and Marden, Jason R. and Bullo, Francesco},
  journal={IEEE Transactions on Automatic Control}, 
  title={A Stochastic Surveillance {Stackelberg} Game: Co-Optimizing Defense Placement and Patrol Strategy}, 
  year={2025},
  volume={70},
  number={8},
  pages={5468-5474}}

@inproceedings{VK-MS-DK:2022,
	title={Drone Surveillance in Flood Affected Areas using Firefly Algorithm},
	author={Kumar, V Suresh and Sakthivel, M and Karras, Dimitrios A and Gupta, Shashi Kant and Gangadharan, Syam Machinathu Parambil and Haralayya, Bhadrappa},
	booktitle={International Conference on Knowledge Engineering and Communication Systems},
	pages={1--5},
	year={2022}
}

@ARTICLE{AA-SS-SS:2025,
  author={Asghar, Ahmad Bilal and Sundaram, Shreyas and Smith, Stephen L.},
  journal={IEEE Transactions on Robotics}, 
  title={Multirobot Persistent Monitoring: Minimizing Latency and Number of Robots With Recharging Constraints}, 
  year={2025},
  volume={41},
  number={},
  pages={236-252},
 }

@ARTICLE{FP-AF-FB:2012,
	author={Pasqualetti, Fabio and Franchi, Antonio and Bullo, Francesco},
	journal={IEEE Transactions on Robotics}, 
	title={On Cooperative Patrolling: Optimal Trajectories, Complexity Analysis, and Approximation Algorithms}, 
	year={2012},
	volume={28},
	number={3},
	pages={592-606}}

@ARTICLE{SS-MS-DR:2012,
	author={Smith, Stephen L. and Schwager, Mac and Rus, Daniela},
	journal={IEEE Transactions on Robotics}, 
	title={Persistent Robotic Tasks: Monitoring and Sweeping in Changing Environments}, 
	year={2012},
	volume={28},
	number={2},
	pages={410-426}}

@ARTICLE{CC-XL-XD:2013,
	author={Cassandras, Christos. G. and Lin, Xuchao and Ding, Xuchu},
	journal={IEEE Transactions on Automatic Control}, 
	title={An Optimal Control Approach to the Multi-Agent Persistent Monitoring Problem}, 
	year={2013},
	volume={58},
	number={4},
	pages={947-961}}

@ARTICLE{NZ-XY-SA-CC:2018,
	author={Zhou, Nan and Yu, Xi and Andersson, Sean B. and Cassandras, Christos G.},
	journal={IEEE Transactions on Automatic Control}, 
	title={Optimal Event-Driven Multiagent Persistent Monitoring of a Finite Set of Data Sources}, 
	year={2018},
	volume={63},
	number={12},
	pages={4204-4217}}

@INPROCEEDINGS{AA-SS-SS:2019,
	author={Asghar, Ahmad Bilal and Smith, Stephen L. and Sundaram, Shreyas},
	booktitle={American Control Conference}, 
	title={Multi-Robot Routing for Persistent Monitoring with Latency Constraints}, 
	year={2019},
	volume={},
	number={},
	pages={2620-2625}}

@INPROCEEDINGS{JG-JB:2005,
	author={Grace, J. and Baillieul, J.},
	booktitle={IEEE Conference on Decision and Control}, 
	title={Stochastic Strategies for Autonomous Robotic Surveillance}, 
	year={2005},
	volume={},
	number={},
	pages={2200-2205}}

@misc{JS-FW-PD:2017,
	title={Proximal Policy Optimization Algorithms}, 
	author={John Schulman and Filip Wolski and Prafulla Dhariwal and Alec Radford and Oleg Klimov},
	year={2017},
	eprint={1707.06347},
	archivePrefix={arXiv},
	primaryClass={cs.LG}
}

@InProceedings{VM-AB-KK:2016,
	title = 	 {Asynchronous Methods for Deep Reinforcement Learning},
	author = 	 {Mnih, Volodymyr and Badia, Adria Puigdomenech and Mirza, Mehdi and Graves, Alex and Lillicrap, Timothy and Harley, Tim and Silver, David and Kavukcuoglu, Koray},
	booktitle = 	 {International Conference on Machine Learning},
	pages = 	 {1928--1937},
	year = 	 {2016},
	editor = 	 {Balcan, Maria Florina and Weinberger, Kilian Q.},
	address = 	 {New York, New York, USA},
	publisher =    {PMLR}
}

@InProceedings{TH-AZ-SL:2018,
	title = 	 {Soft Actor-Critic: Off-Policy Maximum Entropy Deep Reinforcement Learning with a Stochastic Actor},
	author =       {Haarnoja, Tuomas and Zhou, Aurick and Abbeel, Pieter and Levine, Sergey},
	booktitle = 	 {Proceedings of the 35th International Conference on Machine Learning},
	pages = 	 {1861--1870},
	year = 	 {2018},
	editor = 	 {Dy, Jennifer and Krause, Andreas},
	volume = 	 {80},
	publisher =    {PMLR},
}

@inproceedings{JS-PM-PA:2016,
	author       = {John Schulman and
	Philipp Moritz and
	Sergey Levine and
	Michael I. Jordan and
	Pieter Abbeel},
	editor       = {Yoshua Bengio and
	Yann LeCun},
	title        = {High-Dimensional Continuous Control Using Generalized Advantage Estimation},
	booktitle = {Proceedings of the International Conference on Learning Representations},
	year         = {2016}
}

@inproceedings{PA-BD-HY:2020,
	title={Approximation algorithms for multi-robot patrol-scheduling with min-max latency},
	author={Afshani, Peyman and De Berg, Mark and Buchin, Kevin and Gao, Jie and L{\"o}ffler, Maarten and Nayyeri, Amir and Raichel, Benjamin and Sarkar, Rik and Wang, Haotian and Yang, Hao-Tsung},
	booktitle={International Workshop on the Algorithmic Foundations of Robotics},
	pages={107--123},
	year={2020}
}

@book{WR:1976,
	title        = {Principles of Mathematical Analysis},
	author       = {Rudin, Walter},
	year         = {1976},
	publisher    = {McGraw-Hill},
	edition      = {3rd},
}

@article{JK:2024,
	author       = {Kennedy, James},
	title        = {Geometric spectral theory of quantum graphs},
	journal      = {Communications in Mathematics},
	year         = {2024},
	volume       = {32},
	pages = {393-439}
}

@book{SA:2015,
	title={Understanding analysis},
	author={Abbott, Stephen},
	year={2015},
	publisher={Springer}
}

@article{DM-AS-LV:2025,
	title={$\epsilon$-Optimal Multi-Agent Patrol using Recurrent Strategy},
	author={Mallya, Deepak and Sinha, Arpita and Vachhani, Leena},
	year={2025},
	journal= {arXiv preprint arXiv:2509.11640},
}

@book{RE:1989,
	author    = {Ryszard Engelking},
	title     = {General Topology},
	series    = {Sigma Series in Pure Mathematics},
	volume    = {6},
	publisher = {Heldermann Verlag},
	edition   = {2nd},
	year      = {1989}
}

@book{SW:1970,
	title     = {General Topology},
	author    = {Willard, Stephen},
	year      = {1970},
	publisher = {Addison-Wesley Publishing Company}
}

@book{JM:2000,
	title     = {Topology},
	author    = {Munkres, James R.},
	edition   = {2nd},
	year      = {2000},
	publisher = {Prentice Hall},
}

@book{HHS:14,
  author    = {Sohrab, Houshang H.},
  title     = {Basic Real Analysis},
  publisher = {Springer New York},
  year      = {2014},
  edition   = {2nd},
}

@book{MRB-AH:99,
  title     = {Metric Spaces of Non-Positive Curvature},
  author    = {Bridson, Martin R. and Haefliger, Andr{\'e}},
  series    = {Grundlehren der mathematischen Wissenschaften},
  volume    = {319},
  year      = {1999},
  publisher = {Springer-Verlag},
  address   = {Berlin, Heidelberg},
  isbn      = {978-3-540-64324-1},
}

@book{CA-KB:2006,
	title     = {Infinite Dimensional Analysis: A Hitchhiker's Guide},
	author    = {Aliprantis, Charalambos D. and Border, Kim C.},
	edition   = {3rd},
	year      = {2006},
	publisher = {Springer}
}

@article{JP-EG-MB:2025,
  title={Cooperative patrol routing: Optimizing urban crime surveillance through multi-agent reinforcement learning},
  author={Palma-Borda, Juan and Guzm{\'a}n, Eduardo and Belmonte, Mar{\'\i}a-Victoria},
  journal={Engineering Applications of Artificial Intelligence},
  volume={166},
  pages={113706},
  year={2026},
  publisher={Elsevier}
}

@inproceedings{Gymnasium:2024,
  title         = {Gymnasium: A Standardized Interface for Reinforcement Learning Environments},
  author        = {Towers, Mark and Kwiatkowski, Ariel and Balis, John U. and De Cola, Gianluca and Deleu, Tristan and Goul{\~a}o, Manuel and Kallinteris, Andreas and Krimmel, Markus and KG, Arjun and Perez-Vicente, Rodrigo and Terry, Jordan and Pierr{\'e}, Andrea and Schulhoff, Sander and Tai, Jun Jet and Tan, Hannah and Younis, Omar G.},
  booktitle     = {Advances in Neural Information Processing Systems},
  volume        = {38},
  year          = {2025},
}

@article{MM-HS-BA:2022,
title = {Coordinated routing system for fire detection by patrolling trucks with drones},
journal = {International Journal of Disaster Risk Reduction},
volume = {73},
pages = {102859},
year = {2022},
issn = {2212-4209},
author = {Maryam Momeni and Hamed Soleimani and Shahrooz Shahparvari and Behrouz Afshar-Nadjafi},
}

@article{NC:2022,
  title={Worst-case analysis of a new heuristic for the travelling salesman problem},
  author={Christofides, Nicos},
  journal={Operations Research Forum},
  volume={3},
  number={20},
  pages={1--4},
  year={2022},
  organization={Springer}
}

@inproceedings{AK-Nk-OS:2021,
  title={A (slightly) improved approximation algorithm for metric {TSP}},
  author={Karlin, Anna R and Klein, Nathan and Gharan, Shayan Oveis},
  booktitle={Proceedings of the 53rd Annual ACM SIGACT Symposium on Theory of Computing},
  pages={32--45},
  year={2021}
}

@article{MK-ML-RS:2015,
title = {New inapproximability bounds for {TSP}},
journal = {Journal of Computer and System Sciences},
volume = {81},
number = {8},
pages = {1665-1677},
year = {2015},
issn = {0022-0000},
author = {Marek Karpinski and Michael Lampis and Richard Schmied},
}

@article{NB:2022,
  title={Recent trends in robotic patrolling},
  author={Basilico, Nicola},
  journal={Current Robotics Reports},
  volume={3},
  number={2},
  pages={65--76},
  year={2022},
  publisher={Springer}
}

@InProceedings{DP-RR:2011,
author="Portugal, David
and Rocha, Rui",
editor="Camarinha-Matos, Luis M.",
title="A Survey on Multi-robot Patrolling Algorithms",
booktitle="Technological Innovation for Sustainability",
year="2011",
publisher="Springer Berlin Heidelberg",
address="Berlin, Heidelberg",
pages="139--146",

}

@article{LH-MZ-EH:2019,
  title={A survey of multi-robot regular and adversarial patrolling},
  author={Huang, Li and Zhou, MengChu and Hao, Kuangrong and Hou, Edwin},
  journal={IEEE/CAA Journal of Automatica Sinica},
  volume={6},
  number={4},
  pages={894--903},
  year={2019},
  publisher={IEEE}
}

@article{SH-SR-DC:2020,
  title={Optimal {UAV} route planning for persistent monitoring missions},
  author={Hari, Sai Krishna Kanth and Rathinam, Sivakumar and Darbha, Swaroop and Kalyanam, Krishna and Manyam, Satyanarayana Gupta and Casbeer, David},
  journal={IEEE Transactions on Robotics},
  volume={37},
  number={2},
  pages={550--566},
  year={2021},
  publisher={IEEE}
}

@article{DM-AS-LV:2022,
  title={Priority Patrol With a Single Agent—Bounds and Approximations},
  author={Mallya, Deepak and Sinha, Arpita and Vachhani, Leena},
  journal={IEEE Control Systems Letters},
  volume={7},
  pages={1321--1326},
  year={2023},
  publisher={IEEE}
}

@article{DP-RR:2013,
title = {Distributed multi-robot patrol: A scalable and fault-tolerant framework},
journal = {Robotics and Autonomous Systems},
volume = {61},
number = {12},
pages = {1572-1587},
year = {2013},
issn = {0921-8890},
author = {David Portugal and Rui P. Rocha},
}

@article{DP-RR:2016,
  title={Cooperative multi-robot patrol with {Bayesian} learning},
  author={Portugal, David and Rocha, Rui P},
  journal={Autonomous Robots},
  volume={40},
  number={5},
  pages={929--953},
  year={2016},
  publisher={Springer}
}

@inproceedings{NA-DU-PS:2011,
  title={Multiagent patrol generalized to complex environmental conditions},
  author={Agmon, Noa and Urieli, Daniel and Stone, Peter},
  booktitle={Proceedings of the AAAI Conference on Artificial Intelligence},
  volume={25},
  number={1},
  pages={1090--1095},
  year={2011}
}

@inproceedings{DP-RR:2013:iros,
  title={Scalable, fault-tolerant and distributed multi-robot patrol in real world environments},
  author={Portugal, David and Rocha, Rui P},
  booktitle={2013 IEEE/RSJ International Conference on Intelligent Robots and Systems},
  pages={4759--4764},
  year={2013},
  organization={IEEE}
}

@article{AF-LI-DN:2017,
  title={Distributed on-line dynamic task assignment for multi-robot patrolling},
  author={Farinelli, Alessandro and Iocchi, Luca and Nardi, Daniele},
  journal={Autonomous Robots},
  volume={41},
  number={6},
  pages={1321--1345},
  year={2017},
  publisher={Springer}
}

@article{AG-XL-QZ:2024,
  title={Attrition-aware adaptation for multi-agent patrolling},
  author={Goeckner, Anthony and Li, Xinliang and Wei, Ermin and Zhu, Qi},
  journal={IEEE Robotics and Automation Letters},
  volume={9},
  number={8},
  pages={7230--7237},
  year={2024},
  publisher={IEEE}
}

@inproceedings{DM-GR-PT:2009,
title = "SimPatrol: Towards the establishment of multi-agent patrolling as a benchmark for multi-agent systems",
author = "Daniel Moreira and Geber Ramalho and Patr{\'i}cia Tedesco",
year = "2009",
pages = "570--575",
booktitle = "ICAART 2009 - Proceedings of the 1st International Conference on Agents and Artificial Intelligence",
}

@incollection{DP-LI-AF:2018,
  title={A {ROS}-based framework for simulation and benchmarking of multi-robot patrolling algorithms},
  author={Portugal, David and Iocchi, Luca and Farinelli, Alessandro},
  booktitle={Robot Operating System (ROS) The Complete Reference},
  pages={3--28},
  year={2018},
  publisher={Springer}
}

@inproceedings{YE-AS-GK:2008,
  title={A realistic model of frequency-based multi-robot polyline patrolling},
  author={Elmaliach, Yehuda and Shiloni, Asaf and Kaminka, Gal A},
  booktitle={Proceedings of the 7th international joint conference on Autonomous agents and multiagent systems},
  pages={63--70},
  year={2008}
}

@inproceedings{JC-LG-EK:2011,
  title={Boundary patrolling by mobile agents with distinct maximal speeds},
  author={Czyzowicz, Jurek and  G{\k{a}}sieniec, Leszek and Kosowski, Adrian and Kranakis, Evangelos},
  booktitle={European Symposium on Algorithms},
  pages={701--712},
  year={2011},
  organization={Springer}
}

@article{AD-AG-CT:2014, 
title={On Fence Patrolling by Mobile Agents}, 
volume={21}, 
number={3}, 
journal={The Electronic Journal of Combinatorics}, author={Dumitrescu, Adrian and Ghosh, Anirban and Tóth, Csaba D.},
year={2014},
pages={P3.4} }

@article{AK-MS:2020,
title = {Simple strategies versus optimal schedules in multi-agent patrolling},
journal = {Theoretical Computer Science},
volume = {839},
pages = {195-206},
year = {2020},
author = {Akitoshi Kawamura and Makoto Soejima},
}

@article{VM-KK-DS:2015,
  title={Human-level control through deep reinforcement learning},
  author={Mnih, Volodymyr and Kavukcuoglu, Koray and Silver, David and Rusu, Andrei A. and Veness, Joel and Bellemare, Marc G. and Graves, Alex and Riedmiller, Martin and Fidjeland, Andreas K. and Ostrovski, Georg and Petersen, Stig and Beattie, Charles and Sadik, Amir and Antonoglou, Ioannis and King, Helen and Kumaran, Dharshan and Wierstra, Daan and Legg, Shane and Hassabis, Demis},
  journal={Nature},
  volume={518},
  number={7540},
  pages={529--533},
  year={2015},
  publisher={Nature Publishing Group}
}

@inproceedings{MH-JM-DS:2018,
  title={Rainbow: Combining improvements in deep reinforcement learning},
  author={Hessel, Matteo and Modayil, Joseph and Van Hasselt, H. V. and Schaul, Tom and Ostrovski, Georg and Dabney, Will and Horgan, Dan and Piot, Bilal and Azar, Mohammad and Silver, David},
  booktitle={Proceedings of the AAAI conference on artificial intelligence},
  volume={32},
  number={1},
  year={2018}
}

@article{TR-MS-SW:2020,
  title={Monotonic value function factorisation for deep multi-agent reinforcement learning},
  author={Rashid, Tabish and Samvelyan, Mikayel and De Witt, Christian Schroeder and Farquhar, Gregory and Foerster, Jakob and Whiteson, Shimon},
  journal={Journal of Machine Learning Research},
  volume={21},
  number={178},
  pages={1--51},
  year={2020}
}

@inproceedings{PS-GL-AG:2017,
  title={Value-decomposition networks for cooperative multi-agent learning},
  author={Sunehag, Peter and Lever, Guy and Gruslys, Audrunas and Czarnecki, Wojciech Marian and Zambaldi, Vinicius and Jaderberg, Max and Lanctot, Marc and Sonnerat, Nicolas and Leibo, Joel Z and Tuyls, Karl and Graepel, Thore},
  booktitle = {International Conference on Autonomous Agents and MultiAgent Systems},
  pages     = {2085--2087},
  year      = {2018},
}

@article{CD-TG-SW:2020,
  title={Is independent learning all you need in the {Starcraft} multi-agent challenge?},
  author={De Witt, Christian Schroeder and Gupta, Tarun and Makoviichuk, Denys and Makoviychuk, Viktor and Torr, Philip HS and Sun, Mingfei and Whiteson, Shimon},
  journal={arXiv preprint arXiv:2011.09533},
  year={2020}
}

@inproceedings{CY-AV-YW:2022,
  title={The surprising effectiveness of {PPO} in cooperative multi-agent games},
  author={Yu, Chao and Velu, Akash and Vinitsky, Eugene and Gao, Jiaxuan and Wang, Yu and Bayen, Alexandre and Wu, Yi},
  booktitle={Advances in Neural Information Processing Systems},
  volume={35},
  pages={24611--24624},
  year={2022}
}

@article{OR-IB-WC:2019,
  title={Grandmaster level in {StarCraft} {II} using multi-agent reinforcement learning},
  author={Vinyals, Oriol and Babuschkin, Igor and Czarnecki, Wojciech M. and Mathieu, Micha{\"e}l and Dudzik, Andrew and Chung, Junyoung and Choi, David H. and Powell, Richard and Ewalds, Timo and Georgiev, Petko and Oh, Junhyuk and Horgan, Dan and Kroiss, Manuel and Danihelka, Ivo and Huang, Aja and Sifre, Laurent and Cai, Trevor and Agapiou, John P. and Jaderberg, Max and Vezhnevets, Alexander S. and Leblond, R{\'e}mi and Pohlen, Tobias and Dalibard, Valentin and Budden, David and Sulsky, Yury and Molloy, James and Paine, Tom L. and Gulcehre, Caglar and Wang, Ziyu and Pfaff, Tobias and Wu, Yuhuai and Ring, Roman and Yogatama, Dani and W{\"u}nsch, Dario and McKinney, Katrina and Smith, Oliver and Schaul, Tom and Lillicrap, Timothy and Kavukcuoglu, Koray and Hassabis, Demis and Apps, Chris and Silver, David},
  journal={Nature},
  volume={575},
  number={7782},
  pages={350--354},
  year={2019},
  publisher={Nature Publishing Group UK London}
}

@InProceedings{LC-JL-RK:2024,
author="Chen, Li-Hsuan
and Hung, Ling-Ju
and Klasing, Ralf",
editor="Ghosh, Smita
and Zhang, Zhao",
title="Improved Approximation Algorithms for Patrol-Scheduling with Min-Max Latency Using Multiclass Minimum Spanning Forests",
booktitle="Algorithmic Aspects in Information and Management",
year="2024",
publisher="Springer Nature Singapore",
address="Singapore",
pages="99--110",
}

@article{LG-TJ-TR:2024,
title = {Perpetual maintenance of machines with different urgency requirements},
journal = {Journal of Computer and System Sciences},
volume = {139},
pages = {103476},
year = {2024},
issn = {0022-0000},
author = {Leszek G\k{a}sieniec and Tomasz Jurdzi\'{o}ski and Ralf Klasing and Christos Levcopoulos and Andrzej Lingas and Jie Min and Tomasz Radzik}
}

@inproceedings{FL-AK:2014,
  title={Robust multi-agent patrolling strategies using reinforcement learning},
  author={Lauri, Fabrice and Koukam, Abderrafiaa},
  booktitle={International Conference on Swarm Intelligence Based Optimization},
  pages={157--165},
  year={2014},
  organization={Springer}
}

@inproceedings{AG-YS-QZ:2024,
  title={Graph neural network-based multi-agent reinforcement learning for resilient distributed coordination of multi-robot systems},
  author={Goeckner, Anthony and Sui, Yueyuan and Martinet, Nicolas and Li, Xinliang and Zhu, Qi},
  booktitle={IEEE/RSJ International Conference on Intelligent Robots and Systems},
  pages={5732--5739},
  year={2024},
  organization={IEEE}
}

@inproceedings{JW-RM-EH:2025,
  title={Lightweight decentralized neural network-based strategies for multi-robot patrolling},
  author={Ward, James C and McConville, Ryan and Hunt, Edmund R},
  booktitle={Proceedings of the 40th ACM/SIGAPP Symposium on Applied Computing},
  pages={823--831},
  year={2025}
}

@inproceedings{ZH-DZ:2010,
  title={Reinforcement learning for multi-agent patrol policy},
  author={Hu, Zhaohui and Zhao, Dongbin},
  booktitle={IEEE International Conference on Cognitive Informatics},
  pages={530--535},
  year={2010},
  organization={IEEE}
}

@article{JP-DT-CS:2019,
  title={Equitable persistent coverage of non-convex environments with graph-based planning},
  author={Palacios-Gas{\'o}s, Jos{\'e} Manuel and Tardioli, Danilo and Montijano, Eduardo and Sag{\"u}{\'e}s, Carlos},
  journal={The International Journal of Robotics Research},
  volume={38},
  number={14},
  pages={1674--1694},
  year={2019},
  publisher={SAGE Publications Sage UK: London, England}
}

@inproceedings{DK-AK-VV:2021,
  title={Regstar: {Efficient} strategy synthesis for adversarial patrolling games},
  author={Kla{\v{s}}ka, David and Ku{\v{c}}era, Anton{\'\i}n and Musil, V{\'\i}t and {\v{R}}eh{\'a}k, Vojt{\v{e}}ch},
  booktitle={Uncertainty in Artificial Intelligence},
  pages={471--481},
  year={2021},
  organization={PMLR}
}

@inproceedings{AN-MG-TP:2024,
  title={General {Markov} model for solving patrolling games},
  author={Nag{\'o}rko, Andrzej and Godziszewski, Micha{\l} Tomasz and Waniek, Marcin and Rosiak, Barbara and R{\'o}g, Ma{\l}gorzata and Michalak, Tomasz Pawe{\l}},
  booktitle={The 40th Conference on Uncertainty in Artificial Intelligence},
  year={2024}
}

@article{SA-AM-KP:2011,
  title={Patrolling games},
  author={Alpern, Steve and Morton, Alec and Papadaki, Katerina},
  journal={Operations Research},
  volume={59},
  number={5},
  pages={1246--1257},
  year={2011},
  publisher={INFORMS}
}

@article{MD:KG:RS:2023,
  title={A stochastic game framework for patrolling a border},
  author={Darlington, Matthew and Glazebrook, Kevin D and Leslie, David S and Shone, Rob and Szechtman, Roberto},
  journal={European Journal of Operational Research},
  volume={311},
  number={3},
  pages={1146--1158},
  year={2023},
  publisher={Elsevier}
}

@inproceedings{JF-GF-SW:2018,
author = {Foerster, Jakob N. and Farquhar, Gregory and Afouras, Triantafyllos and Nardelli, Nantas and Whiteson, Shimon},
title = {Counterfactual multi-agent policy gradients},
year = {2018},
isbn = {978-1-57735-800-8},
publisher = {AAAI Press},
booktitle = {Proceedings of the Thirty-Second AAAI Conference on Artificial Intelligence and Thirtieth Innovative Applications of Artificial Intelligence Conference and Eighth AAAI Symposium on Educational Advances in Artificial Intelligence},
articleno = {363},
numpages = {9},
}

@article{TS-RC:2018,
  title={Learning {M}ontezuma's {R}evenge from a single demonstration},
  author={Salimans, Tim and Chen, Richard},
  journal={arXiv preprint arXiv:1812.03381},
  year={2018}
}

@inproceedings{TH-MV-OP:2018,
  title={Deep {Q}-learning from demonstrations},
  author={Hester, Todd and Vecerik, Matej and Pietquin, Olivier and Lanctot, Marc and Schaul, Tom and Piot, Bilal and Horgan, Dan and Quan, John and Sendonaris, Andrew and Osband, Ian and Dulac-Arnold, Gabriel and Agapiou, John and Leibo, Joel Z. and Gruslys, Audrunas},
  booktitle={Proceedings of the AAAI conference on artificial intelligence},
  volume={32},
  number={1},
  year={2018}
}

@INPROCEEDINGS{AR-VK-SL:2018,
    AUTHOR    = {Aravind Rajeswaran AND Vikash Kumar AND Abhishek Gupta AND
                 Giulia Vezzani AND John Schulman AND Emanuel Todorov AND Sergey Levine},
    TITLE     = {Learning complex dexterous manipulation with deep reinforcement learning and demonstrations},
    BOOKTITLE = {Proceedings of Robotics: Science and Systems},
    YEAR      = {2018},
}

@article{WC-WY:2023,
  title={Reinforcement learning with non-cumulative objective},
  author={Cui, Wei and Yu, Wei},
  journal={IEEE Transactions on Machine Learning in Communications and Networking},
  volume={1},
  pages={124--137},
  year={2023},
  publisher={IEEE}
}

@article{SG-YP-RN:2020,
  title={Maximum reward formulation in reinforcement learning},
  author={Gottipati, Sai Krishna and Pathak, Yashaswi and Nuttall, Rohan and Sahir and Chunduru, Raviteja and Touati, Ahmed and Subramanian, Sriram Ganapathi and Taylor, Matthew E. and Chandar, Sarath},
  journal={arXiv preprint arXiv:2010.03744},
  year={2020}
}

@inproceedings{RW-PZ-LY:2020,
  title={Planning with general objective functions: Going beyond total rewards},
  author={Wang, Ruosong and Zhong, Peilin and Du, Simon S and Salakhutdinov, Ruslan and Yang, Lin F},
  booktitle={Advances in Neural Information Processing Systems},
  volume={33},
  pages={14486--14497},
  year={2020}
}

@inproceedings{GV-WB-MD:2024,
  title={To the Max: Reinventing Reward in Reinforcement Learning},
  author={Veviurko, Grigorii and Boehmer, Wendelin and De Weerdt, Mathijs},
  booktitle={International Conference on Machine Learning},
  pages={49455--49470},
  year={2024},
  organization={PMLR}
}

@article{NG-SS-AA:2020,
  title={Revisiting parameter sharing in multi-agent deep reinforcement learning},
  author={Grammel, Nathaniel and Son, Sanghyun and Black, Benjamin and Agrawal, Aakriti},
  journal={arXiv preprint arXiv:2005.13625},
  year={2020}
}

@inproceedings{JF-IA-SW:2016,
 author={Foerster, Jakob and Assael, Ioannis Alexandros and De Freitas, Nando and Whiteson, Shimon},
  title={Learning to communicate with deep multi-agent reinforcement learning},
  booktitle={Advances in Neural Information Processing Systems},
  volume={29},
  pages= {2137--2145},
  year={2016}
}

@article{JY-JV-MS:2022,
  title={DiNNO: Distributed neural network optimization for multi-robot collaborative learning},
  author={Yu, Javier and Vincent, Joseph A and Schwager, Mac},
  journal={IEEE Robotics and Automation Letters},
  volume={7},
  number={2},
  pages={1896--1903},
  year={2022},
  publisher={IEEE}
}

@book{DW:1991,
  title={Probability with Martingales},
  author={Williams, David},
  year={1991},
  publisher={Cambridge University Press}
}

@inproceedings{MT:1993,
author = {Tan, Ming},
title = {Multi-agent reinforcement learning: {Independent} versus cooperative agents},
year = {1993},
publisher = {Morgan Kaufmann Publishers Inc.},
address = {San Francisco, CA, USA},
booktitle = {International Conference on Machine Learning},
pages = {330–337},
numpages = {8},
}

@book{RC-IN:2006,
  title={Topological Properties of Spaces of Continuous Functions},
  author={McCoy, Robert A and Ntantu, Ibula},
  year={2006},
  publisher={Springer}
}

@INPROCEEDINGS{YM-JL:2022,
  author={Ma, Yanhao and Luo, Jie},
  booktitle={ China Automation Congress }, 
  title={Value-Decomposition Multi-Agent Proximal Policy Optimization}, 
  year={2022},
  volume={},
  number={},
  pages={3460-3464}}

@article{DM:2021,
  title={What is actually a metric graph?},
  author={Delio Mugnolo},
  year={2021},
  journal={arXiv preprint arXiv:1912.07549},

}

@article{CY-TZ:2016,
  author  = {Yan, Chuanbo and Zhang, Tao},
  title   = {Multi-robot patrol: A distributed algorithm based on expected idleness},
  journal = {International Journal of Advanced Robotic Systems},
  volume  = {13},
  number  = {6},
  pages   = {1--12},
  year    = {2016},
}

@article{HC-TC-SW:2017,
title = {Developing an online cooperative police patrol routing strategy},
journal = {Computers, Environment and Urban Systems},
volume = {62},
pages = {19-29},
year = {2017},
author = {Huanfa Chen and Tao Cheng and Sarah Wise},
}

@article{RS-DP-SS:1999,
  title={Between {MDPs} and semi-{MDPs}: A framework for temporal abstraction in reinforcement learning},
  author={Sutton, Richard S and Precup, Doina and Singh, Satinder},
  journal={Artificial Intelligence},
  volume={112},
  number={1-2},
  pages={181--211},
  year={1999},
  publisher={Elsevier}
}

@article{YX-WT-JH:2025,
  title={Asynchronous multi-agent deep reinforcement learning under partial observability},
  author={Xiao, Yuchen and Tan, Weihao and Hoffman, Joshua and Xia, Tian and Amato, Christopher},
  journal={The International Journal of Robotics Research},
  volume={44},
  number={8},
  pages={1257--1286},
  year={2025},
  publisher={Sage Publications Sage UK: London, England}
}

@inproceedings{Pytorch:2019,
  title={{PyTorch}: {An} imperative style, high-performance deep learning library},
  author={Paszke, Adam and Gross, Sam and Massa, Francisco and Lerer, Adam and Bradbury, James and Chanan, Gregory and Killeen, Trevor and Lin, Zeming and Gimelshein, Natalia and Antiga, Luca and Desmaison, Alban and K{\"o}pf, Andreas and Yang, Edward and DeVito, Zachary and Raison, Martin and Tejani, Alykhan and Chilamkurthy, Sasank and Steiner, Benoit and Fang, Lu and Bai, Junjie and Chintala, Soumith},
  booktitle={Advances in Neural Information Processing Systems},
  volume={32},
  year={2019},
  pages= {8026--8037},
}

@inproceedings{Pettingzoo:2021,
  title={{PettingZoo}: {A} standard {API} for multi-agent reinforcement learning},
  author={Terry, J. K. and Black, Benjamin and Grammel, Nathaniel and Jayakumar, Mario and Hari, Ananth and Sullivan, Ryan and Santos, Luis and Perez, Rodrigo and Horsch, Caroline and Dieffendahl, Clemens and Williams, Niall L. and Lokesh, Yashas and Ravi, Praveen},
  booktitle={Advances in Neural Information Processing Systems},
  volume={34},
  pages={15032-15043},
  year={2021}
}

@inproceedings{MH-PS:2015,
  title={Deep recurrent {Q}-learning for partially observable {MDPs}},
  author={Hausknecht, Matthew and Stone, Peter},
  booktitle={AAAI Fall Symposium Series},
  year={2015}
}

@inproceedings{SK-GO-WD:2019,
  title={Recurrent experience replay in distributed reinforcement learning},
  author={Kapturowski, Steven and Ostrovski, Georg and Quan, John and Munos, Remi and Dabney, Will},
  booktitle={International Conference on Learning Representations},
  year={2019}
}

@article{JP-RW:1994,
  title={Incremental multi-step {Q}-learning},
  author={Jing Peng and Ronald J. Williams},
  journal={Machine Learning},
  year={1994},
  volume={22},
  pages={283-290},
}

@article{Tianshou:2021,
  title={Tianshou: {A} Highly Modularized Deep Reinforcement Learning Library},
  author={Jiayi Weng and Huayu Chen and Dong Yan and Kaichao You and Alexis Duburcq and Minghao Zhang and Yi Su and Hang Su and Jun Zhu},
  journal={Journal of Machine Learning Research},
  year={2022},
  volume={23},
  number = {267},
  pages={1-6},
}

@inproceedings{HH-AG-MH-DS:2016,
  title={Learning values across many orders of magnitude},
  author={H. V. Hasselt and Arthur Guez and Matteo Hessel and Volodymyr Mnih and David Silver},
  booktitle={Neural Information Processing Systems},
  year={2016},
}

@article{MB-PN:2003,
  title={Laplacian eigenmaps for dimensionality reduction and data representation},
  author={Belkin, Mikhail and Niyogi, Partha},
  journal={Neural Computation},
  volume={15},
  number={6},
  pages={1373--1396},
  year={2003},
  publisher={MIT Press}
}

@misc{OpenStreetMap:2017,
   author = {{OpenStreetMap contributors}},
   title = {{Planet dump retrieved from https://planet.osm.org }},
   howpublished = "\url{ https://www.openstreetmap.org }",
   year = {2017},
 }

@InProceedings{AC-RP-FB:2016,
author = {A. Carron and R. Patel and F. Bullo},
title = {Hitting time for doubly-weighted graphs with application to robotic surveillance},
booktitle = {European Control Conference},
year = {2016},
pages = {661-665},
}

@book{MS:2007,
  title={Metric Spaces},
  author={Searc{\'o}id, M{\'\i}che{\'a}l {\'O}},
  year={2007},
  publisher={Springer}
}

@misc{wandb,
  title = {Experiment Tracking with Weights and Biases},
  year = {2020},
  note = {Software available from wandb.com},
  url = {https://www.wandb.com/},
  author = {Biewald, Lukas}
}

@inproceedings{snoek2012practical,
  title = {Practical Bayesian Optimization of Machine Learning Algorithms},
  author = {Snoek, Jasper and Larochelle, Hugo and Adams, Ryan P.},
  booktitle = {Advances in Neural Information Processing Systems},
  volume = {25},
  year = {2012}
}

@article{portugal2013multi,
  title={Multi-robot patrolling algorithms: {E}xamining performance and scalability},
  author={Portugal, David and Rocha, Rui P},
  journal={Advanced Robotics},
  volume={27},
  number={5},
  pages={325--336},
  year={2013},
  publisher={Taylor \& Francis}
}

	\appendix
	\section*{Appendix}
	
	\setcounter{subsection}{0}
	\setcounter{subsubsection}{0}
	
	\renewcommand{\thesubsection}{\Alph{subsection}}
	\renewcommand{\thesubsubsection}{\thesubsection\arabic{subsubsection}}

	\subsection{Auxiliary definitions and lemmas} \label{appendix:auxiliary}
	We first introduce an auxiliary concept and two preliminary results that will be  used repeatedly in the subsequent proofs.

	\textbf{Canonical motion.}
	We fix a canonical motion rule to uniquely determine how the system evolves 
	between two position configurations. Specifically, for every pair of points in $\Ggeo$, we fix one shortest path between them, using an arbitrary but fixed tie-breaking rule whenever multiple shortest paths exist. Given two position configurations $\boldsymbol{p}$ and  $\boldsymbol{p}'$, the canonical motion from $\boldsymbol{p}$ to $\boldsymbol{p}'$ is defined as follows: each robot travels at unit speed along the fixed shortest path, and any robot that arrives earlier waits at its destination until all robots have arrived. This rule induces a unique canonical transition between any two joint position configuration, whose duration is at most  
	the graph diameter $\graphdiameter$.

	\textbf{Auxiliary lemmas.} We next present two elementary lemmas that capture intrinsic properties of the monitoring problem.
	
	\begin{lemma}[(Coverage in sufficiently long intervals)]\label{lemma:coveragelength}
		Let $[\alpha,\beta]$ be a time interval with duration $\tau = \beta-\alpha$. Suppose that the worst-case weighted latency of a strategy $\strategy$ over $[\alpha,\beta]$ is bounded by $L$, i.e., $\sup_{t\in [\alpha,\beta] }\maxlatency{\strategy}{t}\leq L$.
		If $\tau> \frac{L}{\weightmin }$, then every node \(v\in\nodes\) is visited at least once during $[\alpha,\beta]$.
	\end{lemma}
	
	\begin{proof}
		We prove by contradiction.   Suppose that some node $v$ is not visited during  $[\alpha,\beta]$. Then its latency increases continuously throughout the interval, and hence
		$\latency{v}{\strategy}(\beta)= \latency{v}{\strategy}(\alpha) + \tau \ge \tau$.
		Therefore, 
		\[
		\maxlatency{\strategy}{\beta}\ge \weight(v) \latency{v}{\strategy}(\beta) \ge \weightmin \tau > \weightmin \times \frac{L}{\weightmin } = L,
		\]
		which contradicts with the assumption that $\sup_{t\in [\alpha,\beta] }\maxlatency{\strategy}{t}\leq L$. Hence every node must be visited at least once during $[\alpha,\beta]$.
	\end{proof}
	\cref{lemma:coveragelength} shows that if the weighted latency remains below a given  threshold on a sufficiently long interval, then all nodes must be covered during that interval. The next lemma shows that once a node is visited, its terminal latency at the end of the interval no longer depends on its latency value at the beginning of the interval. 
	
	\begin{lemma}[(Independence from initial latencies after coverage)]\label{lemma:independence} If a node \(v\in\nodes\) is visited at least once during an interval $[\alpha,\beta]$ under a strategy $\strategy$, then its terminal latency $\latency{v}{\strategy}(\beta)$  is determined solely by the last visit time of $v$ within $[\alpha,\beta]$, and is independent of the initial latency value $\latency{v}{\strategy}(\alpha)$.
		\begin{proof} Since node $v$ is visited at least once in $[\alpha,\beta]$, its most recent visit time up to $\beta$, denoted by $\allvisitvatt{\strategy}(v,\beta)$, lies in $[\alpha,\beta]$. 
			By the definition of  instantaneous latency in \eqref{definition:latency}, $\latency{v}{\strategy}= \beta - \allvisitvatt{\strategy}(v,\beta)$. Thus, the terminal latency $\latency{v}{\strategy}$ depends only on $\allvisitvatt{\strategy}(v,\beta)$ and is independent of the initial latency $ \latency{v}{\strategy}(\alpha)$. 
		\end{proof}
		
	\end{lemma}
	
	\cref{lemma:independence} shows that once a  node is visited during an interval, its historical latency before that interval becomes irrelevant to its terminal latency.

	\subsection{Proof of \cref{thm:existence}}\label{appendix:existence}
    We first consider the case $T< +\infty$. The proof proceeds in three steps. First, by using the Arzel\`{a}-Ascoli  theorem, we show that the feasible strategy set \(\strategyset\) defined in~\eqref{eq:jointset} is compact under the compact-open topology. Second, we show that the functional \(J_T:\strategyset\to \overRp \) is lower semicontinuous with respect to this topology. Third, the generalized extreme value theorem guarantees that a lower semicontinuous functional on a compact set attains its minimum, thereby ensuring the existence of an optimal strategy for every finite $T$.
    
    The case $T=+\infty$ is handled separately. We show that the optimal value satisfies $\optobjT{\infty}=\optobjT{0}$ and that any optimal solution to $\objectivefm{0}{\cdot}$ 
		is also optimal for $J_{\infty}$.
        
	\subsubsection{Preliminary results.}
	
	We first collect several standard results that will be used in the proof. 
		

    \begin{theorem}[(The Arzel\`{a}-Ascoli theorem~{\citep[Lemma 3.10]{MRB-AH:99}})]\label{thm:AA}
		If $X$ is a separable metric space and $Y$ is a compact metric space, then every sequence of equicontinuous maps $f_n:X\to Y$ has a 
subsequence that converges (uniformly on compact subsets) to a continuous map $f:X\to Y$.
	\end{theorem}

	\begin{lemma}[(Sequential characterization of semicontinuity in first-countable spaces~{\citep[Lemma~2.42]{CA-KB:2006}})]\label{lemma:sequenceofsemi} Let \(X\) be a first-countable topological space and let
\(f:X\to[-\infty,\infty]\). Then $f$ is upper semicontinuous if and only if, for every sequence $x_n \to x$ in $X$,
		$\limsup _{n \rightarrow \infty} f(x_n) \leq f\left(x\right)$.
	\end{lemma}
	
	
	\begin{theorem}[(Generalized extreme value theorem~{\citep[Theorem~2.43]{CA-KB:2006}})]\label{thm:extremevalue}
		A real-valued  lower semicontinuous function on a compact set attains a minimum value.
	\end{theorem}

	\subsubsection{Proof of \cref{thm:existence} for $T<\infty$.} \label{appendix:existenceoffinite}
	
	We will apply the above results to prove~\cref{thm:existence} in the case of $T<\infty$. To apply the generalized extreme value theorem, we first show that the feasible strategy set $\strategyset$ defined in~\eqref{eq:jointset} is compact under the compact-open topology.
	
	\begin{lemma}[(Compactness of the  strategy set $\strategyset$)] \label{lemma:compactofstrategyset}
		Equip \(\mathfrak{C}([0,\infty),\Ggeok)\) with the compact–open topology. Then the feasible strategy set \(\strategyset \subset \mathfrak{C}([0,\infty),\Ggeok)\) is compact with respect to this topology.
	\end{lemma}
	\begin{proof}
		We recall the definition of the strategy set $\strategyset$ in~\eqref{eq:jointset}:
		\begin{multline*}
			\strategyset := \\\{\strategy \in \mathfrak{C}(\Rp, \Ggeo^\numrobot):  \distGk(\strategy(t),\strategy(t'))\leq |t-t'|,\forall t,t' \}.
		\end{multline*}                    
        Since $[0,\infty)$ is hemicompact and  $\Ggeo^\numrobot$ is a compact metric space, the space \(\mathfrak{C}([0,\infty),\Ggeo^\numrobot)\) with the compact-open topology is metrizable~\citep[Chapter IV]{RC-IN:2006}. Hence, its subspace \(\strategyset\) is also metrizable. For metrizable spaces, compactness and sequential compactness are equivalent~\citep[Theorem~28.2]{JM:2000}. Therefore, it suffices to prove that $\strategyset$ is sequentially compact, i.e.,  every sequence in \(\strategyset\) has a subsequence converging to a limit that still belongs to \(\strategyset\).
        
        Let \((\strategy_n)_{n\ge1}\subset\strategyset\) be arbitrary. Since every \(\strategy_n\) is \(1\)-Lipschitz, the sequence \((\strategy_n)_{n\ge1}\) is equicontinuous on \([0,\infty)\)~\cite[Example 9.4.9]{HHS:14}. Moreover, \([0,\infty)\) is a separable metric space, and \(\Ggeok\) is a compact metric space. Therefore, by \cref{thm:AA}, there exists a subsequence \((\hat\strategy_n)_{n\ge1}\) and a continuous map $\strategy:[0,\infty)\to\Ggeok$ such that \(\hat\strategy_n\) converges to \(\strategy\) uniformly on every compact subset of \([0,\infty)\). Since \(\Ggeok\) is a metric space, the compact-open topology coincides with the topology of compact convergence~\citep[Theorem~46.8]{JM:2000}. Hence $\hat\strategy_n\to\strategy$ in the compact-open topology.
        
        It remains to show that \(\strategy\in\strategyset\). Fix arbitrary \(t,t'\in\Rp\). Since \(\hat\strategy_n\to\strategy\) uniformly on compact subsets, in particular, we have $\hat\strategy_n(t)\to\strategy(t)$ and $\hat\strategy_n(t')\to\strategy(t')$. 
        Using the continuity of the metric \(\distGk\), we obtain
        \[
        \begin{aligned}
        \distGk(\strategy(t),\strategy(t'))
        &=
        \lim_{n\to\infty}
        \distGk(\hat\strategy_n(t),\hat\strategy_n(t')) \\
        &\le
        \lim_{n\to\infty}|t-t'| \\
        &=
        |t-t'|.
        \end{aligned}
        \]
        Thus \(\strategy\) is \(1\)-Lipschitz, and hence \(\strategy\in\strategyset\). Consequently, $\strategyset$ is sequentially compact and thus compact.

	\end{proof}
	
	After showing that $\strategyset$ is compact, we move to the second condition of~\cref{thm:extremevalue} and prove that $\objectivefm{T}{\cdot}$ is lower semicontinuous. 
	
	\begin{lemma}[(Lower semicontinuity of $\objectivefm{T}{\cdot
		}$)] \label{lemma:semicontinuous}
		The functional \(J_T:\strategyset\to \overRp\) with $T<\infty$ is lower semicontinuous.
		
	\end{lemma}
	\begin{proof}
	The proof is organized as follows.
    We first establish the upper semicontinuity of the most recent visit time $\allvisitvatt{\strategy}(v,t)$ with respect to $\strategy$ for each fixed $v\in\nodes$ and $t\geq 0$. Since $\latency{v}{\strategy}(t)=t-\allvisitvatt{\strategy}(v,t)$, this implies that the instantaneous latency is lower semicontinuous. The lower semicontinuity of $\objectivefm{T}{\strategy}$ then follows from the fact that pointwise suprema of lower semicontinuous functions are lower semicontinuous.
    
        Recall that, as defined in~\eqref{eq:deftauv1}, for each node \(v\in\nodes\), the most recent visiting time to $v$ by the robot system up to time $t$ is 
		\begin{multline*} \label{eq:deftauv2}
			\allvisitvatt{\strategy}(v,t)=\\
			\begin{cases}
				\sup\{ t'\leq t: \schedule r(t')=v,  \exists  r\in\robots\},&\text{if  such  $t'$ exists}, \\
				0, & \text{otherwise}.
			\end{cases}
		\end{multline*}
		We first prove that $  \allvisitvatt{\strategy}(v,t)$ is upper semicontinuous with respect to $\strategy$. Since $\mathfrak{C}([0,\infty),\Ggeok)$ is metrizable and hence first-countable~\citep[Chapter 4.1]{RE:1989}, its subspace \(\strategyset\) is also  first-countable~\citep[§16]{SW:1970}. Therefore, by \cref{lemma:sequenceofsemi}, it suffices to verify the sequential characterization of upper semicontinuity. Let \(\strategy_n = (\schedule{n,r})_{r\in\robots}\to\strategy=(\schedule{r})_{r\in\robots}\) in the compact-open topology. Equivalently, $\sup_{t'\in[0,m]}\distGk\bigl(\strategy_n(t'),\strategy(t')\bigr)\to 0$ for all $m<+\infty$. We prove that $\limsup_{n\to\infty}\   \allvisitvatt{\strategy_n}(v,t)\ \le\   \allvisitvatt{\strategy}(v,t)$.
        

         Let $a:=\allvisitvatt{\strategy}(v,t)$. Suppose, for contradiction, that $\limsup_{n\to\infty}\allvisitvatt{\strategy_n}(v,t)>a$. Then there exists \(\varepsilon>0\) and a subsequence of $(\strategy_n)$, denoted by
        \((\strategy^{(1)}_{n})\), such that $\allvisitvatt{\strategy^{(1)}_n}(v,t)
        >
        a+\varepsilon$ for all $n\ge1$. 
        By the definition of \(\allvisitvatt{\strategy^{(1)}_n}(v,t)\), for each \(n\), there
        exist a robot \(r_n^{(1)}\in\robots\) and a time $\hat{t}^{(1)}_n\in[a+\varepsilon,t]$ 
        such that $\schedule{n,r^{(1)}_n}^{(1)}(\hat{t}^{(1)}_n)=v
        $.

Since the robot set $\robots$ is finite while the sequence $(r^{(1)}_n)_{n\geq 1}$ is infinite, 
the pigeonhole principle implies that at least one robot $r\in \robots$ 
appears as $r^{(1)}_n$ for infinitely many indices. 
And thus there exists a further
subsequence $(\strategy^{(2)}_n)$ of $(\strategy^{(1)}_n)$ such that $r^{(2)}_n\equiv r$ along this subsequence, i.e., we have $\schedule{n,r}^{(2)}(\hat{t}_n^{(2)})=v$  for all $n\ge1$.
Now the robot index is fixed as $r$, while the corresponding visiting times 
$(\hat t^{(2)}_n)_{n\ge1}$ all lie in the compact interval $[a+\varepsilon,t]$.
there exists a further subsequence $(\strategy^{(3)}_n)$ of $(\strategy^{(2)}_n)$,
such that $\hat{t}^{(3)}_n\to \hat{t}$ for some $\hat{t}\in[a+\varepsilon,t]$.

We now show that \(\schedule{r}(\hat{t})=v\). Since \(\strategy_n^{(3)}\to\strategy\)
uniformly on \([0,t]\), $ 0\leq \hat t \leq t$, and each \(\schedule{n,r}^{(3)}\) is \(1\)-Lipschitz, we have
\[
\begin{aligned}
	&\quad\distG\bigl(\schedule{n,r}^{(3)}(\hat{t}^{(3)}_n),\schedule{r}(\hat{t})\bigr)\\
	&\le
	\distG\bigl(\schedule{n,r}^{(3)}(\hat{t}^{(3)}_n),\schedule{n,r}^{(3)}(\hat{t})\bigr)
	+
	\distG\bigl(\schedule{n,r}^{(3)}(\hat{t}),\schedule{r}(\hat{t})\bigr) \\
	&\le
	|\hat{t}^{(3)}_n-\hat{t}|
	+
	\sup_{t'\in[0,t]}
	\distGk\bigl(\strategy_n^{(3)}(t'),\strategy(t')\bigr).
\end{aligned}
\]
The first term converges to \(0\) because \(\hat{t}^{(3)}_n\to \hat{t}\), and the second term
converges to \(0\) by compact-open convergence~\citep[Theorem~46.8]{JM:2000}. Hence $\schedule{n,r}^{(3)}(\hat{t}^{(3)}_n)\to\schedule{r}(\hat{t})$. 
Moreover, \(\schedule{n,r}^{(3)}(\hat{t}^{(3)}_n)=v\) for every \(n\). Therefore, $\schedule{r}(\hat{t})=v$. 
Thus, under the limiting strategy \(\strategy\), node \(v\) is visited at time
\(\hat{t}\). Since $\hat{t}\ge a+\varepsilon>a$ and $\hat{t} \leq t$, this contradicts the definition of $a=\allvisitvatt{\strategy}(v,t)$ as the most recent visit time to \(v\) no later than \(t\). Hence,
\[
\limsup_{n\to\infty}
\allvisitvatt{\strategy_n}(v,t)
\le
\allvisitvatt{\strategy}(v,t).
\]
        Combining this sequential inequality with the first-countability of  $\strategyset$,  \cref{lemma:sequenceofsemi} implies that the map $\strategy\mapsto \allvisitvatt{\strategy}(v,t)$ is upper  semicontinuous. Consequently, $\strategy\mapsto \latency{v}{\strategy}(t)=t-\allvisitvatt{\strategy}(v,t)$ is lower semicontinuous. Since \(\weight(v)>0\), the weighted latency $\strategy\mapsto \weight(v)\latency{v}{\strategy}(t)
        $ is also lower semicontinuous. Because \(\nodes\) is finite, $\maxlatency{\strategy}{t}=\max_{v\in\nodes}
        \weight(v)\latency{v}{\strategy}(t)$ is lower semicontinuous as a finite maximum of lower semicontinuous functions. Finally, for \(T<+\infty\), since the pointwise supremum of any family of lower semicontinuous functions is lower semicontinuous~\citep[Lemma~2.41]{CA-KB:2006}, it follows that $\strategy\mapsto \objectivefm{T}{\strategy}$ is lower semicontinuous under the compact-open topology.
        
	\end{proof}
	
	Based on the above two lemmas, we  now prove~\cref{thm:existence} for $T<\infty$.
	
	\begin{proof}[Proof of \cref{thm:existence} for $T<\infty$]   As shown in~\cref{subsection:existence}, the optimal value is finite, i.e., $\inf_{\strategy\in\strategyset} \objectivefm{T}{\strategy}<\infty$.
		By \cref{lemma:compactofstrategyset}, the strategy set \(\strategyset\) is compact under the compact-open topology. By \cref{lemma:semicontinuous}, the functional  $\strategy\mapsto \objectivefm{T}{\strategy}$ is lower semicontinuous for every finite $T<\infty$. Therefore, by the generalized extreme value theorem  in~\cref{thm:extremevalue}, there exists an optimal strategy $\strategy_T^* \in \strategyset$ such that $ \objectivefm{T}{\strategy_T^*} \;=\; \inf_{\strategy\in\strategyset} \objectivefm{T}{\strategy}$.
	\end{proof}
	\subsubsection{Proof of \cref{thm:existence} for $T=\infty$.}\label{appendix:existencewithinfty}
	We next prove that an optimal solution also exists when $T=\infty$. The key observation is that $\optobjT{\infty}=\optobjT{0}$. Then any  optimal solution for $\objectivefm{0}{\cdot}$ is also optimal for $\objectivefm{\infty}{\cdot}$.

	\begin{proof}[Proof of \cref{thm:existence} for $T=\infty$]
		For any feasible strategy $\strategy \in \strategyset$, by definition, $\objectivefm{\infty}{\strategy} = \lim_{T_0 \to \infty} \sup_{t\geq T_0} \maxlatency{\strategy}{t} \leq \sup_{t\geq 0} \maxlatency{\strategy}{t} = \objectivefm{0}{\strategy}$. Thus, we have $\optobjT{\infty} \leq \optobjT{0}$.

We now prove the reverse inequality. Suppose, for contradiction, that there exists a strategy \(\strategy_\infty\in\strategyset\) such that $\objectivefm{\infty}{\strategy_\infty}<\optobjT{0}$. Let $\delta:=\optobjT{0}-\objectivefm{\infty}{\strategy_\infty}>0$. By the definition of the limit of the tail suprema~{\citep[Theorem~3.17]{WR:1976}}, there exists \(T_1\ge0\) such that when $t\geq T_1$, 
$$
\maxlatency{\strategy_\infty}{t} \leq \objectivefm{\infty}{\strategy_\infty} +\frac{\delta}{2} < \optobjT{0}.
$$

Define the shifted strategy
\[
\strategy_0(t)
:=
\strategy_\infty(t+T_1),
\qquad
\forall t\ge0.
\]
Since \(\strategy_\infty\) is continuous and \(1\)-Lipschitz, the shifted strategy \(\strategy_0\) also belongs to \(\strategyset\). We claim that $\maxlatency{\strategy_0}{t}
\le
\maxlatency{\strategy_\infty}{t+T_1}$ for $t\geq 0$. 
Indeed, \(\strategy_0\) follows exactly the same robot motions as \(\strategy_\infty\) after time \(T_1\), but the latency variables of \(\strategy_0\) are initialized from zero at time \(0\). Therefore, for each node \(v\), the latency of \(v\) under the shifted strategy is no larger than the corresponding latency under \(\strategy_\infty\) at the shifted time \(t+T_1\). Taking the weighted maximum over all nodes yields the claim. Consequently,
\begin{equation*}
   \objectivefm{0}{\strategy_0} = \sup_{t\geq 0} \maxlatency{\strategy_0}{t} \leq \sup_{t\geq T_1} \maxlatency{\strategy_\infty}{t} < \optobjT{0}.
\end{equation*}
This contradicts the definition of \(\optobjT{0}\). Hence no strategy can satisfy $\objectivefm{\infty}{\strategy}<\optobjT{0}$, and therefore $\objectivefm{\infty}{\strategy}\ge \optobjT{0}$ for all $\strategy\in\strategyset$. Taking the infimum over \(\strategy\in\strategyset\), we obtain $\optobjT{\infty}\ge \optobjT{0}$. 
Together with the opposite inequality proved above, this gives $\optobjT{\infty}=\optobjT{0}$.

Now let \(\strategy_0^*\) be an optimal solution for \(\objectivefm{0}{\cdot}\), whose existence has been proved in the case of finite \(T\). Then
$$
		\objectivefm{\infty}{\strategy_0^*} = \lim_{T_0 \to \infty} \sup_{t\geq T_0} \maxlatency{\strategy_0^*}{t} \leq \sup_{t\geq 0} \maxlatency{\strategy_0^*}{t} = \optobjT{0} = \optobjT{\infty}.
		$$
Since \(\optobjT{\infty}\) is the infimum of \(\objectivefm{\infty}{\cdot}\) over \(\strategyset\), we also have $\objectivefm{\infty}{\strategy_0^*}
\ge
\optobjT{\infty}$. Therefore, $\objectivefm{\infty}{\strategy_0^*}=\optobjT{\infty}$. 
Thus \(\strategy_0^*\) is also an optimal strategy for \(T=\infty\). This proves the existence of an optimal strategy for \(T=\infty\).

	\end{proof}

	Combining the proof for \(T<+\infty\) in Appendix~\ref{appendix:existenceoffinite} and the proof for \(T=\infty\) in Appendix~\ref{appendix:existencewithinfty}, we complete the proof of~\cref{thm:existence}.

	\subsection{Proof of~\cref{thm:properties}.} \label{appendix:properties}
	\begin{proof}[Proof of~\cref{thm:properties}]
    We prove the three claims in order.
    
		Regarding \ref{item:independence}), the case  $\optobjT{\infty} = \optobjT{0}$ has already been established in Appendix~\ref{appendix:existencewithinfty}. It remains to consider finite $T<\infty$. For any feasible strategy $\strategy\in \strategyset$, we have $\objectivefm{T}{\strategy} = \sup_{t\geq T} \maxlatency{\strategy}{t} \leq \sup_{t\geq 0} \maxlatency{\strategy}{t} = \objectivefm{0}{\strategy}$. Taking the infimum over \(\strategy\in\strategyset\) gives  $\optobjT{T} \leq \optobjT{0}$. Conversely, let $\strategy_T^*$ be an optimal solution to $\objectivefm{T}{\cdot}$, whose existence follows from \cref{thm:existence}. Define the shifted strategy $\strategy_0(t) = \strategy_T^*(t+T)$ for all $t\in \Rp$. Then 
		$$
		\objectivefm{0}{\strategy_0} = \sup_{t\geq 0} \maxlatency{\strategy_0}{t} \leq \sup_{t\geq T} \maxlatency{\strategy_T^*}{t} = \optobjT{T}.
		$$
		Since $\optobjT{0}\le \objectivefm{0}{\strategy_0}$, it follows that $\optobjT{0} \leq \optobjT{T}$. Combining the two inequalities yields $\optobjT{T}=\optobjT{0}=\optobjall$ for all finite $T<+\infty$. Together with the case $T=+\infty$, this proves $\optobjT{T}=\optobjall$ for all $T\in \overRp$.
		
		Regarding~\ref{item:including}), let $T_1,T_2\in\overRp$ satisfy $T_1<T_2$, and let $\strategy_{T_1}^* \in \optstrategyT{T_1}$. Since the tail objective is monotone nonincreasing with respect to the tail parameter, we have
        \[
        \objectivefm{T_2}{\strategy_{T_1}^*}
        \le
        \objectivefm{T_1}{\strategy_{T_1}^*}
        =
        \optobjT{T_1}
        =
        \optobjall.
        \]
        On the other hand, by \ref{item:independence}), we have $\objectivefm{T_2}{\strategy}\ge\optobjT{T_2}=\optobjall$ for any feasible $\strategy\in\strategyset$. 
        Thus, $\objectivefm{T_2}{\strategy_{T_1}^*}=
        \optobjall=\optobjT{T_2}$, which implies $\strategy_{T_1}^*\in\optstrategyT{T_2}$.  
        Therefore, $\optstrategyT{T_1}\subseteq\optstrategyT{T_2}$.
		
		Regarding~\ref{item:accessible}), fix any initial configuration $\inipos \in \Ggeok$, and let $\strategy^*_0 \in \optstrategyT{0}$. We construct a feasible strategy $\tilde\strategy$ starting from $\inipos$ as follows. First, $\tilde\strategy$  follows the canonical motion 
		defined in Appendix~\ref{appendix:auxiliary} from $\inipos$ to the initial joint position $\strategy_0^*(0)$. Let \(\tau\) denote the duration of this steering phase. By the definition of canonical motion, we have $\tau\le\graphdiameter$. After time \(\tau\), the strategy \(\tilde\strategy\) follows the same robot motions as \(\strategy_0^*\), that is, $\tilde\strategy(\tau+t)=\strategy_0^*(t)$ for $t\ge0$ at the level of robot positions. Since \(\strategy_0^*\) is optimal for \(T=0\), we have $\sup_{t\ge0}\maxlatency{\strategy_0^*}{t}=\optobjall$. 
By \cref{lemma:coveragelength}, every node is visited at least once by \(\strategy_0^*\) during any interval of length strictly greater than $\frac{\optobjall}{\weightmin}$. 
Therefore, after the steering phase of \(\tilde\strategy\), once an additional time longer than \(\optobjall/\weightmin\) has elapsed, all node latencies become independent of their values accumulated during the steering phase. By \cref{lemma:independence}, from that point onward the latency evolution of \(\tilde\strategy\) coincides with the corresponding latency evolution of \(\strategy_0^*\).

Now let $T>\graphdiameter+\frac{\optobjall}{\weightmin}$. 
Since \(\tau\le\graphdiameter\), we have $T-\tau>\frac{\optobjall}{\weightmin}$. 
Hence, at all times \(t\ge T\), the transient effect of the initial steering phase has disappeared, and the latency profile of \(\tilde\strategy\) is synchronized with that of \(\strategy_0^*\) at the corresponding shifted time. Consequently, $\objectivefm{T}{\tilde\strategy}\le \objectivefm{0}{\strategy_0^*}=\optobjall$.
By \ref{item:independence}), the optimal value at tail parameter \(T\) is $\optobjT{T}=\optobjall$. Therefore,
$\objectivefm{T}{\tilde\strategy}=\optobjT{T}$,
and hence $\tilde\strategy\in\optstrategyT{T}\cap\intstrategyset{\inipos}$.
This proves $\optstrategyT{T}\cap\intstrategyset{\inipos}\neq\emptyset$ for all $T>\graphdiameter+\frac{\optobjall}{\weightmin}$.
        
		
	\end{proof}

	\subsection{Proof of \cref{thm:periodic}} \label{appendix:periodic}

	We first prove that \cref{thm:periodic} holds for $\longlatency{\strategy}$, 
	and then extend the argument to the case of finite $T<\infty$.  
	For $\longlatency{\strategy}$,  the proof is based on the following idea. Along any feasible trajectory with finite asymptotic tail performance, one can find a finite trajectory segment whose weighted latency remains below a prescribed threshold and whose endpoint joint states are sufficiently close. By appending a short connecting motion, this segment can be converted into a closed position cycle. Repeating this cycle periodically yields a feasible periodic strategy whose tail worst-case weighted latency only becomes worse by at most an arbitrarily small tolerance.
	
	\subsubsection{Useful definitions and lemmas.}
    
    We first introduce several useful definitions and lemmas that will be used in the proof.

    \textbf{Joint state.} To reason about the system evolution, we use a joint state consisting of the current robot positions and the current node latency values. Specifically, at any time $t$, let $\boldsymbol{p}=(p_{r_1},\dots,p_{r_\numrobot})\in\Ggeok$ and $\boldsymbol{L}=(L_{1},\cdots,L_\numnode)\in \mathbb{R}_{\geq 0}^{\numnode}$ denote the joint robot position and the vector of weighted node latency values, respectively. We define the joint state as 
\begin{equation}\label{eq:jointstate}
		\jointstate := \left(\boldsymbol{p},\boldsymbol{L}\right).
	\end{equation}
	The joint state space is  
	\begin{equation}\label{eq:jointstatespace}
		\jointspace := \Ggeo^{\numrobot} \times \mathbb{R}_{\geq 0}^{\numnode}.
	\end{equation}
    We equip \(\jointspace\) with the following metric. 
    For two joint states $\jointstate = \left(\boldsymbol{p},\boldsymbol{L}\right)$ and $\jointstate' = \left(\boldsymbol{p}',\boldsymbol{L}'\right)$, define
	\[
	\jointdist\big(\Xi,{\Xi}')
	:= \max\Big\{\,\distGk(\boldsymbol{p},\boldsymbol{p}'),\;\max_{v\in\nodes}|L_v-L'_v|\,\Big\}.
	\]
	This metric measures the maximum discrepancy between robot positions, under the shortest path metric \(\distGk\) on \(\Ggeok\), and node latencies, under the absolute difference on \(\Rp\). 
    It is straightforward to verify that \((\jointspace, \jointdist)\) is a metric space.  Moreover, 
    the topology induced by this sup metric \(\jointdist\) coincides 
with the product topology on $\jointspace=\Ggeo^{\numrobot}\times \mathbb{R}_{\geq 0}^{\numnode}$, where \(\Ggeo^{\numrobot}\) is equipped with the topology induced by \(\distGk\) 
and \(\mathbb{R}_{\geq 0}^{\numnode}\) is equipped with the usual product topology.

	\textbf{Compactness and finite covers.} 
	Given the joint state $\jointstate$ defined in~\eqref{eq:jointstate} and the joint state space $\jointspace$ defined in~\eqref{eq:jointstatespace}, for a feasible strategy \(\strategy\), let $\jointstate^{\strategy}(t):=(\strategy(t),\boldsymbol {L}^{\strategy}(t))\in\jointspace$ denote the joint state induced by \(\strategy\) at time \(t\), where \(\boldsymbol {L}^{\strategy}(t)\) is the vector of node weighted  latencies at time \(t\). Although \(\jointspace\) itself is not compact because the latency component is unbounded, the joint-state trajectory $\jointstate^{\strategy}(t)$ of any feasible strategy $\strategy$ with finite asymptotic tail performance $\longlatency{\strategy} = \rho < \infty$ is eventually confined to a compact subset of \(\jointspace\). Indeed, by the defining property of the limit superior~{\citep[Theorem~3.17]{WR:1976}}, there exists $T_1 > 0$ such that
	\begin{equation}\label{eq:T1}
	\maxlatency{\strategy}{t} \le\rho+1, \quad \forall t \ge T_1.
	\end{equation}
	Therefore, for every node \(v\in\nodes\), 
	\[
	\latency{v}{\strategy}(t)\le\; \frac{\rho+1}{\weight(v)}, \quad \forall t \ge T_1.
	\]
	Consequently, for all $t\geq T_1$, the joint state lies in the set
	\[
	\cmpctjointspace{\rho} := \Ggeo^{\numrobot}\times \prod_{v\in\nodes}\Big[0,\tfrac{\rho+1}{\weight(v)}\Big].
	\]
	Since $\Ggeo^\numrobot$ is compact and \(\nodes\) is finite, $\cmpctjointspace{\rho}$ is also compact. Thus, we then restrict our attention to trajectories evolving in the compact domain $\cmpctjointspace{\rho}$. 
	A standard property of compact metric spaces is that they admit finite open covers~\citep{WR:1976}. Specifically, for any  $\delta>0$, there exists a finite collection of open balls of radius $\delta$ in \((\jointspace,\jointdist)\) that covers $\cmpctjointspace{\rho}$. Let the minimal number of such balls be $N_{\cmpctjointspace{\rho}}(\delta)< \infty$, i.e.,
	\begin{multline}\label{definition:coveringnumber}
		N_{\cmpctjointspace{\rho}}(\delta) := \\\min\left\{N : \exists\,y_1,\dots,y_N\in \cmpctjointspace{\rho},\, \cmpctjointspace{\rho} \subset \bigcup_{i=1}^N \mathcal{B}_{\delta}(y_i)\right\},
	\end{multline}
	where $\mathcal{B}_{\delta}(y_i)$ denotes  the open ball of radius $\delta$ centered at $y_i$ in the metric space \((\jointspace,\jointdist)\). The finite-cover property will be used to identify a long trajectory segment whose endpoints are close in the joint state space and whose weighted latency remains uniformly bounded. 
	
	We next develop several auxiliary lemmas that serve as the main technical ingredients of the proof. The first lemma shows that, after a sufficiently long transient period, one can find a finite trajectory segment whose weighted latency remains uniformly bounded and whose endpoint joint states are close in the joint state space.

	\begin{lemma}[(Existence of a long segment with close endpoints)]\label{lemma:segment}
		Let \(\strategy\in\strategyset\) satisfy  $\longlatency{\strategy}=\rho < \infty$. 
        For any $\delta\in (0,1)$,  $T_0>0$, and positive integer $m \in \mathbb{Z}_{>0}$, define  $ n_m :=  N_{\cmpctjointspace{\rho}} (\frac{1}{m})$, where $ N_{\cmpctjointspace{\rho}}(\cdot)$ is the covering number of the compact set $\cmpctjointspace{\rho}$  in~\eqref{definition:coveringnumber}. Suppose $m$ is sufficiently large so that
		\begin{equation}\label{eq:numberm}
			n_m \geq 2\left[N_{\cmpctjointspace{\rho}}\left(\frac{\delta}{2}\right) -1\right].
		\end{equation}
		Then there exist two time instants $\alpha_m <\beta_m$ such that:
		\begin{itemize}
			\item (Duration length) The duration length $\tau_m=\beta_m-\alpha_m$  satisfies
			\begin{equation}\label{eq:duration}
				\frac{T_0}{2N_{\cmpctjointspace{\rho}}\left(\frac{\delta}{2}\right)} \le \tau_m \le T_0;
			\end{equation}
			\item  (Bounded worst-case latency) The worst weighted latency is uniformly bounded on this interval:
			\begin{equation}\label{eq:boundedlatency}
				\maxlatency{\strategy}{t} \leq \rho+\frac{1}{m}, \quad \forall t\in [\alpha_m,\beta_m];
			\end{equation}
			\item (Close endpoints) The endpoint joint states at time $\alpha_m$ and $\beta_m$ are close:
			\begin{equation}\label{eq:endpointsdistance}
				\jointdist(\jointstate^{\strategy}\left({\alpha_m}),\jointstate^{\strategy}({\beta_m})\right) <  \delta.
			\end{equation}
		\end{itemize}
	\end{lemma}
	\begin{proof}
		We construct the desired interval $[\alpha_m,\beta_m]$ by the pigeonhole principle. Since $\longlatency{\strategy}=\rho <\infty$, the defining property of the limit superior implies that, for every positive integer $m$, there exists a time $T_m\geq T_1$ such that 
        \begin{equation}\label{eq:Tm}			\maxlatency{\strategy}{t} \leq \rho+\frac{1}{m}, \quad \forall t\geq T_m,
		\end{equation}
		where $T_1$ is given in~\eqref{eq:T1}. 
        
        Consider the interval  $[T_m,T_m+T_0]$, and divide it into $n_m$ equal subintervals of length $T_0/n_m$. This gives $n_m+1$ grid points:   
		\begin{equation*}
			t_{m,j} = T_m+j \frac{T_0}{n_m}, \quad j=0,1,\dots,n_m.
		\end{equation*}
		Since $T_m\geq T_1$, all joint states $\jointstate^{\strategy}({t_{m,j}})$ lie in the compact set $\cmpctjointspace{\rho}$. By definition of the covering number, $\cmpctjointspace{\rho}$ can be covered by $N_{\cmpctjointspace{\rho}}(\frac{\delta}{2})$ open balls of radius $\frac{\delta}{2}$. By the pigeonhole principle, among the $n_m+1$ joint states at the grid points, at least one such ball  contains $a_m$  joint states, where 
		\begin{equation*}\label{eq:boundofrm}
			a_m \geq\left\lceil\frac{n_m+1}{N_{\cmpctjointspace{\rho}}\left(\frac{\delta}{2}\right)}\right\rceil.
		\end{equation*}
		Let $j_1$ and $j_2$ be the smallest and largest indices, respectively, among the grid times whose joint states lie in this ball.  Define
		\begin{align*}
			&\alpha_m=t_{m, j_1}, \; \beta_m=t_{m, j_2},\\ &\tau_m=\beta_m-\alpha_m=\left(j_2-j_1\right) \frac{T_0}{n_m}.
		\end{align*}
        By~\eqref{eq:numberm}, the chosen ball contains at least two joint states, and hence $j_1<j_2$ and $\alpha_m<\beta_m$. 
		We now verify the three stated properties.
		\begin{itemize}
			\item (Duration length). Since $j_1$ and  $j_2$ are the extreme indices among the sampled joint states contained in the same ball, we have
			\begin{equation*}
				j_2-j_1\geq a_m-1\geq \left\lceil\frac{n_m+1}{N_{\cmpctjointspace{\rho}}\left(\frac{\delta}{2}\right)}\right\rceil-1\geq \frac{n_m+1}{N_{\cmpctjointspace{\rho}}\left(\frac{\delta}{2}\right)}-1.
			\end{equation*}
			Thus,
			\begin{equation*}
				\begin{aligned}
					\tau_m=(j_2-j_1)\frac{T_0}{n_m}&\geq \frac{n_m+1-N_{\cmpctjointspace{\rho}}(\frac{\delta}{2})}{n_m}\frac{T_0}{N_{\cmpctjointspace{\rho}}\left(\frac{\delta}{2}\right)} \\&\geq \frac{T_0}{2N_{\cmpctjointspace{\rho}}\left(\frac{\delta}{2}\right)},
				\end{aligned}
			\end{equation*}
			where the condition~\eqref{eq:numberm} is used in the second line. The upper bound \(\tau_m\le T_0\) follows immediately from $\alpha_m,\beta_m\in[T_m,T_m+T_0]$. 
            Hence the duration bound in~\eqref{eq:duration}  holds.
			\item (Bounded worst-case latency) Since  $\alpha_m,\beta_m \in [T_m,T_m+T_0]$, the latency bound in~\eqref{eq:boundedlatency} follows directly from~\eqref{eq:Tm}.
			\item (Close endpoints) By construction, the two joint states \(\jointstate^{\strategy}({\alpha_m})\) and \(\jointstate^{\strategy}({\beta_m})\) belong to the same open ball of radius $\frac{\delta}{2}$ in \((\jointspace,\jointdist)\), we have
			\[
			\jointdist(\jointstate^{\strategy}({\alpha_m}),\jointstate^{\strategy}({\beta_m}))< \delta.
			\]
			This proves the endpoint closeness condition in~\eqref{eq:endpointsdistance} and completes the proof.
		\end{itemize}
	\end{proof}

	The duration bounds in \cref{lemma:segment} serve two purposes.  First, the lower bound ensures that the selected interval is not arbitrarily short. By choosing $T_0$ sufficiently 
	large, the interval length can be made large enough for \cref{lemma:coveragelength} to apply, thereby guaranteeing that every node is visited at least once within the interval. Second, the upper bound ensures that the selected interval has finite duration, which is necessary for constructing a finite repeating pattern. By \cref{lemma:segment}, we can identify a finite trajectory segment $[\alpha_m,\beta_m]$ whose weighted latency is uniformly bounded and whose endpoint joint states are close. The next step is to convert this segment into one whose endpoint positions coincide, so that it can be repeated periodically. To this end, we append a short connection that steers the robot team from the terminal joint position $\{\schedule{r}(\beta_m)\}_{r\in \robots}$ back to the initial joint position $\{\schedule{r}(\alpha_m)\}_{r\in \robots}$. Since these two joint states are close, the required connection is short, and the additional latency incurred during this connection can be controlled. This idea is formalized in the following lemma.

	\begin{lemma}[(Short connection for closing position  states)]\label{lemma:closeposition}
		Suppose a joint-state trajectory segment $\jointstate^{\strategy}({t})$ on an interval $[\alpha,\beta]$ satisfies \(\jointdist(\jointstate^{\strategy}({\alpha}),\jointstate^{\strategy}({\beta}))< \delta\). 
		Let $\theta := \distGk(\strategy(\alpha),\strategy(\beta))$, then $\theta < \delta$.
		Construct an extended trajectory segment \(\tilde\strategy\) on \([0,\beta+\theta]\) as follows. On the  interval \([0,\beta]\), \(\tilde\strategy\) coincides with \(\strategy\). On the appended interval \([\beta,\beta+\theta]\), \(\tilde\strategy\) follows the canonical transition defined in Appendix~\ref{appendix:auxiliary} from the terminal joint position \(\strategy(\beta)\) back to the initial joint position \(\strategy(\alpha)\). Then the extended segment satisfies
        \begin{gather*}
			\tilde\strategy(\beta+\theta) = \strategy(\alpha), \\  \sup _{t \in [\alpha, \beta+\theta]} \maxlatency{\tilde\strategy}{t} \leq \sup _{t \in[\alpha, \beta]} \maxlatency{\strategy}{t}+\weightmax \delta. 
		\end{gather*}        
	\end{lemma}
	\begin{proof}
		By the definition of the joint-state metric \(\jointdist\), the condition \(\jointdist(\Xi^{\strategy}(\alpha),\Xi^{\strategy}(\beta))<\delta\) implies that the position components satisfy
		\[
		\theta =\distGk(\strategy(\alpha),\strategy(\beta)) < \delta.
		\]
		The canonical transition from \(\strategy(\beta)\) to \(\strategy(\alpha)\) is therefore feasible and has duration \(\theta\). By construction, this transition brings the robot team back to the initial joint position of the segment, namely, $\tilde\strategy(\beta+\theta) = \strategy(\alpha)$
        
        It remains to bound the latency during the extended segment. On the original interval \([\alpha,\beta]\), the extended trajectory coincides with \(\strategy\), and hence
        $\sup_{t\in[\alpha,\beta]}
        \maxlatency{\tilde\strategy}{t}
        =\sup_{t\in[\alpha,\beta]}
        \maxlatency{\strategy}{t}
        $. During the appended connection interval, for any node \(v\in\nodes\), the latency of \(v\) either increases at rate one or is reset to zero when \(v\) is visited. Therefore,       
        \[
        \latency{v}{\tilde\strategy}(\beta+t)
        \le
        \latency{v}{\strategy}(\beta)+t,
        \qquad
         \forall t\in[0,\theta].
        \]
        Multiplying by \(\weight(v)\) and taking the maximum over all nodes gives
        \[
        \begin{aligned}
        \maxlatency{\tilde\strategy}{\beta+t}
        &=
        \max_{v\in\nodes}
        \weight(v)\latency{v}{\tilde\strategy}(\beta+t) \\
        &\le
        \max_{v\in\nodes}
        \weight(v)
        \bigl(
        \latency{v}{\strategy}(\beta)+t
        \bigr) \\
        &\le
        \max_{v\in\nodes}
        \weight(v)\latency{v}{\strategy}(\beta)
        +
        \weightmax t \\
        &\le
        \sup_{s\in[\alpha,\beta]}
        \maxlatency{\strategy}{s}
        +
        \weightmax\theta \\
        &\le
        \sup_{s\in[\alpha,\beta]}
        \maxlatency{\strategy}{s}
        +
        \weightmax\delta .
        \end{aligned}
        \]
        Thus,
        \[
        \sup_{t\in[\beta,\beta+\theta]}
        \maxlatency{\tilde\strategy}{t}
        \le
        \sup_{t\in[\alpha,\beta]}
        \maxlatency{\strategy}{t}
        +
        \weightmax\delta .
        \]
        Combining this with the equality on the original interval \([\alpha,\beta]\), we obtain
        \[
        \sup_{t\in[\alpha,\beta+\theta]}
        \maxlatency{\tilde\strategy}{t}
        \le
        \sup_{t\in[\alpha,\beta]}
        \maxlatency{\strategy}{t}
        +
        \weightmax\delta .
        \]
        This completes the proof.     
	\end{proof}

	Lemma~\ref{lemma:closeposition} constructs a finite segment whose robot \emph{position states} are closed and whose latency performance remains well controlled. We next prove \cref{thm:periodic}  by repeatedly concatenating such a position-closed segment to build a periodic strategy. Although the full joint state need not be closed across cycles, because the latency vector may not return to its initial value, the coverage property of the selected segment ensures that node latencies are reset within each cycle and remain uniformly bounded. This allows us to show that the resulting periodic strategy satisfies the required tail-latency bound.
	
	\subsubsection{Proof of \cref{thm:periodic} }
	Given the useful definitions and auxiliary lemmas above,  we are now ready to prove the main approximation result.

	\begin{proof}[Proof of~\cref{thm:periodic} for  $T=+\infty$]  
        Let $\rho:=\longlatency{\strategy}<+\infty$. We first construct a position-closed finite segment using previous lemmas, then repeat it to form a periodic strategy, and finally show that the resulting strategy satisfies the desired tail-latency bound.
		
		Given a tolerance $\torlance>0$, choose $\bar{\delta}>0$ and $m\in \mathbb{Z}_{>0}$ such that 
		\begin{equation} \label{eq:constraints}
			\bar{\delta}\leq \min \{1,\frac{\torlance}{6 \weightmax}\}, \quad m\geq \max\{2,\frac{2}{\torlance} \}.
		\end{equation}
		Define $\bar{N}:=N_{\cmpctjointspace{\rho}}\left(\frac{\bar{\delta}}{2}\right)<\infty$. Furthermore, pick $T_0<\infty$ such that
		\begin{equation*}
			\frac{T_0}{2\bar{N}} = \frac{\rho+1}{\weightmin }.
		\end{equation*}
		By~\cref{lemma:segment}, there exist two time instants $\alpha_m<\beta_m$ such that, with $\tau_m=\beta_m-\alpha_m$, we have
		\begin{gather*}
			\frac{T_0}{2\bar{N}} \le \tau_m \le T_0, \\
			\jointdist\big(\jointstate^{\strategy}(\alpha_m),\jointstate^{\strategy}(\beta_m)\big) < \bar\delta,\\
			\sup_{t\in[\alpha_m,\beta_m]} \maxlatency{\strategy}{t}\le \rho+\tfrac{1}{m}.
		\end{gather*}
		By~\cref{lemma:closeposition}, we can append a canonical connection of duration $\theta_m < \bar\delta$ that steers the robot team from  $\strategy(\beta_m)$ to $\strategy(\alpha_m)$. Let $W_m:=\tau_m+\theta_m$ denote the duration of the resulting position-closed segment, and  define an auxiliary strategy 
        $\tilde \strategy$ such that on $[0,\beta_m]$, it coincides with the original strategy $\strategy$, and on $(\beta_m,\beta_m+\theta_m]$, it connects the segment,   as in the proof of~\cref{lemma:closeposition}.          
        Then
		\begin{gather*}
			\tilde\strategy(\beta_m+\theta_m) =  \tilde\strategy(\alpha_m),  \\  \sup _{t \in [\alpha_m,\alpha_m+W_m]} \maxlatency{\tilde \strategy}{t}  \leq \rho+\frac{1}{m}+\weightmax \bar\delta.
		\end{gather*}
        We next note that every node is visited at least once during the original segment \([\alpha_m,\beta_m]\). Indeed, since $\tau_m\geq  \frac{\rho+1}{\weightmin } >\frac{\rho+\frac{1}{m}}{\weightmin }$ and $\sup_{t\in[\alpha_m,\beta_m]} \maxlatency{{\tilde{\strategy}}}{t}=\sup_{t\in[\alpha_m,\beta_m]}\maxlatency{{{\strategy}}}{t}\le \rho+\tfrac{1}{m}$, \cref{lemma:coveragelength} implies  that every node is visited at least once in \([\alpha_m,\beta_m]\).  Hence every node is also visited at least once during the extended segment \([\alpha_m,\alpha_m+W_m]\). 
        
        We now construct the periodic strategy $\perstrategy := \{\perschedule{r}\}_{r\in \robots}$. 
        On the initial prefix, set 
        \begin{equation*}
            \perstrategy(t) =   \strategy(t), \ \forall t \in [0,\alpha_m].
        \end{equation*}
        For \(t\in[0,W_m]\) and $j\in \mathbb{N}$, define
        \begin{equation*}
				\perstrategy(\alpha_m+t+jW_m) = \tilde\strategy(\alpha_m+ t). 
		\end{equation*}
        Since the position of \(\tilde\strategy\) at \(\alpha_m+W_m\) coincides with its position at \(\alpha_m\), this concatenation is feasible. Thus \(\perstrategy\in\strategyset\), and it is periodic after time \(\alpha_m\) with period \(W_m\).
        
        
        Having constructed a periodic strategy, we now prove that its long-run worst-case weighted latency satisfies the desired bound. 
		During the initial segment $[0,\alpha_m+W_m]$, the periodic strategy 
		$\perstrategy$ follows the same robot motions as the auxiliary strategy $\tilde \strategy$. Since both strategies start from the same latency state and obey the same latency update rules, we have
		$$\jointstate^{\perstrategy}({t}) =\jointstate^{\tilde\strategy}(t),\qquad \forall\, t\in[0,\alpha_m+W_m].$$
		In particular, we have $\maxlatency{\perstrategy}{t} = \maxlatency{\tilde{\strategy}}{t}$ for all $t\in[0,\alpha_m+W_m]$. When the second period $[\alpha_m+W_m,\alpha_m+2W_m]$ begins, the robot positions coincide with those at the beginning of the repeated segment, but the latency values may differ. We first bound this discrepancy. Since $\jointdist\big(\jointstate^{\strategy}(\alpha_m),\jointstate^{\strategy}(\beta_m)\big) < \bar\delta$, we know that  $\latency{v}{\perstrategy}({\alpha_m+\tau_m})=\latency{v}{\tilde\strategy}({\alpha_m+\tau_m}) \leq \latency{v}{\tilde\strategy}({\alpha_m})+\bar\delta$. During the appended connection, whose duration satisfies \(\theta_m<\bar\delta\), the latency of any node can increase by at most \(\theta_m\). Therefore,
		$$
		\begin{aligned}
			\latency{v}{\perstrategy}({\alpha_m+W_m})& \leq \latency{v}{\perstrategy}({\alpha_m+\tau_m}) + \bar\delta \\& \le \latency{v}{\tilde\strategy}({\alpha_m})+2\bar\delta.
		\end{aligned}
		$$
        Thus, at the beginning of the second period, the latency of each node is at most \(2\bar\delta\) larger than the corresponding initial latency of the repeated segment.

		Since every node is visited at least once in each period, let $h_v^{\rm{first}} \in [0,W_m]$ denote the first relative time within one period at which node \(v\) is visited. Equivalently, node \(v\) is first visited during the second period at time $\alpha_m+W_m+h_v^{\rm first}$. For $t\in[\alpha_m+W_m, \alpha_m+W_m+h_v^{\rm{first}}]$, since both $\latency{v}{\perstrategy}({t})$ and $\latency{v}{\tilde\strategy}(t-W_m)$ grow linearly with slope $1$, we have 
		$$
		\latency{v}{\perstrategy}({t}) \le \latency{v}{\tilde\strategy}(t-W_m)+2\bar\delta.
		$$
		For $t\in (\alpha_m+W_m+h^{\rm{first}}_v,\alpha_m+2W_m]$, node $v$ has already been visited during the second period. By~\cref{lemma:independence}, its latency is then independent of the latency value at the beginning of the period. Since the robot position trajectory in the second period is identical to the repeated segment, we have
		$$
		\latency{v}{\perstrategy}({t}) = \latency{v}{\tilde\strategy}(t-W_m).
		$$
		In particular, for every node $v$, 
        \begin{align*}
         \latency{v}{\perstrategy}({\alpha_m+2W_m}) &=  \latency{v}{\tilde\strategy}({\alpha_m+W_m})\\
         &=\latency{v}{\perstrategy}({\alpha_m+W_m}). 
        \end{align*}
        Together with the position closure of the repeated segment, this shows that the full joint state at time \(\alpha_m+2W_m\) coincides with that at time \(\alpha_m+W_m\). Since the subsequent position schedule repeats the same period and the latency dynamics are deterministic, the joint-state evolution repeats identically from the third period onward. Hence, it suffices to bound the latency over the second period.

        Using the above bound and multiplying by \(\weight(v)\), we obtain, for all \(h\in[0,W_m]\),
        \[
        \maxlatency{\perstrategy}{\alpha_m+W_m+h}
        \le
        \maxlatency{\tilde\strategy}{\alpha_m+h}
        +
        2\weightmax\bar\delta .
        \]
        Thus,
        \[
        \begin{aligned}
        &\quad\sup_{t\in[\alpha_m+W_m,\alpha_m+2W_m]}
        \maxlatency{\perstrategy}{t}\\        
        &\le
        \sup_{s\in[\alpha_m,\alpha_m+W_m]}
        \maxlatency{\tilde\strategy}{s}
        +
        2\weightmax\bar\delta  \\
        &\le
        \rho+\frac{1}{m}
        +
        3\weightmax\bar\delta .
        \end{aligned}
        \]
        Since the joint-state evolution repeats from the third period onward, the long-run worst-case weighted latency of \(\perstrategy\) satisfies
        \[
        \longlatency{\perstrategy}
        \le
        \rho+\frac{1}{m}
        +
        3\weightmax\bar\delta .
        \]
        By the choices of \(\bar\delta\) and \(m\) in~\eqref{eq:constraints}, we have
        \[
        \frac{1}{m}
        +
        3\weightmax\bar\delta
        \le
        \torlance .
        \]
        Therefore,
        \[
        \longlatency{\perstrategy}
        \le
        \rho+\torlance .
        \]

	\end{proof}
	
	We have proved \cref{thm:periodic} for the asymptotic tail objective $T=+\infty$. We now consider the case of a finite tail parameter $T<+\infty$.

	\begin{proof}[Proof of~\cref{thm:periodic} for $T<+\infty$]
		Let $\strategy\in\strategyset$ satisfy $\rho: = \objectivefm{T}{\strategy} <+\infty$ and $T<+\infty$. Since $\objectivefm{\infty}{\strategy} = \lim_{T_0\to\infty} \objectivefm{T_0}{\strategy}$, we have $ \underline{\rho} :=\objectivefm{\infty}{\strategy} \leq \objectivefm{T}{\strategy} = \rho$. 
          
        We now apply the construction used in the proof for \(T=+\infty\), with \(\underline\rho\) in place of \(\rho\). More precisely, for the given tolerance \(\torlance>0\), the previous construction produces a periodic strategy \(\perstrategy\) that follows the original strategy up to time \(\alpha_m\), then repeats a position-closed finite segment constructed from a sufficiently late portion of \(\strategy\). Moreover, the same argument gives
        \[
        \sup_{t\ge \alpha_m}
        \maxlatency{\perstrategy}{t}
        \le
        \underline\rho+\torlance.
        \]

        When $T\geq \alpha_m$, we have $\objectivefm{T}{\perstrategy}
        =
        \sup_{t\ge T}
        \maxlatency{\perstrategy}{t}\leq\sup_{t\ge \alpha_m}
        \maxlatency{\perstrategy}{t}\leq \underline\rho+\torlance\leq\rho+\torlance$.
        
        When $T< \alpha_m$, we control the finite prefix before \(\alpha_m\). By construction, we have $\perstrategy(t)=\strategy(t)$ for all $t\in[0,\alpha_m]$. Hence, for all \(t\in[T,\alpha_m]\),
        \[
        \maxlatency{\perstrategy}{t}
        =
        \maxlatency{\strategy}{t}
        \le
        \sup_{s\ge T}
        \maxlatency{\strategy}{s}
        =
        \rho .
        \]
        Therefore,
        \[
        \begin{aligned}
        \objectivefm{T}{\perstrategy}
        &=
        \sup_{t\ge T}
        \maxlatency{\perstrategy}{t} \\
        &\le
        \max\left\{
        \sup_{t\in[T,\alpha_m]}
        \maxlatency{\perstrategy}{t},
        \,
        \sup_{t\ge\alpha_m}
        \maxlatency{\perstrategy}{t}
        \right\} \\
        &\le
        \max\{\rho,\underline\rho+\torlance\}\\
        &\leq \rho+\torlance.        
        \end{aligned}
        \]
        This proves \cref{thm:periodic} for every finite \(T<+\infty\).
        
	\end{proof}
	
	Combining the above two cases, we have proved~\cref{thm:periodic}.

	\subsection{Proof of \cref{lemma:existenceofdiscretized}}\label{subsection:existenceofdiscretized}

    Since we have established in \cref{lemma:semicontinuous} that $\objectivefm{T}{\cdot}$ is lower semicontinuous for every finite $T<\infty$, to apply the extreme value theorem again, it remains to  prove that $\rationalstrategyset_{\inipos}^\Delta$ is compact.
\begin{lemma}[(Compactness of the discretized rational strategy class)]
\label{lemma:compact-discrete}
For any waiting-time resolution \(\Delta>0\) and any initial node configuration
\(\inipos\in\nodes^\numrobot\), the strategy class
\(\rationalstrategyset^\Delta_{\inipos}\)
is compact under the compact-open topology inherited from
\(\mathfrak C([0,\infty),\Ggeok)\).
\end{lemma}

\begin{proof}
    Since \(\rationalstrategyset^\Delta_{\inipos}\subset \strategyset\), and
\(\strategyset\) has already been shown to be compact under the compact-open topology in \cref{lemma:compactofstrategyset}, it suffices to prove that
\(\rationalstrategyset^\Delta_{\inipos}\) is closed in \(\strategyset\).
Moreover, since \(\mathfrak{C}([0,\infty),\Ggeok)\) with the compact-open topology  is metrizable, it is enough to prove sequential closedness. Let
\(
\strategy_n = (\schedule{n,r})_{r\in\robots}\in \rationalstrategyset^\Delta_{\inipos}\) for $n\geq 1$ 
be a sequence converging to some \(\strategy\in\strategyset\) in the compact-open topology. We show that $\strategy \in \rationalstrategyset^\Delta_{\inipos}$ in the following.

First, since \(\strategy_n\to\strategy\) in the compact-open topology,
we have \(\strategy_n(0)\to\strategy(0)\). Since
\(\strategy_n(0)\equiv\inipos\) for all \(n\), it follows that $\strategy(0)=\inipos$.

We next show that, on every finite time interval, the limit strategy has
a discretized rational representation. We use a primitive
action representation for rational strategies: each robot's trajectory is obtained by
concatenating primitive actions, where each primitive action is either
an edge traversal at unit speed or a waiting action of duration
\(\Delta\) at a node. A longer waiting time can be  represented as
multiple consecutive primitive waiting actions. Define \(\varepsilon_\Delta := \min\left\{\Delta,\,\min_{e\in\edges}\edgelength(e)\right\}>0\). Then every primitive action has duration at least \(\varepsilon_\Delta\). Fix any finite horizon \(H>0\). Since each primitive action has duration
at least \(\varepsilon_\Delta\), each robot can complete at most
\(M_H:=\lceil H/\varepsilon_\Delta\rceil\) primitive actions before time
\(H\), and for each robot, its first \(M_H\) primitive actions determine its
continuous schedule on the whole interval \([0,H]\). Because the graph is finite, each robot has only finitely many possible
primitive actions at nodes: it may move to one of finitely many
neighboring nodes or wait for duration \(\Delta\). Since the number of
robots is finite, there are only finitely many possible joint primitive
action prefixes of length \(M_H\) for all robots. Hence, by the
pigeonhole principle, from the sequence \((\strategy_n)_{n\ge1}\) we can
extract a subsequence, denoted by
\((\strategy_n^H)_{n\ge1}\), such that the first \(M_H\) primitive
actions of every robot are identical along this subsequence. These identical primitive action prefixes uniquely determine the same
continuous joint trajectory on \([0,H]\). Consequently, all strategies in
\((\strategy_n^H)_{n\ge1}\) coincide on \([0,H]\). Since
\(\strategy_n\to\strategy\) in the compact-open topology, the subsequence
\((\strategy_n^H)_{n\ge1}\) also converges uniformly to \(\strategy\) on
\([0,H]\). Therefore, the uniform limit \(\strategy|_{[0,H]}\) must
coincide with this same discretized rational trajectory on \([0,H]\).

Since \(H>0\) is arbitrary, for every compact interval \([0,H]\), the
restriction \(\strategy|_{[0,H]}\) admits a representation by edge
traversals at unit speed and primitive \(\Delta\)-waiting actions at
nodes. These finite-horizon representations for different $H$ are compatible on overlaps,
because they all represent the same continuous trajectory \(\strategy\).
Thus they define a global discretized rational representation of
\(\strategy\) on \([0,\infty)\). Together with
\(\strategy(0)=\inipos\), this implies $\strategy\in\rationalstrategyset^\Delta_{\inipos}$. Therefore, \(\rationalstrategyset^\Delta_{\inipos}\) is sequentially
closed in \(\strategyset\), and hence closed in \(\strategyset\). Since
\(\strategyset\) is compact by~\cref{lemma:compactofstrategyset}, it
follows that \(\rationalstrategyset^\Delta_{\inipos}\) is compact.

\end{proof}

  \begin{proof}[Proof of \cref{lemma:existenceofdiscretized}] 
By \cref{lemma:compact-discrete}, the discretized rational strategy class $\rationalstrategyset^\Delta_{\inipos}$ is compact under the compact-open topology. By \cref{lemma:semicontinuous}, the functional \(\objectivefm{T}{\cdot}\) is lower semicontinuous for every finite tail parameter $T<\infty$. Therefore, by the generalized extreme value theorem in~\cref{thm:extremevalue}, \(\objectivefm{T}{\cdot}\) attains its  minimum over \(\rationalstrategyset_{\inipos}^{\Delta}\).

  \end{proof}

	\subsection{Proof of \cref{thm:suboptimality}}\label{appendix:discretization}
	We first prove an auxiliary lemma. 
	\begin{lemma}[(Bound on the discretization error)]\label{lemma:discretization} Let $\Delta>0$. For any rational strategy \(\strategy\in\rationalstrategyset\), there exists a  strategy with discretized waiting time \(\strategy_\Delta \in \rationalstrategyset_\Delta\) with the same initial configuration as \(\strategy\), such that the following inequalities hold:
    \begin{enumerate}
        \item for any finite \(T>\Delta\),
        \begin{equation*}
			\objectivefm{T}{\strategy_\Delta} \leq  \objectivefm{T-\Delta}{\strategy}+2 \weightmax \Delta;
		\end{equation*}
        \item for \(T=+\infty\),
        \begin{equation*}
            \objectivefm{\infty}{\strategy_\Delta} \le \objectivefm{\infty}{\strategy}+2\weightmax\Delta.
		\end{equation*}
    \end{enumerate}
	\end{lemma}

	\begin{proof} 
		We construct $\strategy_\Delta$ by discretizing the waiting duration of $\strategy$ while keeping the traveling phases unchanged. For each robot \(r\in\robots\), decompose its motion into an alternating sequence of \emph{waiting} and \emph{traveling} phases, and let
		\begin{equation*}
			0=t_{r,0}<t_{r,1}<t_{r,2}<\cdots			
		\end{equation*}
        denote the corresponding \emph{event times}. Define
        \begin{equation*}
            \tau_{r,j}=t_{r,j}-t_{r,j-1}>0,\quad \textup{for }j\ge1.
        \end{equation*}
		Let $\mathcal W_r$ denote the set of indices corresponding to waiting phases, and let $\mathcal M_r$ denote the set of indices corresponding to traveling phases. Thus, $ \{\tau_{r,j}\}_{j\in \mathcal W_r}$ are waiting durations, while $\{\tau_{r,j}\}_{j\in \mathcal M_r}$ are traveling durations determined by the edge lengths.

        For each traveling phase $j\in \mathcal M_r$, we keep the duration unchanged: $\hat\tau_{r,j}=\tau_{r,j}$. For each waiting phase $j \in \mathcal{W}_r$, we quantize $\tau_{r,j}$ to one of its two nearest integer multiples of \(\Delta\): $\hat\tau_{r,j}\in\bigl\{\lfloor \frac{\tau_{r,j}}{\Delta}\rfloor\Delta,\ \lceil \frac{\tau_{r,j}}{\Delta} \rceil\Delta\bigr\}$. The choice is made so as to keep the accumulated timing error bounded. Specifically, define
        \begin{equation*}
			E_{r,0}=0,\quad
			E_{r,j}=\sum_{i\in \mathcal W_r,\,i\le j}(\hat\tau_{r,i}-\tau_{r,i}),\quad j\ge1.
        \end{equation*}
        When processing a waiting phase $j\in \mathcal W_r$, if the accumulated error satisfies $E_{r,j-1} \leq 0$, we round up: $\hat\tau_{r,j}= \lceil \frac{\tau_{r,j}}{\Delta} \rceil \Delta$. Otherwise, we round down: $\hat\tau_{r,j}= \lfloor \frac{\tau_{r,j}}{\Delta} \rfloor \Delta$. By induction, this greedy rule guarantees $|E_{r,j}|\le\Delta$ for all $r\in\robots$ and  $j\geq0$.

        Define the discretized event times recursively by
		$$
		\hat{t}_{r,0} = 0, \quad \hat{t}_{r,j} = \hat{t}_{r,j-1} + \hat\tau_{r,j}, \quad \textup{for }j\geq 1.
		$$
		If a waiting duration is rounded to zero, the corresponding waiting segment simply disappears without affecting the order of the event times. The resulting strategy is denoted by $\strategy_{\Delta}$. By construction, $\strategy_{\Delta}$ has the same initial configuration as $\strategy$, follows the same sequence of node-to-node motions, and has waiting durations in integer multiples of $\Delta$. Moreover, for every robot $r\in \robots$ and every event index $j \geq 1$,
		$$
		\begin{aligned}
			\hat{t}_{r,j} - t_{r,j}& = \sum_{i\leq j} (\hat \tau_{r,i} - \tau_{r,i}) \\&   =\sum_{i\in \mathcal{W}_r,\,i\le j}(\hat\tau_{r,i}-\tau_{r,i})\\& = E_{r,j}.
		\end{aligned}
		$$
		Therefore 
		\begin{equation} \label{eq:error}
			\left|\hat{t}_{r, j}-t_{r, j}\right| \leq \Delta,  \quad \forall r \in \robots,\ \forall  j\geq 0 .
		\end{equation}

        We now compare the node visit times under \(\strategy\) and \(\strategy_\Delta\). Fix a node $v\in \nodes$, a robot \(r\in\robots\), and a time \(t\ge\Delta\). Let $\rvisitvatt{r}{\strategy}(v,t)$ be the most recent time at which robot \(r\) visits node \(v\) no later than \(t\) under \(\strategy\), as defined in~\eqref{eq:tauvr}. 
        We consider two cases:
		\begin{itemize}
            \item If robot $r$ has not visited $v$ by time $t-\Delta$ under $\strategy$, then $\rvisitvatt{r}{\strategy}(v,t-\Delta)=0 \leq \rvisitvatt{r}{\strategy_\Delta}(v,t)$, and hence  $\rvisitvatt{r}{\strategy_\Delta}(v,t) \geq \rvisitvatt{r}{\strategy}(v,t-\Delta)-\Delta$ holds trivially.
			\item  If robot $r$ has visited $v$ by time $t-\Delta$ under $\strategy$, then there are two possible cases. When $\rvisitvatt{r}{\strategy}(v,t-\Delta)$ is an event time in $\{t_{r,j}\}_{j\geq0}$, the corresponding visit under \(\strategy_\Delta\) occurs at some time $s$ satisfying $\rvisitvatt{r}{\strategy}(v,t-\Delta)-\Delta\leq  s\leq \rvisitvatt{r}{\strategy}(v,t-\Delta)+\Delta\le t$ by~\eqref{eq:error}. Thus, we have $\rvisitvatt{r}{\strategy_\Delta}(v,t)\ge s\geq 
                    \rvisitvatt{r}{\strategy}(v,t-\Delta)-\Delta$. When the robot is waiting at \(v\) at time $t-\Delta$, i.e., $\rvisitvatt{r}{\strategy}(v,t-\Delta) = t-\Delta$, the endpoints of the corresponding discretized waiting interval are each shifted by at most \(\Delta\), and the last visit time before $t-\Delta$ to $v$ by robot $r$ under $\strategy_\Delta$ is no earlier than $t-2\Delta$. In this case, we also have $\rvisitvatt{r}{\strategy_\Delta}(v,t) \geq t -2\Delta = \rvisitvatt{r}{\strategy}(v,t-\Delta)-\Delta$.

			
		\end{itemize}
		Taking the maximum over all robots gives
		$$
		\allvisitvatt{\strategy_\Delta} (v,t)\geq  \allvisitvatt{\strategy}(v,t-\Delta)-\Delta.
		$$
		Consequently, 
		\begin{equation}
			\begin{aligned}
				\latency{v}{\strategy_\Delta}(t)& = t - \allvisitvatt{\strategy_\Delta}(v,t) \\& \leq t - ( \allvisitvatt{\strategy}(v,t-\Delta)-\Delta )\\& = \latency{v}{\strategy}({t-\Delta})+2 \Delta.
			\end{aligned}  
		\end{equation}
		Multiplying by \(\weight(v)\) and taking the maximum over all nodes yields
		$$
		\maxlatency{\strategy_\Delta}{t} \leq \maxlatency{\strategy}{t-\Delta}+2\weightmax \Delta.
		$$
		For any finite \(T>\Delta\), we have
		\begin{equation*}
			\begin{aligned}
				\objectivefm{T}{\strategy_\Delta}& = \sup_{t\geq T} \maxlatency{\strategy_\Delta}{t}\\
				& \leq \sup_{t\geq T} \maxlatency{\strategy}{t-\Delta} +2\weightmax\Delta \\
				&= \sup_{t\geq T-\Delta} \maxlatency{\strategy}{t} +2\weightmax\Delta   \\
				& = \objectivefm{T-\Delta}{\strategy} + 2\weightmax \Delta.
			\end{aligned}
		\end{equation*}
       For $T=\infty$, we have
        $$
        \begin{aligned}
        \longlatency{\strategy_\Delta}& = \limsup_{t\to +\infty  }\maxlatency{\strategy_\Delta}{t} \\
        & \leq \limsup_{t\to +\infty  }\maxlatency{\strategy}{t-\Delta} +2\weightmax\Delta \\
        & =  \longlatency{\strategy} +2\weightmax\Delta.
        \end{aligned}
        $$
		This completes the proof.        							
		
	\end{proof}
	
	We now prove~\cref{thm:suboptimality} using~\cref{lemma:discretization}. 
		\begin{proof}[Proof of~\cref{thm:suboptimality}]

            Define
            \[
            T_\Delta
            :=
            \begin{cases}
            T-\Delta, & T<+\infty,\\
            \infty, & T=+\infty .
            \end{cases}
            \]
            Since $T
            >
            \graphdiameter
            +
            \frac{\optobjall}{\weightmin}
            +
            \Delta$, we have $T_\Delta
            >
            \graphdiameter
            +
            \frac{\optobjall}{\weightmin}$. By \cref{thm:properties}, there exists a tail-optimal strategy starting from \(\inipos\) for the objective with tail parameter \(T_\Delta\). Moreover, this strategy can be chosen from $\rationalstrategyset_{\inipos}$. Denote it by $\strategy^*_{\inipos}
            \in
            \optstrategyT{T_\Delta}
            \cap
            \rationalstrategyset_{\inipos}$. Then $\objectivefm{T_\Delta}{\strategy^*_{\inipos}}
            =
            \optobjall$. Applying \cref{lemma:discretization} to \(\strategy^*_{\inipos}\), we obtain a discretized-waiting-time strategy $\strategy^*_{\inipos,\Delta}
            \in
            \rationalstrategyset^\Delta_{\inipos}$ with the same initial configuration \(\inipos\). Moreover, we have $\objectivefm{T}{\strategy_{\inipos,\Delta}^*}\leq \objectivefm{T_{\Delta}}{\strategy^*_{\inipos}}+2\weightmax\Delta =\optobjall+2\weightmax\Delta$. This proves \cref{thm:suboptimality}.
		\end{proof}

	
	\subsection{Proof of~\cref{thm:MDPopt}}\label{appendix:MDPopt}

We prove \cref{thm:MDPopt} by establishing two constructive comparisons between the TWLO-MDP and the original monitoring problem. 
First, we show that every sample trajectory generated by a non-Zeno MDP policy induces a feasible rational monitoring strategy. 
Moreover, for any non-Zeno MDP policy, there exists at least one trajectory-induced rational monitoring strategy whose tail performance is no larger than the policy average cost. This implies that the optimal MDP value over non-Zeno policies cannot be smaller than the optimal value of the original monitoring problem. 

Conversely, starting from an optimal rational monitoring strategy, we construct a non-Zeno stationary Markov policy for the TWLO-MDP whose average cost is no larger than the tail performance of the monitoring strategy. 
This gives the reverse inequality. 
Combining the two inequalities, together with the existence of an optimal rational monitoring strategy established in \cref{thm:existence}, proves the value equivalence and the existence of an optimal non-Zeno stationary Markov policy.
We first prove the direction that maps non-Zeno MDP trajectories to monitoring strategies, which yields the lower bound on the MDP value.

\begin{lemma}[
(From a non-Zeno MDP policy to a monitoring strategy)]
\label{lemma:mdp2strategy}
For any non-Zeno MDP policy $\mu\in\Omega_{\mathrm{NZ}}$, there exists a feasible rational monitoring strategy 
$\strategy_\mu\in\rationalstrategyset_{\inipos}$ such that
\[
\objectivefm{T}{\strategy_\mu} \le \aveobj_{s_0}(\mu).
\]
\end{lemma}
\begin{proof}
    Fix a non-Zeno policy \(\mu\in\Omega_{\mathrm{NZ}}\). We first prove a pathwise correspondence. 
    Consider any sample trajectory generated by \(\mu\),  $\trajectory=\left\{\left(s_0, \boldsymbol{a}_0\right),\left(s_1, \boldsymbol{a}_1\right), \ldots\right\}$ with associated event times $\{t_n\}_{n\ge 0}$. Once the sample trajectory is fixed, the sequence of states and actions is deterministic.
    
    We construct a continuous-time monitoring strategy
$\strategy_\trajectory$ by executing the actions prescribed along this event-driven trajectory. 
Specifically, at each event time $t_n$, the action $\boldsymbol a_n$ specifies, for each available robot, either a neighboring node to move toward or a waiting segment at its current node. 
On the subsequent event interval $[t_n,t_{n+1})$, the strategy $\strategy_\trajectory$ follows exactly this prescribed movement or waiting segment. 
Repeating this construction over all event intervals gives a concatenated continuous-time robot trajectory. 
Since $\mu\in\Omega_{\mathrm{NZ}}$, the induced event time sequence satisfies $t_n\to+\infty$. Hence these event intervals cover the whole time horizon $[0,+\infty)$. 
Therefore, the concatenation defines a unique feasible rational monitoring strategy
$\strategy_\trajectory\in\rationalstrategyset_{\inipos}$.

We next compare the trajectory-wise MDP cost with the tail monitoring objective of \(\strategy_\trajectory\). By the construction of the TWLO-MDP, the tracker component satisfies
		$$
		z_n = \begin{cases} \sup_{T\leq t \le t_n} \maxlatency{\strategy_\trajectory}{t} , & \text { if } t_{n}\geq T, \\ 0, & \text { otherwise} .\end{cases}
		$$
	    Note that for $t_n\leq T$, $z_n=0$ while for $t_n\geq T$,
		$$z_{n+1} = \sup_{T\leq t \leq t_{n+1}} \maxlatency{\strategy_{\trajectory}}{t} \geq  \sup_{T\leq t \leq t_{n}} \maxlatency{\strategy_{\trajectory}}{t} = z_n,$$ so $\{z_n\}$ is non-decreasing. 
        Since $t_n\to+\infty$, we have $\cup_{n\geq0} [T,t_n) = [T,+\infty)$. Consequently, 
		$$
		\sup_{n\geq0} z_n = \sup_{n\geq0} \sup_{T\leq t \leq t_n} \maxlatency{\strategy_{\trajectory}}{t} = \sup_{t\geq T} \maxlatency{\strategy_{\trajectory}}{t} =\objectivefm{T}{\strategy_{\trajectory}}
		.$$
		If $\objectivefm{T}{\strategy_{\trajectory}}<+\infty$ , then by the monotone convergence theorem~\citep[Theorem 2.4.2]{SA:2015}, the nondecreasing sequence \((z_n)_{n\ge0}\) converges to its supremum: $\lim_{n \to \infty} z_n = \sup_n z_n=\objectivefm{T}{\strategy_\trajectory}$. Then by the Ces\`{a}ro mean convergence theorem~\citep[Chapter  8.5]{SA:2015},
		$$
		\lim_{N \to \infty} \frac{1}{N}\sum_{n=0}^{N-1} z_n =\lim_{n\to \infty} z_n = \objectivefm{T}{\strategy_{\trajectory}}. 
		$$
		If $\objectivefm{T}{\strategy_{\trajectory}}=+\infty$, then \(\sup_{n\geq0} z_n=+\infty\), and since \((z_n)_{n\ge0}\) is nondecreasing, we have \(z_n\to+\infty\). For every $N$,
		\begin{equation*}
        \begin{aligned}
            \frac{1}{N} \sum_{n=0}^{N-1} z_n & \geq \frac{1}{N} \sum_{n=\lfloor N / 2\rfloor}^{N-1} z_n \\ & \geq \frac{N-\lfloor N / 2\rfloor}{N} z_{\lfloor N / 2\rfloor}  \geq \frac{1}{2} z_{\lfloor N / 2\rfloor}.
        \end{aligned}
		\end{equation*}
		Letting $N\to \infty$ gives $\frac{1}{N} \sum_{n=0}^{N-1} z_n \rightarrow+\infty$. Therefore, in this case also,
		$$\lim _{N \rightarrow \infty} \frac{1}{N} \sum_{n=0}^{N-1} z_n=\objectivefm{T}{\strategy_{\trajectory}}.
		$$
         Moreover, the stage cost of trajectory $\trajectory$ is $C_n(\trajectory)=z_n$. Thus, for every sample trajectory \(\trajectory\) generated by the non-Zeno MDP policy \(\mu\), the induced monitoring strategy \(\strategy_\trajectory\) satisfies the pathwise identity $\lim_{N\to\infty}
        \frac{1}{N}
        \sum_{n=0}^{N-1} C_n(\trajectory)
        =
        \objectivefm{T}{\strategy_\trajectory}$.


        We now pass from the pathwise identity to the policy-level average cost. Since \(C_n(\trajectory)=z_n\) and \((z_n)_{n\ge0}\) is nonnegative and nondecreasing along every trajectory, the sequence \((\frac{1}{N}
        \sum_{n=0}^{N-1} C_n(\trajectory))_{N\ge1}\) is also nonnegative and nondecreasing. Hence, by the conditional monotone convergence theorem~\citep[Chapter~9.7]{DW:1991},
        \begin{equation}
            \label{eq:aveexchange}
            \begin{aligned}
        \aveobj_{s_0}(\mu)
        &=
        \lim_{N\to\infty}
        \frac{1}{N}
        \sum_{n=0}^{N-1}
        \mathbb E_{\trajectory\sim P_\mu}
        \left[
        C_n(\trajectory)
        \,\middle|\,
        s_0
        \right] \\
        &=
        \lim_{N\to\infty}
        \mathbb E_{\trajectory\sim P_\mu}
        \left[
        \frac{1}{N}
        \sum_{n=0}^{N-1} C_n(\trajectory)
        \,\middle|\,
        s_0
        \right] \\
        &=
        \mathbb E_{\trajectory\sim P_\mu}
        \left[
        \left.
        \lim_{N\to\infty}\frac{1}{N}
        \sum_{n=0}^{N-1} C_n(\trajectory)
        \,\right|\,
        s_0
        \right].
        \end{aligned}
        \end{equation}
        Using the pathwise identity above, this becomes
        \[
        \aveobj_{s_0}(\mu)
        =
        \mathbb E_{\trajectory\sim P_\mu}
        \left[
        \left.
        \objectivefm{T}{\strategy_\trajectory}
        \,\right|\,
        s_0
        \right].
        \]
        
        If \(\aveobj_{s_0}(\mu)<+\infty\), then there must exist at least one sample trajectory \(\trajectory^*\) such that $\objectivefm{T}{\strategy_{\trajectory^*}}
        \le
        \aveobj_{s_0}(\mu)$. 
        Otherwise, if $\objectivefm{T}{\strategy_{\trajectory}}
        >
        \aveobj_{s_0}(\mu)$ for every sample trajectory, the conditional expectation would be strictly larger than \(\aveobj_{s_0}(\mu)\), a contradiction. If \(\aveobj_{s_0}(\mu)=+\infty\), the desired inequality is trivial: since \(\mu\in\Omega_{\mathrm{NZ}}\), there exists a  sample trajectory \(\trajectory^*\), and the induced strategy satisfies $\objectivefm{T}{\strategy_{\trajectory^*}}
        \le
        +\infty
        =
        \aveobj_{s_0}(\mu)$. In both cases, there exists a sample trajectory \(\trajectory^*\) generated by \(\mu\) such that the induced rational monitoring strategy $\strategy_\mu
        :=
        \strategy_{\trajectory^*}
        \in
        \rationalstrategyset_{\inipos}$
        satisfies $\objectivefm{T}{\strategy_\mu}
        =
        \objectivefm{T}{\strategy_{\trajectory^*}}
        \le
        \aveobj_{s_0}(\mu)$. 
        This completes the proof.

\end{proof}

This construction immediately implies the lower bound on the optimal MDP value over non-Zeno policies.
\begin{corollary}[(Lower bound of non-Zeno MDP policies)]
\label{cor:mdp_lower_bound}
The optimal average-cost value of the TWLO-MDP over non-Zeno policies is no smaller than the optimal value of the original monitoring problem, i.e.,
\[
\inf_{\mu\in\Omega_{\mathrm{NZ}}}\aveobj_{s_0}(\mu)
\ge
\optobjall .
\]
\end{corollary}

\begin{proof}
By \cref{lemma:mdp2strategy}, for any $\mu\in\Omega_{\mathrm{NZ}}$, there exists a feasible rational monitoring strategy
$\strategy_\mu\in\rationalstrategyset_{\inipos}$ such that
\[
\objectivefm{T}{\strategy_\mu}
\le
\aveobj_{s_0}(\mu).
\]
Since $\strategy_\mu$ is feasible for the original monitoring problem and $\optobjall$ is the optimal value of~\cref{prob:monitoring-problem}, we have
\[
\optobjall
\le
\objectivefm{T}{\strategy_\mu}
\le
\aveobj_{s_0}(\mu).
\]
Since $\mu\in\Omega_{\mathrm{NZ}}$ is arbitrary, it follows that
\[
\inf_{\mu\in\Omega_{\mathrm{NZ}}}\aveobj_{s_0}(\mu)
\ge
\optobjall .
\]
\end{proof}

To show that the MDP value is not larger than the original optimal value, we start from an optimal rational monitoring strategy and construct a corresponding non-Zeno stationary Markov policy for the TWLO-MDP.
  \begin{lemma}[(From an optimal monitoring strategy to a Markov policy)]
\label{lemma:strategy2mdp}
Let $ \rationalstrategyset^*_{\inipos,T}$ denote the set of optimal rational strategies for \cref{prob:monitoring-problem} with initial configuration $\inipos$ and tail parameter $T$. Suppose $T>\graphdiameter+\optobjall/\weightmin$. 
Then, for any optimal rational monitoring strategy 
$\strategy^*\in\rationalstrategyset^*_{\inipos,T}$, there exists a non-Zeno stationary Markov policy \(\mu_{\strategy^*}\) for the TWLO-MDP such that
\[
\aveobj_{s_0}(\mu_{\strategy^*})
\le
\objectivefm{T}{\strategy^*}
=
\optobjall .
\]
\end{lemma}

\begin{proof}
    First, by item~\ref{item:accessible} of \cref{thm:properties}, since $T>\graphdiameter+\optobjall/\weightmin$, there exists an optimal strategy $\strategy_1^*\in \optstrategyT{T}\cap \intstrategyset{\inipos}$. Moreover, by the discussion in \cref{subsection:discretized}, the
strategy  $\strategy_1^*$ can be  transformed into a rational strategy $\strategy^*$ without
degrading the monitoring performance.  Hence, $ \rationalstrategyset^*_{\inipos,T}$ is nonempty. Now fix arbitrary an optimal rational strategy $\strategy^*\in\rationalstrategyset^*_{\inipos,T}$. By definition, it satisfies $\objectivefm{T}{\strategy^*}=\optobjall$.

Since $\strategy^*$ is rational, it admits a canonical event-driven representation. 
Let $(t_n)_{n\ge0}$ denote the corresponding event times. On each event interval $[t_n,t_{n+1})$, every available robot either moves to a neighboring node or waits at its current node for a prescribed duration. 
Therefore, the behavior of $\strategy^*$ over this interval can be encoded by a joint MDP action defined in TWLO-MDP, denoted by $\boldsymbol a_n^*$.
We now record, along this event-driven execution of $\strategy^*$, the state variables used in the TWLO-MDP. 
Specifically, let $s_n^*
=
(\jointpsotion_n^*,\jointlatency_n^*,z_n^*,\eta_n^*)$ be the MDP state obtained from the robot joint positions, weighted latency vector, worst latency tracker value, and elapsed-time counter induced by $\strategy^*$ at event time $t_n$. In this way, the fixed monitoring strategy $\strategy^*$ induces a deterministic MDP-compatible state-action sequence
$(s_0^*,\boldsymbol a_0^*),
(s_1^*,\boldsymbol a_1^*),
(s_2^*,\boldsymbol a_2^*),\ldots$
The goal is then to convert this state-action sequence into a stationary Markov policy. 
We distinguish two cases according to whether the event-state sequence $(s_n^*)_{n\ge0}$ contains repetitions.


\medskip
\noindent
\textbf{Case C1: no repeated event state.}
Suppose that the sequence \((s_n^*)_{n\ge0}\) contains no repetition. Then the
map $n\mapsto s_n^*$ is injective on the set of event indices. Therefore, the time-indexed action
sequence \((\boldsymbol a_n^*)_{n\ge0}\) can be reparameterized as a state-feedback
mapping on the visited MDP states:
\[
\mu_{\strategy^*}(s_n^*)
:=
\boldsymbol a_n^*,
\qquad \forall n\in\Nature.
\]
On MDP states that are not visited by \(\strategy^*\), define
\(\mu_{\strategy^*}\) arbitrarily by choosing any feasible action. Thus
\(\mu_{\strategy^*}\) is a stationary Markov policy.

Starting from \(s_0\), the MDP trajectory generated by \(\mu_{\strategy^*}\)
follows exactly the same sequence of states and actions as those induced by
\(\strategy^*\). Indeed, the TWLO-MDP transition kernel uses the same
event-driven update rules as the monitoring dynamics under \(\strategy^*\), and
the tracker component \(z_n\) is updated according to the same tail-latency rule.
Therefore, by the pathwise equivalence established in \cref{lemma:mdp2strategy},
the average cost collected along this MDP trajectory coincides with the tail
monitoring performance of \(\strategy^*\). Hence, $\aveobj_{s_0}(\mu_{\strategy^*})
=
\objectivefm{T}{\strategy^*}
=
\optobjall$.

\medskip
\noindent
\textbf{Case C2: repeated event states.}
Suppose that the sequence \((s_n^*)_{n\ge0}\) contains at least one repetition.
Let \(s_{\ell^*}^*\) be the first repeated MDP state along the sequence, and let
\(j^*<\ell^*\) be the unique earlier index such that $s_{j^*}^*=s_{\ell^*}^*$. Since the MDP state includes the transient-time component \(\eta_n\), any repeated MDP state must occur after the transient phase. Indeed, by construction, \(\eta_n\) increases with the event time before \(T\) and remains constant only after the event time reaches \(T\). Hence, if \(s_{j^*}^*=s_{\ell^*}^*\) for some \(j^*<\ell^*\), then necessarily $t_{j^*}\ge T$ and $t_{\ell^*}\ge T$. Thus the recurrent cycle considered below lies entirely in the tail interval. For each \(0\le n<\ell^*\), define
\[
\mu_{\strategy^*}(s_n^*)
:=
\boldsymbol a_n^*.
\]
This definition is well-posed because \(s_{\ell^*}^*\) is the first repeated state, and hence the states $s_0^*,s_1^*,\ldots,s_{\ell^*-1}^*$ are pairwise distinct. Since $s_{\ell^*}^*=s_{j^*}^*$ for some \(j^*<\ell^*\), the action selected when the trajectory returns to \(s_{\ell^*}^*\) is already determined by the value of the policy at the same state: $\mu_{\strategy^*}(s_{\ell^*}^*)
=
\mu_{\strategy^*}(s_{j^*}^*)
=
\boldsymbol a_{j^*}^*$. Thus, after reaching the recurrent state, the policy repeats the same transition that was taken from its first occurrence \(s_{j^*}^*\). On all other MDP
states, define the policy arbitrarily by choosing any feasible action. Then
\(\mu_{\strategy^*}\) is a stationary Markov policy.

Starting from \(s_0\), the induced MDP trajectory follows the prefix $s_0^*\to s_1^*\to\cdots\to s_{j^*}^*$ and then repeats the finite cycle $s_{j^*}^*\to s_{j^*+1}^*\to\cdots\to
s_{\ell^*}^*(=s_{j^*}^*)$ indefinitely. This execution is non-Zeno because the cycle consists of finitely
many event intervals from the rational strategy \(\strategy^*\), and its total
duration is strictly positive; hence repeating the cycle indefinitely yields
event times diverging to \(+\infty\).

It remains to compare the average cost. Since $s_{j^*}^*=s_{\ell^*}^*$, then, in particular, $z_{j^*}^*=z_{\ell^*}^*$. Because the tracker sequence \((z_n^*)_{n\ge0}\) is nondecreasing, it follows that $z_n^*=z_{j^*}^*$ for all  $n\in\{j^*,j^*+1,\ldots,\ell^*\}$. Therefore, repeating the cycle cannot increase the tracker value beyond the
value already attained along the original trajectory. The limiting average cost
of the constructed MDP policy is thus no larger than the tail objective of
\(\strategy^*\). Hence,
\[
\aveobj_{s_0}(\mu_{\strategy^*})
\le
\objectivefm{T}{\strategy^*}
=
\optobjall .
\]

Combining the two cases, we have constructed a stationary Markov policy
\(\mu_{\strategy^*}\) such that
\[
\aveobj_{s_0}(\mu_{\strategy^*})
\le
\objectivefm{T}{\strategy^*}
=
\optobjall .
\]
Moreover, the constructed policy is non-Zeno in both cases. In Case C1, it
reproduces the non-Zeno event-time sequence of the rational strategy
\(\strategy^*\). In Case C2, it eventually repeats a finite cycle with positive
total duration. Therefore, $\mu_{\strategy^*}\in\Omega_{\mathrm{NZ}}$. This completes the proof.

\end{proof}

This construction immediately yields the reverse value inequality, as the constructed policy is a feasible candidate in $\Omega_{\mathrm{NZ}}$.

\begin{corollary}[(Upper bound from an optimal monitoring strategy)]
\label{cor:mdp_upper_bound}
Suppose $T>\graphdiameter+\optobjall/\weightmin$. Then the optimal average-cost value of the TWLO-MDP over non-Zeno policies is no larger than the optimal value of the original monitoring problem, i.e.,
\[
\inf_{\mu\in\Omega_{\mathrm{NZ}}}\aveobj_{s_0}(\mu)
\le
\optobjall .
\]
\end{corollary}

\begin{proof}
By \cref{lemma:strategy2mdp}, for an optimal rational monitoring strategy $\strategy^* \in  \rationalstrategyset^*_{\inipos,T}$  with $T>\graphdiameter+\optobjall/\weightmin$, we can construct a non-Zeno stationary Markov policy
$\mu_{\strategy^*}\in\Omega_{\mathrm{NZ}}$  such that
\[
\aveobj_{s_0}(\mu_{\strategy^*})
\le
\objectivefm{T}{\strategy^*}
=
\optobjall.
\]
Since \(\mu_{\strategy^*}\) is feasible for the optimization over \(\Omega_{\mathrm{NZ}}\), we have
\[
\inf_{\mu\in\Omega_{\mathrm{NZ}}}\aveobj_{s_0}(\mu)
\le
\aveobj_{s_0}(\mu_{\strategy^*})
\le
\optobjall .
\]
\end{proof}

Having established both value inequalities, we now combine them to complete the proof of \cref{thm:MDPopt}.
   \begin{proof}[Proof of~\cref{thm:MDPopt}]
       Combining \cref{cor:mdp_lower_bound,cor:mdp_upper_bound}, we obtain the equality between the optimal average-cost value of the TWLO-MDP over non-Zeno policies and the optimal tail performance of the original monitoring problem: 
		\begin{equation}\label{eq:equivalenceofMDP}
			\inf_{\mu\in{\Omega}_{\mathrm{NZ}}} \aveobj_{s_0}(\mu)
			=
			\optobjall.
		\end{equation}
            
        It remains to prove that the infimum on the MDP side is attained by a non-Zeno stationary Markov policy. Since $T>\graphdiameter+\optobjall/\weightmin$,  we may choose an optimal rational strategy $\strategy^*\in\rationalstrategyset^*_{\inipos,T}$ 
        such that $\objectivefm{T}{\strategy^*}
        =
        \optobjall$. 
        By \cref{lemma:strategy2mdp}, this strategy induces a non-Zeno stationary Markov policy $\mu_{\strategy^*}\in\Omega_{\mathrm{NZ}}$  
        satisfying $\aveobj_{s_0}(\mu_{\strategy^*})
        \le
        \objectivefm{T}{\strategy^*}
        =
        \optobjall$. Together with~\eqref{eq:equivalenceofMDP}, this implies $\aveobj_{s_0}(\mu_{\strategy^*})
        =
        \inf_{\mu\in\Omega_{\mathrm{NZ}}}
        \aveobj_{s_0}(\mu)
        =
        \optobjall$. 
        Therefore, the TWLO-MDP admits an optimal non-Zeno stationary Markov policy. In particular, one such policy is $\mu^*:=\mu_{\strategy^*}$.

        Finally, we show that an optimal non-Zeno MDP policy
induces an optimal monitoring strategy. Let $\mu^*\in\Omega_{\mathrm{NZ}}$ be an optimal non-Zeno MDP policy.  By \cref{lemma:mdp2strategy}, we can construct a feasible rational monitoring strategy  $\strategy_{\mu^*}\in\rationalstrategyset_{\inipos}$ such that \[
\objectivefm{T}{\strategy_{\mu^*}}
\le
\aveobj_{s_0}(\mu^*)
=
\optobjall .
\]
Since $\strategy_{\mu^*}$ is feasible for the original monitoring problem, we also have
$\objectivefm{T}{\strategy_{\mu^*}}\ge \optobjall$.
Therefore,
\[
\objectivefm{T}{\strategy_{\mu^*}}=\optobjall .
\]
Consequently, any optimal non-Zeno MDP policy admits an induced feasible monitoring strategy starting from \(\inipos\) that achieves the same optimal value \(\optobjall\). If the optimal policy is deterministic, its unique induced monitoring strategy is optimal.
        
        Thus, the TWLO-MDP over non-Zeno policies and the original monitoring problem share the same optimal value \(\optobjall\), and an optimal monitoring strategy can be obtained by solving the TWLO-MDP and mapping an optimal non-Zeno policy back to the continuous-time monitoring setting.

		\end{proof}

	\subsection{Proof of~\cref{lemma:timenormalized}}\label{appendix:timenormalized}
	\begin{proof}[Proof of~\cref{lemma:timenormalized}]
		Fix a non-Zeno stationary policy \(\mu\in\Omega_{\mathrm{NZ}}\) and an initial state \(s_0\). Consider any trajectory \(\trajectory=\{(s_0,\boldsymbol a_0),(s_1,\boldsymbol a_1),\dots\}\) generated by $\mu$ from $s_0$. As shown in Appendix~\ref{appendix:MDPopt}, along every such trajectory the tracker sequence \(\{z_n\}_{n\ge0}\) is nonnegative and nondecreasing, and satisfies
		$$
		\lim _{N \rightarrow \infty} \frac{1}{N} \sum_{n=0}^{N-1} z_n =\lim _{n \rightarrow \infty} z_n \triangleq z_{\infty}\in \overRp.
		$$
		We show that the time-normalized average converges to the same limit. Define 
		\[
		B_N:=\frac{\sum_{n=0}^{N-1} z_n \Delta t_n}{\sum_{n=0}^{N-1} \Delta t_n}.
		\]
		Since $\Delta t_n>0$, $B_N$ is a positive weighted average of $z_0,\dots,z_{N-1}$, and we have $B_N\le z_{N-1}\le z_\infty$. This implies $\limsup_{N\to\infty}B_N\le z_\infty$.
		
        For the reverse inequality, fix any $m\ge 0$. 
        For all $N>m$, we have
		\[
		B_N
		\ge
		\frac{\sum_{n=m}^{N-1} z_n\Delta t_n}{\sum_{n=0}^{N-1} \Delta t_n}
		\ge
		z_m\,
		\frac{\sum_{n=m}^{N-1}\Delta t_n}{\sum_{n=0}^{N-1}\Delta t_n}.
		\]
		Since \(\mu\in\Omega_{\mathrm{NZ}}\), we have
		\begin{equation*} \label{eq:tNinfinite}
			\sum_{n=0}^{N-1} \Delta t_n=t_N \xrightarrow{\text { as } N\to\infty} \infty,
		\end{equation*}
		where \(t_N\) is the physical time at which the \(N\)-th event occurs. 
		Moreover, since $m$ is fixed, we have
		\[
		\frac{\sum_{n=m}^{N-1}\Delta t_n}{\sum_{n=0}^{N-1}\Delta t_n}
		=
		1-
		\frac{\sum_{n=0}^{m-1}\Delta t_n}{\sum_{n=0}^{N-1}\Delta t_n}
		\xrightarrow{\text { as } N\to\infty} 1.
		\]
		Therefore, $\liminf_{N\to\infty} B_N \ge z_m$. 
		Letting $m\to\infty$ gives $\liminf_{N\to\infty} B_N \ge z_\infty$. Together with $\limsup_{N\to\infty} B_N \le z_\infty$, this yields
		\[
		\lim_{N \to \infty}B_N = z_\infty =\lim _{N \rightarrow \infty} \frac{1}{N} \sum_{n=0}^{N-1} z_n.
		\]
		It remains to pass from the pathwise equality to the policy objective. Since $B_N\le z_{N-1}\le z_N$, we have 
		$$
		B_{N+1}-B_N = \frac{\Delta t_N\,(z_N-B_N)}{\sum_{n=0}^{N}\Delta t_n}\ge 0.
		$$
        Thus \(\{B_N\}_{N\ge1}\) is nonnegative and nondecreasing. By the conditional monotone convergence theorem, we have
		$$
		\begin{aligned}
			\lim_{N\to\infty}\mathbb E_\mu[B_N\mid s_0] & =
			\mathbb{E}_{\mu}\left[  \lim_{N\to\infty} B_N \, \big| \, s_0 \right] \\&  =\mathbb{E}_{\mu}\left[  z_\infty \, \big| \, s_0 \right] \\&
			=\mathbb{E}_{\mu}\left[  \lim _{N \rightarrow \infty} \frac{1}{N} \sum_{n=0}^{N-1} z_n \, \big| \, s_0 \right] \\ & = \aveobj_{s_0}(\mu),
		\end{aligned}
		$$
		where the last equality holds as we have shown in~\eqref{eq:aveexchange}.
		This completes the proof of \cref{lemma:timenormalized}.

	\end{proof}

\subsection{Training details}\label{appendix:training}

\subsubsection{Training process}
All experiments were conducted on a workstation equipped with a single NVIDIA GeForce RTX 4090 GPU and an Intel Core i9-14900K CPU. We tuned hyperparameters using the Bayesian sweep strategy implemented in Weights \& Biases Sweeps~\citep{wandb}, following the general
Bayesian optimization framework for machine-learning hyperparameter
selection~\citep{snoek2012practical}. For each map, we ran $48$ sweep trials for the proposed TWLO-MDP using the search configuration specified in \cref{tab:sweep_hyperparams}. Each sweep trial corresponds
to a fixed task, method, random seed, and hyperparameter configuration. The hyperparameter setting that achieves the lowest metric value is selected for the final evaluation.

\begin{table*}[!t]
\centering
\caption{Search space and sampling rules for the principal hyperparameters in the TWLO-MDP sweep. Continuous-valued hyperparameters are sampled either log-uniformly or uniformly,  discrete-valued hyperparameters are sampled from categorical candidate sets, and fixed hyperparameters are kept unchanged across sweep trials.}
\label{tab:sweep_hyperparams}

\renewcommand{\arraystretch}{1.25}
\begin{tabular}{llcc}
\toprule
{Category} & {Hyperparameter} & {Search space / value} & {Sampling rule} \\
\midrule
\multirow{4}{*}{Optimizer}
  & actor\_lr           & $\left[1{\times}10^{-6},\ 1{\times}10^{-4}\right]$ & log-uniform \\
  & critic\_lr          & $\left[5{\times}10^{-6},\ 5{\times}10^{-4}\right]$ & log-uniform \\
  & update\_epochs      & $\{3,\ 5,\ 10\}$                                  & categorical \\
  & minibatch\_size     & $\{2048,\ 4096\}$                                 & categorical \\
\midrule
\multirow{4}{*}{PPO/GAE}
  & clip\_range         & $\left[0.1,\ 0.3\right]$        & uniform \\
  & vf\_coef            & $\left[0.5,\ 1.0\right]$        & uniform \\
  & ent\_coef           & $\left[0.0,\ 0.1\right]$        & uniform \\
  & gae\_lambda         & $\{0.99,\ 0.999\}$              & categorical \\
\midrule
\multirow{2}{*}{Rollout}
  & $N_{\mathrm{steps}}$          & $\{1024,\ 2048\}$               & categorical \\
  & total\_steps        & $5{\times}10^{7}$               & fixed \\
\midrule
\multirow{2}{*}{Auxiliary credit}
  & $\lambda_L$         & $\{0.0,\ 0.05,\ 0.1\}$          & categorical \\
  & $\lambda_z$ & $\{0.0,\ 0.05,\ 0.1\}$          & categorical \\
  \midrule
Discounting  & $\gamma$            & $0.999$                         & fixed \\
\bottomrule
\end{tabular}
\end{table*}

In addition to the swept quantities in \cref{tab:sweep_hyperparams}, several data-generation and environment-level hyperparameters are fixed separately  for each map and are excluded from the Bayesian search space. First, the heuristic algorithm used for MAPPO imitation pretraining is selected on a per-map basis. Specifically, we evaluate all candidate heuristic baselines on each map and use the best-performing one to collect the corresponding pretraining trajectories. Second, the episode length $H_{\text{ep}}$ defines the physical-time horizon of each episode. It is chosen according to the spatial scale of the graph, especially the edge-length distribution, so that each episode covers a sufficiently long monitoring horizon. The same $H_{\text{ep}}$ is also used to truncate episodes during evaluation. Third, the tail parameter $T$ is determined jointly by the graph size $\numnode$, the number of monitoring robots $\numrobot$, and the episode length $H_{\mathrm{ep}}$, so that the transient period is adapted to both the map scale and the multi-robot monitoring timescale. Fourth, the initial robot positions are fixed rather than randomized. For each map, the $\numrobot$ robots are assigned to starting nodes that are approximately uniformly spaced over the graph, which removes initialization randomness as a confounding factor in cross-map comparisons. Finally,  the number of parallel environments \(N_{\mathrm{env}}\) is fixed to  either $32$ or $64$, depending on the map size and the number of monitoring robots in each task. The fixed values used for all maps are summarized in \cref{tab:fixed_hparams}.

\begin{table*}[htbp]
\centering
\caption{Fixed TWLO-MDP hyperparameters used in the main benchmark experiments. The initial positions are given as node indices.}
\label{tab:fixed_hparams}
\renewcommand{\arraystretch}{1.2}
\setlength{\tabcolsep}{8pt}
\small
\begin{tabular}{lccccc}
\toprule
Map & Imitation algorithm & Episode length \(H_{\mathrm{ep}}\) & Tail parameter $T$ & Initial positions & $N_{\mathrm{env}}$\\
\midrule
Long-edge  & AHPA & 100  & 10.0  & \{1, 2, 3\} & 64\\
Chain      & AHPA & 100  & 10.0  & \{1, 3, 6, 9\} & 64\\
MapA       & BAPS & 500  & 50.0  & \{1, 11, 21, 31, 41, 50\} & 32\\
MapB       & BAPS & 500  & 50.0  & \{1, 11, 21, 31, 41, 50\} & 32\\
Grid       & DTAP & 500  & 50.0  & \{1, 9, 17, 25, 33, 41\} & 32\\
Island     & AHPA & 500  & 50.0  & \{1, 11, 21, 31, 41, 50\} & 32\\
Cumberland & AHPA & 500  & 50.0  & \{1, 9, 17, 25, 33, 40\} & 32\\
Milwaukee  & BAPS & 800  & 80.0  & \{1, 9, 17, 25, 33, 40\} & 32\\
Marostica  & DTAP & 500  & 50.0  & \{1, 9, 17, 25, 33, 41\} & 32\\
SF map     & BAPS & 3000 & 200.0 & \{1, 6, 12\} & 64\\
\bottomrule
\end{tabular}
\end{table*}

Given a fixed number of parallel environments \(N_{\mathrm{env}}\) and a selected rollout length \(N_{\mathrm{steps}}\), each training trial proceeds through repeated rollout-and-update iterations. In each
iteration, \(N_{\mathrm{env}}\) environments are executed in parallel for \(N_{\mathrm{steps}}\) rollout event steps. Since each environment contains \(\numrobot\) monitoring robots and contributes one transition tuple per robot at each rollout event step, the resulting on-policy batch contained
\begin{equation*}
    N_{\mathrm{batch}}
    = N_{\mathrm{env}} \times N_{\mathrm{steps}} \times \numrobot .
\end{equation*}
training samples. The collected batch is divided into minibatches and used to update the neural network parameters for the specified number of epochs. Accordingly, the cumulative global-step count is increased by \(N_{\mathrm{batch}}\) after each rollout-and-update iteration. During training, the corresponding performance metric is periodically logged against this cumulative global-step count. Each learning curve reported in \cref{section:experiments} is obtained by plotting these recorded values and connecting consecutive logging points.

\subsubsection{Surveillance graphs}\label{appendix:graphs}
To evaluate the performance of different algorithms, we consider a diverse suite of monitoring environments. The benchmark set includes special graph topologies with analytically characterized optima, which allow  validation against theoretical benchmarks,   large-scale synthetic environments, and realistic map-based scenarios derived from real-world monitoring settings. For graph instances with analytically characterized optimal solutions, in addition to the long-edge graph shown in \cref{fig:example}, we also consider the chain graph studied in~\citep{FP-AF-FB:2012}, shown in \cref{fig:chain}. For the four-robot setting considered here, the optimal strategy assigns one robot to each subchain, with each robot repeatedly traversing back and forth within its assigned segment. Under this strategy, the optimal worst-case latency is $4$.
\begin{figure*}[thbp]
    \centering
    \includegraphics[width=0.8\linewidth]{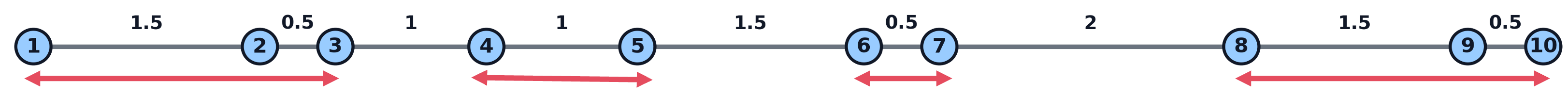}
    \caption{Chain graph benchmark with uniform node priorities ($4$ robots) adapted from~\citep{FP-AF-FB:2012}. The optimal strategy partitions the graph into four segments, each patrolled cyclically by one robot.}
    \label{fig:chain}
\end{figure*}
The  large-scale synthetic environments  considered in our experiments are shown in \cref{fig:graph}.
     \begin{figure*}[p]
		\centering
		\subfloat[MapA: Regular map with heterogeneous node priorities ($6$ robots)]{
		\includegraphics[width=0.48\textwidth]{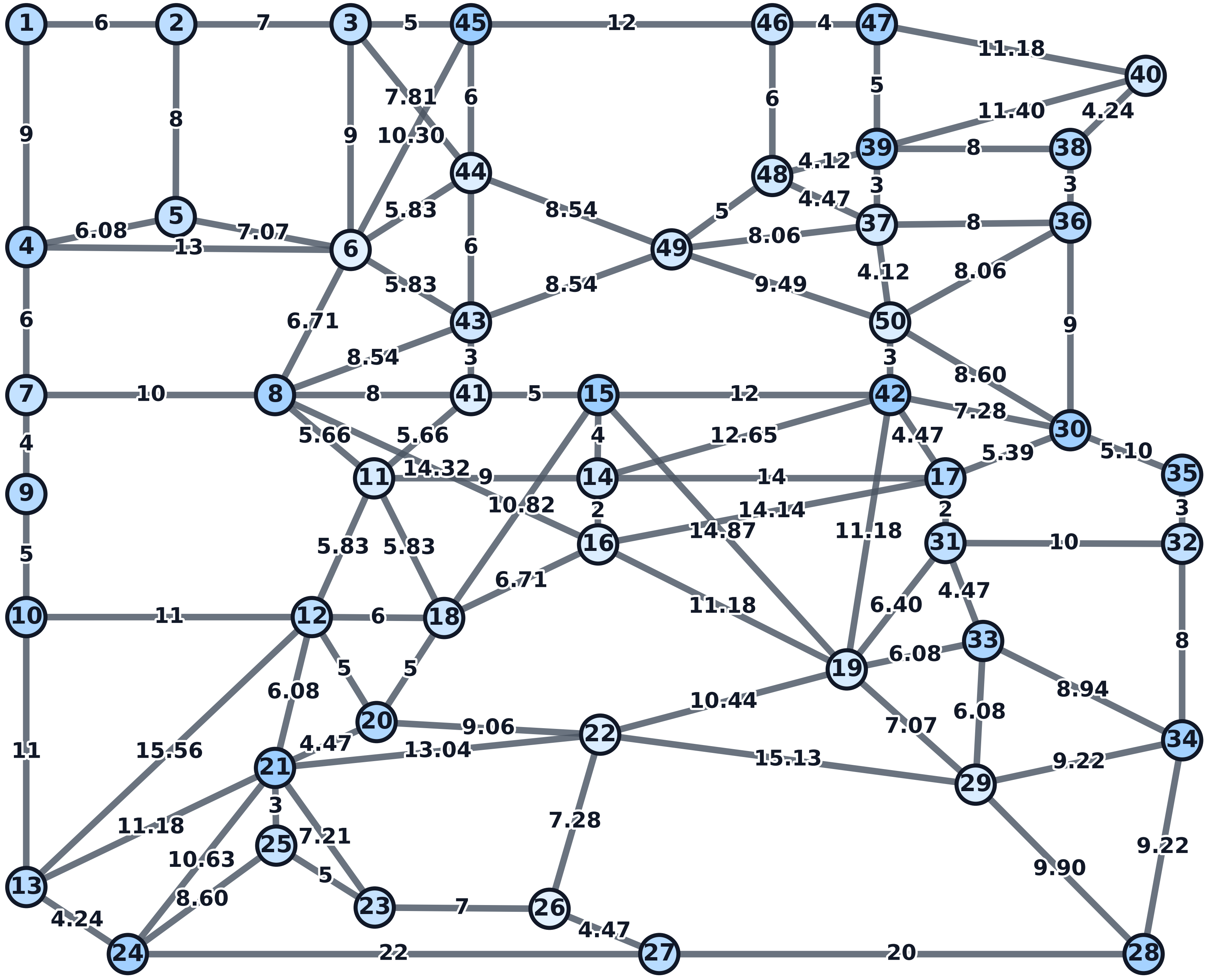}
			\label{fig:mapA}
		}
		\subfloat[MapB: Regular map with bottlenecks and heterogeneous node priorities ($6$ robots)]{
		\includegraphics[width=0.48\textwidth]{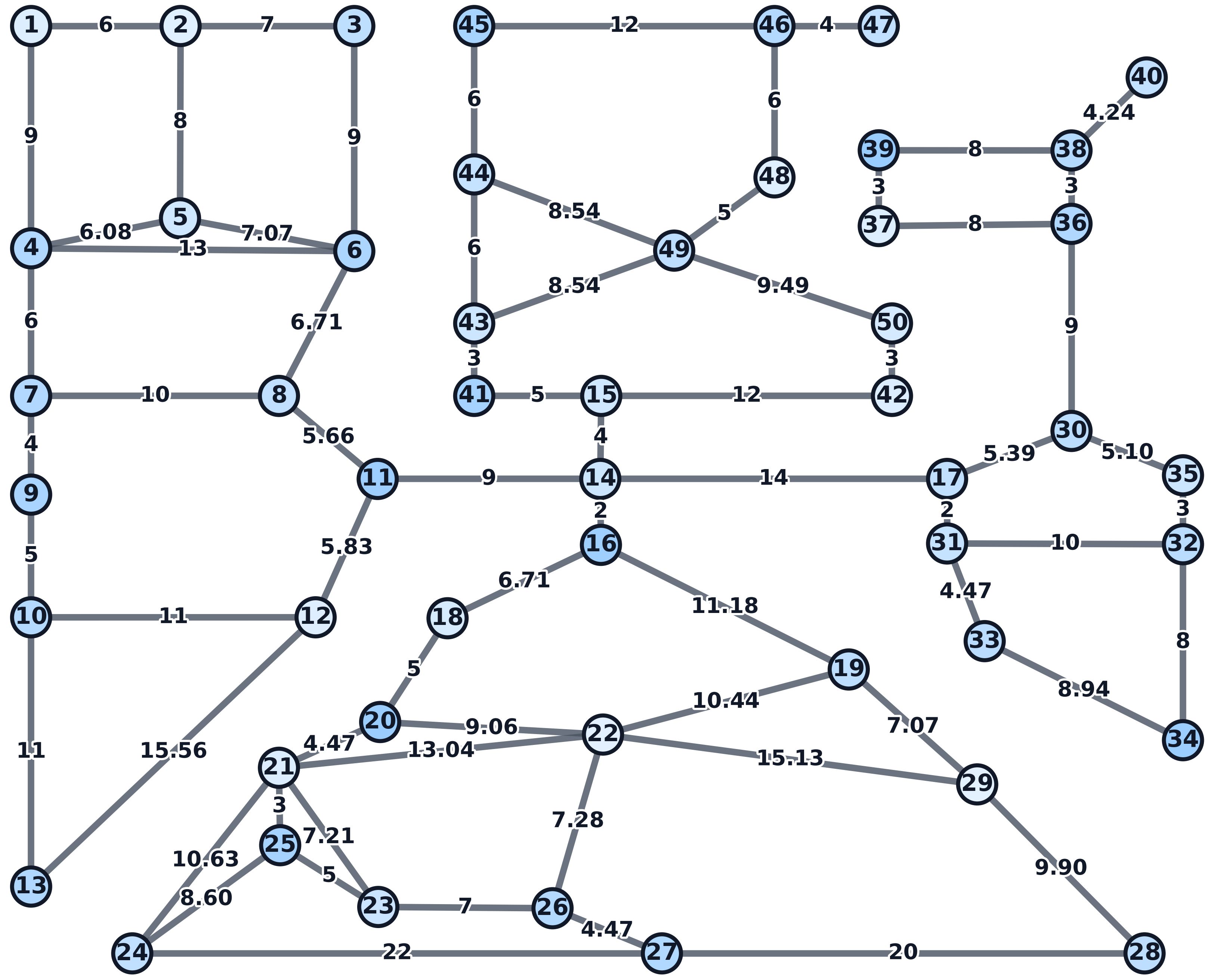}
			\label{fig:mapB}
		}
        
        \hspace*{2pt}
        \subfloat[Grid graph with uniform node priorities ($6$ robots)]{
		\includegraphics[width=0.45\textwidth]{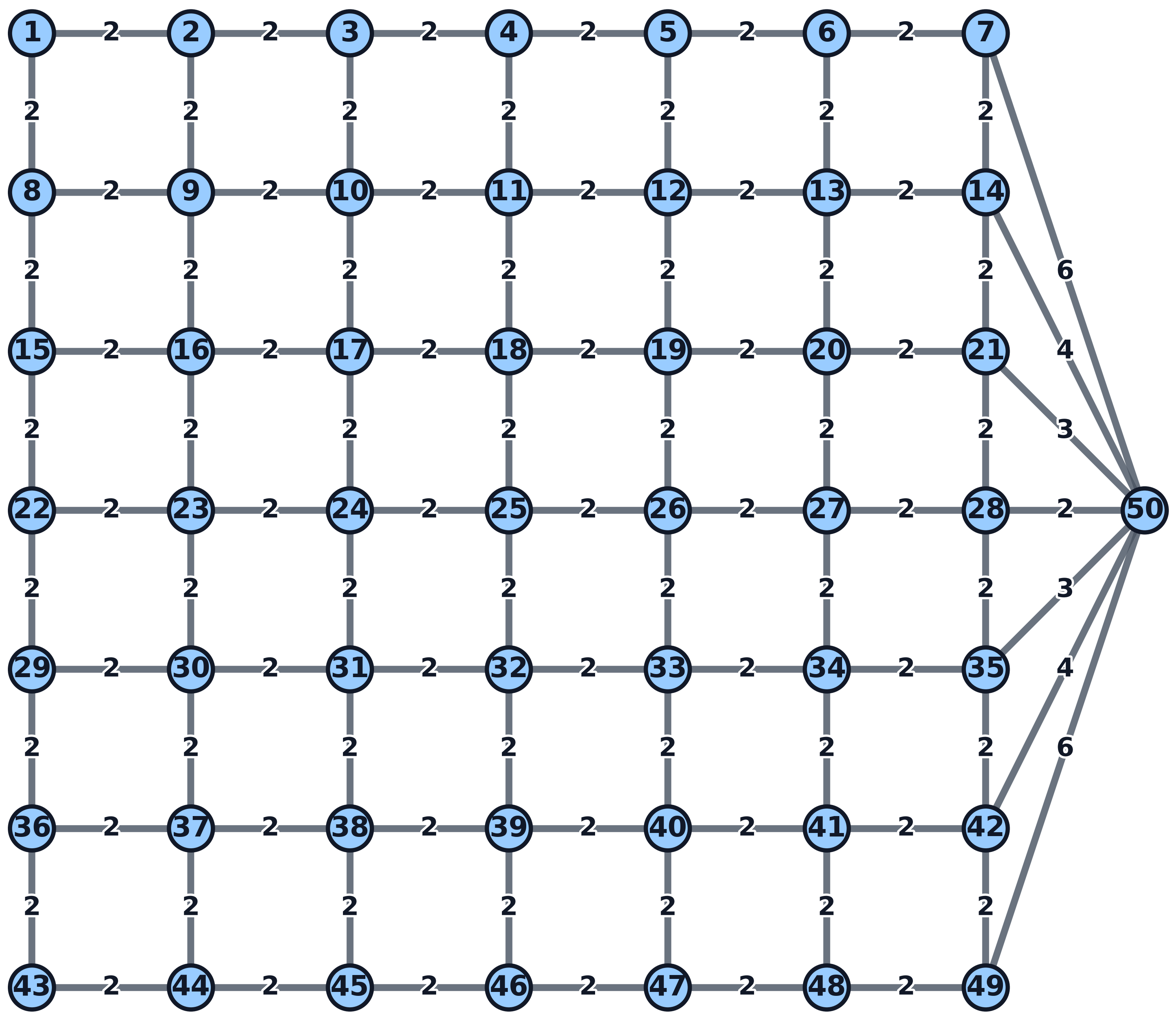}
			\label{fig:grid}
		}
             \hfill
         \subfloat[Islands with uniform node priorities ($6$ robots)]{
		\includegraphics[width=0.50\textwidth]{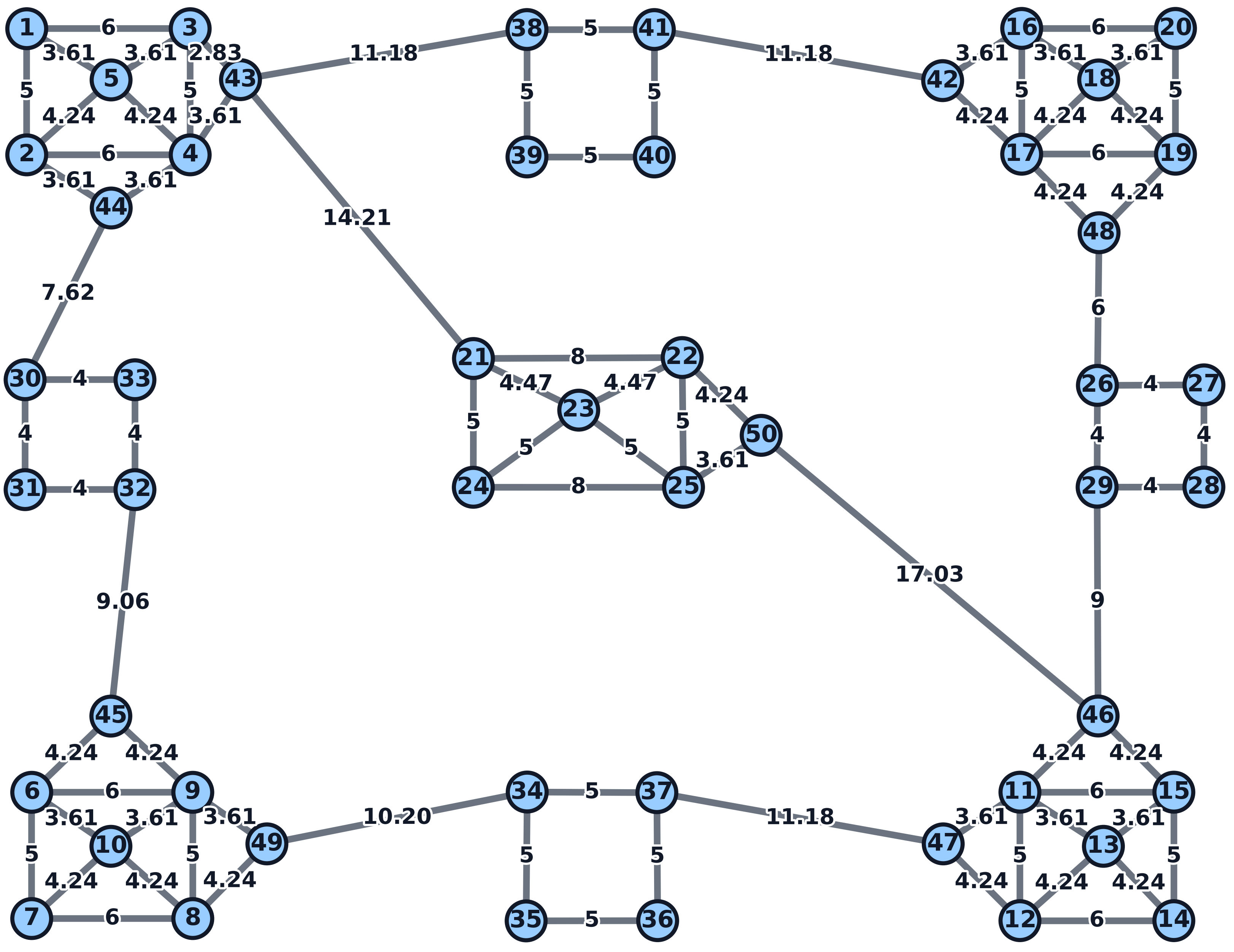}
			\label{fig:island}
		}

         \subfloat[Cumberland map  with uniform node priorities ($6$ robots)]{
		\includegraphics[width=0.52\textwidth]{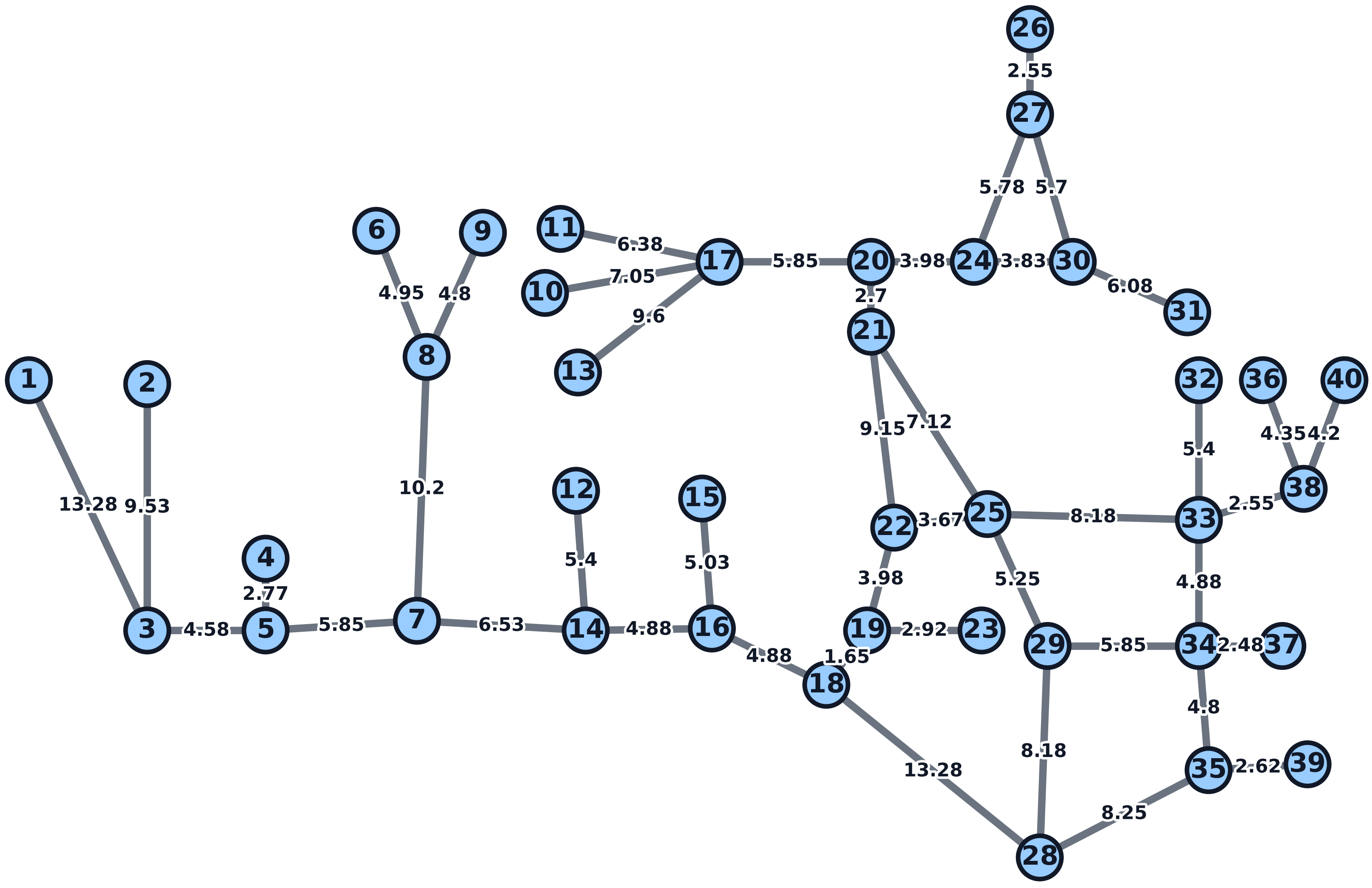}
			\label{fig:cumberland}
		}
        \subfloat[Milwaukee map  with uniform node priorities ($6$ robots)]{
		\includegraphics[width=0.44\textwidth]{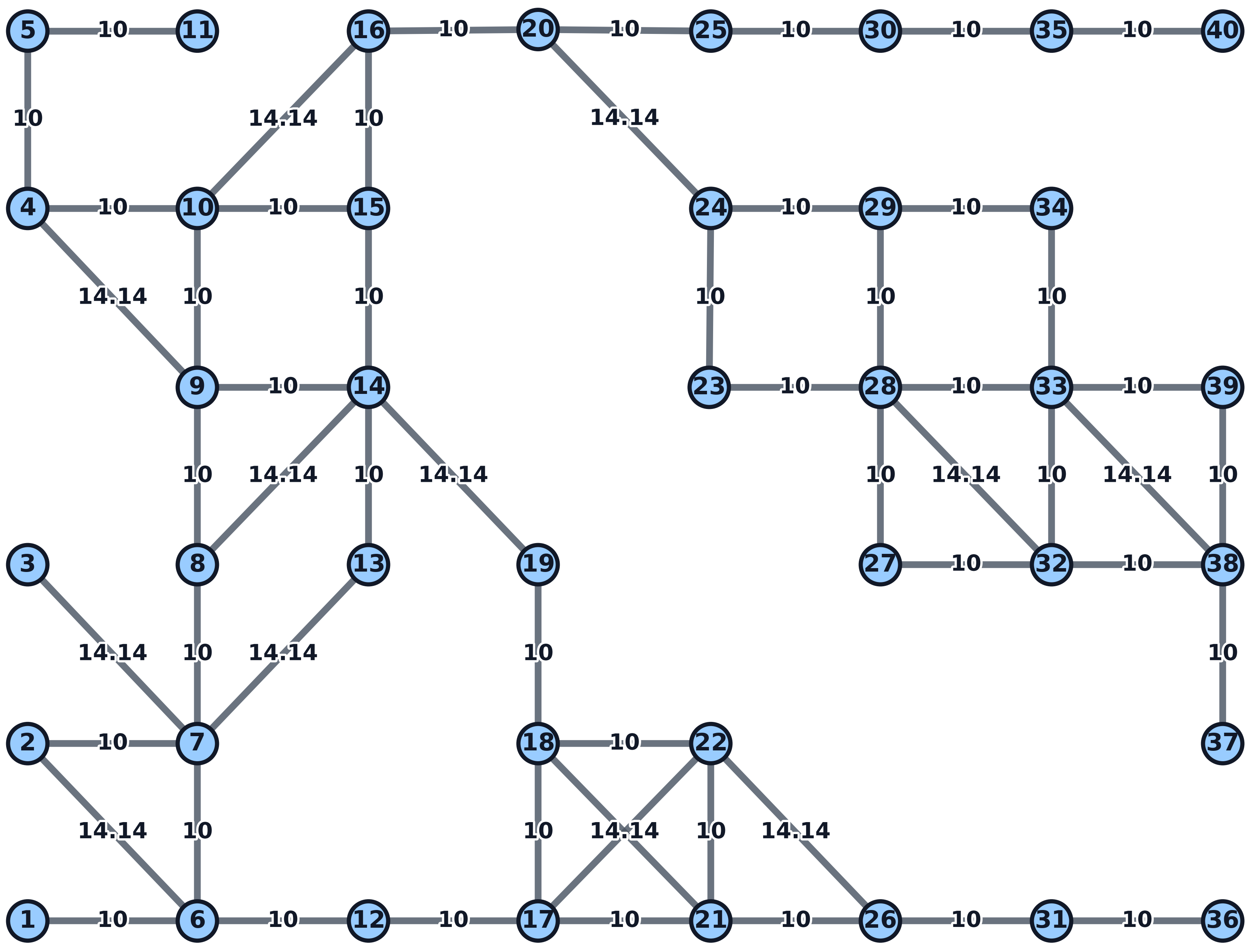}
			\label{fig:milwaukee}
		}
		\caption{Representative large-scale synthetic environments used in the experiments in \cref{section:experiments}. Node colors in \cref{fig:mapA,fig:mapB} indicate node weights, with darker colors corresponding to higher-priority nodes; these node weights are randomly generated for our experiments.  The topologies of \cref{fig:mapA,fig:mapB,fig:grid,fig:island} are based on~\citep{YC:2004}, that of \cref{fig:cumberland} is based on~\citep{DP-RR:2011}, and that of \cref{fig:milwaukee} is based on~\citep{AG-YS-QZ:2024}. }
		\label{fig:graph}
	\end{figure*}
    For realistic map-based surveillance environments, we adopt two representative benchmark instances from prior literature, as shown in \cref{fig:real}. The first is the Marostica road network from~\citep{AC-RP-FB:2016}, which represents a real urban roadmap with uniform node priorities. The second consists of $12$ locations in the central district of the San Francisco crime map from~\citep{SA-EF-SS:2012}, where  heterogeneous monitoring priorities are derived from local crime statistics. Following the original setting, the San Francisco environment is represented as a fully connected graph, with edge lengths given by the pairwise car travel times reported in the original work, as listed in \cref{table:SFweight}. To ensure numerical consistency in training, the node priorities are normalized by dividing each value by $\weightmax$.
\begin{table*}[htbp]
\centering
\caption{Pairwise travel times for the San Francisco crime map in~\cref{fig:SF}, adapted from~\citep{SA-EF-SS:2012}. The original travel times are uniformly scaled by a factor of $0.1$ for numerical convenience; this proportional scaling preserves the relative travel-time structure of the original map.}
\label{table:SFweight}
\begin{tabular}{c|rrrrrrrrrrrr}
\toprule
 & A & B & C & D & E & F & G & H & I & J & K & L \\
\midrule
A & 0 & 14.1 & 12.1 & 29.3 & 20.9 & 32.9 & 13.4 & 25.0 & 40.6 & 19.9 & 35.8 & 34.4 \\
B & 14.1 & 0 & 27.1 & 20.0 & 10.5 & 22.6 & 20.1 & 29.9 & 29.7 & 16.9 & 25.4 & 27.4 \\
C & 12.7 & 29.1 & 0 & 36.8 & 31.1 & 43.3 & 15.3 & 19.8 & 49.1 & 21.9 & 46.1 & 36.2 \\
D & 30.4 & 20.7 & 41.7 & 0 & 25.3 & 30.9 & 22.6 & 38.7 & 24.9 & 35.8 & 33.7 & 38.4 \\
E & 21.0 & 14.7 & 34.0 & 24.4 & 0 & 18.0 & 24.4 & 26.8 & 34.2 & 16.4 & 20.9 & 23.0 \\
F & 33.0 & 21.6 & 46.0 & 24.4 & 17.5 & 0 & 31.3 & 37.0 & 12.6 & 31.1 & 6.1 & 16.3 \\
G & 9.0 & 24.6 & 16.2 & 24.4 & 31.0 & 36.9 & 0 & 27.1 & 40.0 & 29.2 & 39.7 & 42.7 \\
H & 14.7 & 29.3 & 10.5 & 37.0 & 33.8 & 41.2 & 15.4 & 0 & 49.2 & 15.3 & 40.6 & 28.7 \\
I & 42.6 & 32.4 & 53.9 & 20.3 & 34.3 & 22.6 & 34.8 & 50.9 & 0 & 44.8 & 29.9 & 38.9 \\
J & 20.1 & 17.0 & 23.1 & 32.2 & 16.4 & 29.0 & 27.9 & 15.9 & 41.5 & 0 & 28.3 & 16.4 \\
K & 35.4 & 24.0 & 47.4 & 33.7 & 19.9 & 10.5 & 33.7 & 33.2 & 22.6 & 27.3 & 0 & 12.5 \\
L & 33.4 & 22.0 & 35.4 & 31.6 & 17.9 & 12.1 & 31.7 & 21.2 & 24.6 & 15.3 & 11.4 & 0 \\
\bottomrule
\end{tabular}
\end{table*}
    \begin{figure*}[htbp]
    \centering
    \subfloat[Marostica roadmap with uniform node priorities ($6$ robots)]{
    \includegraphics[width=0.33\linewidth]{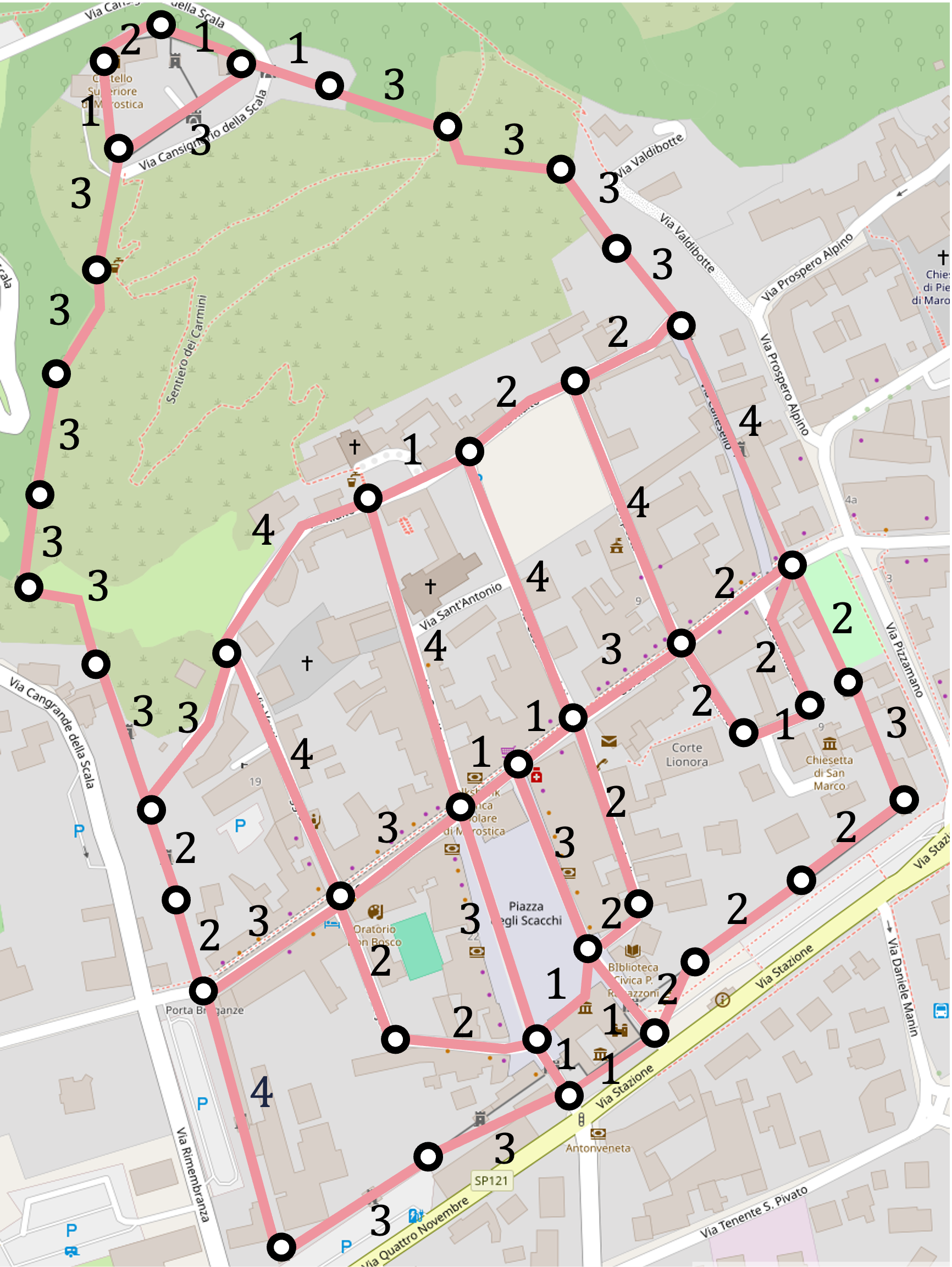}
    \label{fig:ity}}
    \subfloat[San Francisco crime map with heterogeneous node priorities ($1$ and $3$ robots). The number associated with each node denotes its priority, which is proportional to the corresponding local crime rate. ]{
    \includegraphics[width=0.47\linewidth]{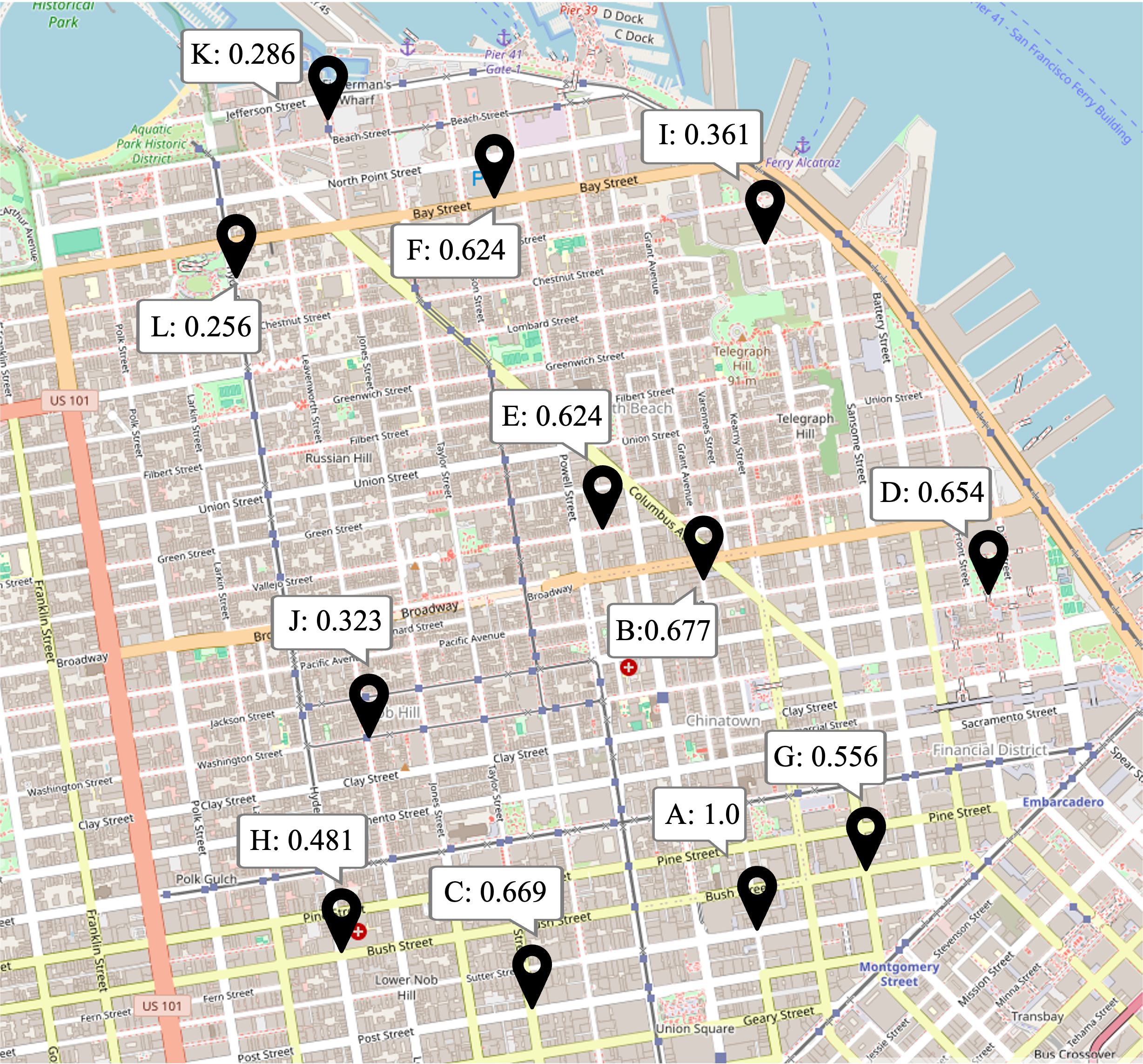}
    \label{fig:SF}}
    \caption{Realistic map-based
scenarios created with OpenStreetMap~\citep{OpenStreetMap:2017}. The Marostica roadmap in \cref{fig:ity} is adapted from~\citep{AC-RP-FB:2016}, and the San Francisco crime map in \cref{fig:SF} is adapted from \citep{SA-EF-SS:2012}.}
    \label{fig:real}
\end{figure*}

\subsubsection{Action mask} \label{subsubsection:action_mask}
In graph-based monitoring environments, the set of feasible actions depends on both the robot's current status and the local graph structure. In particular, the number of feasible movement actions at a node is determined by the degree of that node, whereas the actor network has a fixed output dimension. To handle this mismatch, we  define a fixed-size discrete action space based on the maximum node degree of the graph, together with waiting and \emph{no-op} actions. At a node with degree smaller than the maximum degree, only the action indices corresponding to its incident edges are physically valid, while the remaining movement-action slots are invalid. 

At each event step $n$, each robot $r$ is assigned an action mask $m_{n,r}$ with the same dimension as the discrete action space. Each entry of $m_{n,r}$ is binary: a value of $1$ indicates that the corresponding action is admissible for robot $r$ at step $n$, while a value of $0$ indicates that the action is masked out. If robot $r$ has arrived at a node or completed a waiting action, then it is available for decision-making. In this case, the mask permits the movement actions corresponding to the incident edges of its current node, together with the valid waiting action. In contrast, if robot $r$ is still traversing an edge or executing a waiting action, then it is unavailable for decision-making, and its mask permits only the \emph{no-op} action. Therefore, unavailable robots remain in the synchronized joint-action representation but are not allowed to choose a new physical action.

This mechanism converts asynchronous robot-level events into synchronized joint transitions. It preserves a tensorized rollout structure for efficient batching and parallel optimization, while avoiding the engineering complexity of ragged, agent-specific event logs. The resulting rollout contains both active decision steps and inactive intermediate steps for each robot. During policy and value optimization, inactive steps are handled by Trans-GAE, which folds them into the interval between consecutive active decision points and prevents them from being treated as ordinary policy-update steps.

     \subsubsection{Auxiliary credit-shaping signals}\label{subsubsection:credit_shaping}
In addition to the main TWLO-MDP cost, we optionally include auxiliary credit-shaping signals during MARL training to provide more informative feedback to individual robots. These signals are implementation-level training aids: they are not part of the theoretical TWLO-MDP formulation, and all reported results are evaluated using the monitoring metrics defined in~\cref{section:platform}. Their coefficients are treated as hyperparameters and selected together with other training hyperparameters. In some environments, the selected coefficients are zero, in which case training uses only the main TWLO-MDP signal. 
Specifically, for each robot $r\in\robots$ at event step $n$, we compute one-step counterfactual signals by comparing the realized next state with a counterfactual next state in which only robot $r$'s action is replaced by a baseline action, such as a random admissible action or a waiting action, while all other robots' actions are kept unchanged. We consider two monitoring-related quantities: the next-step latency profile and the historical worst-case weighted latency tracker. 

Let $L_{n+1,v}$ denote the realized next-step latency of node $v$, and let $L_{n+1,v}^{(-r)}$ denote its counterfactual counterpart. The latency-based auxiliary credit signal is defined as $C_{n,r}^{L}
=
\sum_{v\in\nodes}
\left(L_{n+1,v}^{(-r)}-L_{n+1,v}\right)$. Similarly, let $z_{n+1}$ and $z_{n+1}^{(-r)}$ denote the realized and counterfactual values of the historical worst-latency tracker. We define $C_{n,r}^{z}
=
z_{n+1}^{(-r)}-z_{n+1}$. The auxiliary shaping signal used during training is
$$
C_{n, r}^{\text {aux }}=\lambda_L C_{n, r}^L+\lambda_z C_{n, r}^z,$$
where $\lambda_L$ and $\lambda_z$ are tunable hyperparameters. This auxiliary signal is combined with the main training signal during policy optimization. When $\lambda_L=\lambda_z=0$, the auxiliary shaping is disabled.

\end{document}